\documentclass[10pt]{article} 

\usepackage[accepted]{rlj}           

\usepackage{amssymb}            
\usepackage{mathtools}          
\usepackage{mathrsfs}           
\usepackage{graphicx}           
\usepackage{subcaption}         
\usepackage[space]{grffile}     
\usepackage{url}                
\usepackage{lipsum}             

\usepackage{macros}
\usepackage{dirtytalk}
\usepackage[capitalize,noabbrev]{cleveref}

\theoremstyle{plain}
\newtheorem{theorem}{Theorem}[section]

\theoremstyle{definition}
\newtheorem{definition}[theorem]{Definition}
\newtheorem{assumption}[theorem]{Assumption}
\theoremstyle{remark}
\newtheorem{remark}[theorem]{Remark}

\usepackage[textsize=tiny]{todonotes}

\newcounter{findingcounter}
\newcommand{\customfinding}[1][\normalfont]{\stepcounter{findingcounter}\textbf{Finding \thefindingcounter:}#1 }

\title{Pretraining Decision Transformers with Reward Prediction for In-Context Multi-task Structured Bandit Learning

}

\setrunningtitle{Pretraining Decision Transformers with Reward Prediction for Multi-task Bandit Learning}

\author{Subhojyoti Mukherjee\textsuperscript{1,$\dagger$}, Josiah P.\ Hanna\textsuperscript{2}, Qiaomin Xie\textsuperscript{3}, Robert Nowak\textsuperscript{4}}

\emails{subhomuk@adobe.com, \ \{jphanna,qiaomin.xie,rdnowak\}@wisc.edu}

\affiliations{
$^{1}$\textbf{Adobe Research, San Jose, CA, USA}\\
$^{2}$\textbf{Computer Sciences Department, University of Wisconsin -- Madison}\\
$^{3}$\textbf{ISYE, University of Wisconsin -- Madison} \\
$^{4}$\textbf{ECE Department, University of Wisconsin -- Madison} 
\par 
$^\dagger$ Work was done as a PhD candidate at the ECE Department, University of Wisconsin-Madison
}

\contribution{
    We introduce a new pre-training and test time decision-making procedure that in-context learns the underlying reward structure for structured bandit settings, resulting in a near-optimal policy without access to privileged information even when training data comes from a sub-optimal demonstrator.  
    }
    {
    Previous works like \dpt\ \citep{lee2023supervised} required access to the optimal action per task, Algorithmic Distillation (\ad) could not outperform the demonstrator, other works need to know the structure to perform optimally.
    }

\contribution{
    We show that our approach enables successful in-context learning across a diverse set of structured bandit settings where it matches the performance of existing algorithms that were developed with knowledge of the structure.
    }
    {
    We evaluate our approach in linear,non-linear, bilinear, and latent bandit settings as well as bandit experiments based on real-life datasets and show that it lowers regret compared to \dpt\ and \ad\ while matching the near-optimal performance of specialized algorithms.
    }

\contribution{
    We show that our algorithm leverages the latent structure and conducts a two-phase exploration to minimize regret. 
    }
    {
    We analyze the exploration of the pretrained decision transformer in the simplified linear bandit setting where the optimal policy is well-understood. Previous works like \dpt\ do not study the exploration conducted by such transformer algorithms. We introduce new actions both at train and test time. Since new actions are not shared across tasks now, the transformer algorithm fails to learn the latent structure as we scale up the number of new actions, thus indicating that it is relying on a discovered underlying structure.
    We observed in our experiments that our proposed algorithm implicitly conducts two-phase exploration, following the distribution of optimal action across training tasks and then switching to the most rewarding action for the task after observing a few in-context examples. 
    }


\summary{We study learning to learn for the multi-task structured bandit problem where the goal is to learn a near-optimal algorithm that minimizes cumulative regret. The tasks share a common structure and an algorithm should exploit the shared structure to minimize the cumulative regret for an unseen but related test task. We use a transformer as a decision-making algorithm to learn this shared structure from data collected by a demonstrator on a set of training task instances. Our objective is to devise a training procedure such that the transformer will learn to outperform the demonstrator's learning algorithm on unseen test task instances. Prior work on pretraining decision transformers either requires privileged information like access to optimal arms or cannot outperform the demonstrator. Going beyond these approaches, we introduce a pre-training approach that trains a transformer network to learn a near-optimal policy in-context. This approach leverages the shared structure across tasks, does not require access to optimal actions, and can outperform the demonstrator. 
We validate these claims over a wide variety of structured bandit problems to show that our proposed solution is general and can quickly identify expected rewards on unseen test tasks to support effective exploration. 
}

\begin{document}

\makeCover  
\maketitle  

\begin{abstract}
We study learning to learn for the multi-task structured bandit problem where the goal is to learn a near-optimal algorithm that minimizes cumulative regret. The tasks share a common structure and an algorithm should exploit the shared structure to minimize the cumulative regret for an unseen but related test task. We use a transformer as a decision-making algorithm to learn this shared structure from data collected by a demonstrator on a set of training task instances. Our objective is to devise a training procedure such that the transformer will learn to outperform the demonstrator's learning algorithm on unseen test task instances. Prior work on pretraining decision transformers either requires privileged information like access to optimal arms or cannot outperform the demonstrator. Going beyond these approaches, we introduce a pre-training approach that trains a transformer network to learn a near-optimal policy in-context. This approach leverages the shared structure across tasks, does not require access to optimal actions, and can outperform the demonstrator. We validate these claims over a wide variety of structured bandit problems to show that our proposed solution is general and can quickly identify expected rewards on unseen test tasks to support effective exploration. 
\end{abstract}

\vspace*{-0.5em}
\section{Introduction}
\vspace*{-0.5em}
\label{sec:intro}

In this paper, we study multi-task bandit learning with the goal of learning an algorithm that discovers and exploits structure in a family of related tasks.
In multi-task bandit learning, we have multiple distinct bandit tasks for which we want to learn a policy. Though distinct, the tasks share some structure, which we hope to leverage to speed up learning on new instances in this task family.
Traditionally, the study of such structured bandit problems has relied on knowledge of the problem structure like linear bandits \citep{li2010contextual, abbasi2011improved, degenne2020gamification}, bilinear bandits \citep{jun2019bilinear}, hierarchical bandits \citep{hong2022deep, hong2022hierarchical}, Lipschitz bandits \citep{bubeck2008online, bubeck2011lipschitz, magureanu2014lipschitz},  other structured bandits settings \citep{riquelme2018deep, lattimore2019information, dong2021provable} and even linear and bilinear multi-task bandit settings \citep{yang2022nearly, du2023multi, mukherjee2023multi}.
When structure is unknown an alternative is to adopt sophisticated model classes, such as kernel machines or neural networks, exemplified by kernel or neural bandits \citep{valko2013finite, chowdhury2017kernelized, zhou2020neural, dai2022federated}. However, these approaches are also costly as they learn complex, nonlinear models from the ground up without any prior data \citep{justus2018predicting, zambaldi2018relational}.

 In this paper, we consider an alternative approach of synthesizing a bandit algorithm from historical data where the data comes from recorded bandit interactions with past instances of our target task family.
Concretely, we are given a set of state-action-reward tuples obtained by running some bandit algorithm in various instances from the task family.
We then aim to train a transformer \citep{vaswani2017attention} from this data such that it can learn in-context to solve new task instances.
\citet{laskin2022context} consider a similar goal and introduce the Algorithm Distillation (AD) method, however, AD aims to copy the algorithm used in the historical data and thus is limited by the ability of the data collection algorithm.
%
\citet{lee2023supervised} develop an approach, DPT, that enables learning a transformer that obtains lower regret in-context bandit learning compared to the algorithm used to produce the historical data. However, this approach requires knowledge of the optimal action at each stage of the decision process. In real problems, this assumption is hard to satisfy and we will show that DPT performs poorly when the optimal action is only approximately known.
With this past work in mind, the goal of this paper is to answer the question:
\begin{tcolorbox}
\begin{center}
    \textit{Can we learn an in-context bandit learning algorithm that obtains lower regret than the algorithm used to produce the training data without knowledge of the optimal action in each training task?}
\end{center}
\end{tcolorbox}    
To answer this question, we introduce a new pre-training methodology, called \textbf{Pre}-trained \textbf{De}cision \textbf{T}ransf\textbf{o}rmer with \textbf{R}eward Estimation (\pred) that obviates the need for knowledge of the optimal action in the in-context data –- a piece of information that is often inaccessible. 
%
%
Our key observation is that while the mean rewards of each action change from task to task, certain probabilistic dependencies are persistent across all tasks with a given structure \citep{yang2020impact, yang2022nearly, mukherjee2023multi}. These probabilistic dependencies can be learned from the pretraining data and exploited to better estimate mean rewards and improve performance in a new unknown test task.  
The nature of the probabilistic dependencies depends on the specific structure of the bandit and can be complex (i.e., higher-order dependencies beyond simple correlations).  
We propose to use transformer models as a general-purpose architecture to capture the unknown dependencies by training transformers to predict the mean rewards in each of the given trajectories \citep{mirchandani2023large, zhao2023expel}.  The key idea is that transformers have the capacity to discover and exploit complex dependencies in order to predict the rewards of all possible actions in each task from a \emph{small} history of action-reward pairs in a new task. 
This paper demonstrates how such an approach can achieve lower regret by outperforming state-of-the-art baselines, relying solely on historical data, without the need for any supplementary information like the action features or knowledge of the complex reward models. 
We also show that the shared actions across the tasks are vital for \pred\ to exploit the latent structure.  
We show that \pred\ learns to adapt, in-context, to novel actions and new tasks as long as the number of new actions is small compared to shared actions across the tasks.
%

\paragraph{Contributions}
\begin{enumerate}
    \item We introduce a new pre-training procedure, \pred, for learning the underlying reward structure and using this to circumvent the issue of requiring access to the optimal (or approximately optimal) action during training time.
    \item We demonstrate empirically that this training procedure results in lower regret in a wide series of tasks (such as linear, nonlinear, bilinear, and latent bandits) compared to prior in-context learning algorithms and bandit algorithms with privileged knowledge of the common structure.
    %
    %
    \item We also show that our training procedure leverages the shared latent structure. We systematically show that when the shared structure breaks down no reward structure or exploration is learned.
    \item Finally, we theoretically analyze the generalization ability of \pred\ through the lens of algorithmic stability and new results for the transformer setting.
\end{enumerate}
\section{Background}
\label{sec:prelim}
In this section, we first introduce our notation and the multi-task, structured bandit setting. 
We then formalize the in-context bandit learning model studied in \citet{laskin2022context, lee2023supervised, sinii2023context, lin2023transformers, ma2023rethinking, liu2023reason, liu2023self}. 
%


\subsection{Preliminaries}

In this paper, we consider the multi-task linear bandit setting \citep{du2023multi, yang2020impact, yang2022nearly}.  
In the multi-task setting, we have a family of related bandit problems that share an action set $\A$ and also a common action feature space $\X$.
%
The actions in $\A$ are indexed by $a=1,2,\ldots, A$. The feature of each action is denoted by $\bx(a) \in \R^d$ and $d \ll A$. 
%
%
A policy, $\pi$, is a probability distribution over the actions.

Define $[n] = \{1,2,\ldots,n\}$. 
In a multi-task structured bandit setting the expected reward for each action in each task is assumed to be an unknown function of the hidden parameter and action features \citep{lattimore2020bandit, gupta2020unified}.
The interaction proceeds iteratively over $n$ rounds for each task $m\in [M]$.
At each round $t\in [n]$ for each task $m\in [M]$, the learner selects an action $I_{m,t} \in \A$ and observes the reward $r_{m,t} = f(\bx(I_{m,t}), \btheta_{m,*}) + \eta_{m,t}$, where $\btheta_{m,*}\in\R^d$ is the hidden parameter specific to the task $m$ to be learned by the learner. The function $f(\cdot, \cdot)$ is the unknown reward structure. This can be $f(\bx(I_{m,t}), \btheta_{m,*}) = \bx(I_{m,t})^\top\btheta_{m,*}$ for the linear setting or even more complex correlation between features and $\btheta_{m,*}$ \citep{filippi2010parametric, abbasi2011improved,  riquelme2018deep, lattimore2019information, dong2021provable}.
%

In our paper, we assume that there exist weak demonstrators denoted by $\pi^{w}$. These weak demonstrators are stochastic $A$-armed bandit algorithms like Upper Confidence Bound (UCB)  \citep{auer2002finite-time, auer2010ucb} or Thompson Sampling \citep{thompson1933likelihood,agrawal2012analysis,russo2018tutorial,zhu2020thompson}.
We refer to these algorithms as weak demonstrators because they do not use knowledge of task structure or arm feature vectors to plan their sampling policy.
In contrast to a weak demonstrator, a strong demonstrator, like LinUCB, uses feature vectors and knowledge of task structure to conduct informative exploration.
Whereas weak demonstrators always exist, there are many real-world settings with no known strong demonstrator algorithm or where the feature vectors are unobserved and the learner can only use the history of rewards and actions. 
%
%
%
%

\vspace*{-1em}
\subsection{In-Context Learning Model}
\label{sec:learning-model}
\vspace*{-0.7em}

Similar to \citet{lee2023supervised, sinii2023context, lin2023transformers, ma2023rethinking, liu2023reason, liu2023self} we assume the in-context learning model. We first discuss the pretraining procedure.

\textbf{Pretraining:} Let $\cTp$ denote the distribution over tasks $m$ at the time of pretraining. 
%
Let $\Dpr$ be the distribution over all possible interactions that the $\pi^w$ can generate.
We first sample a task $m\sim \cT_{\text {pre }}$ and then a context $\H_m$ which is a sequence of interactions for $n$ rounds conditioned on the task $m$ such that $\H_m\sim\Dpr(\cdot|m)$.
%
%
So $\H_m = \{I_{m,t}, r_{m,t}\}_{t=1}^n$.
%
We call this dataset $\H_m$ an in-context dataset as it contains the contextual information about the task $m$. 
We denote the samples in $\H_m$ till round $t$ as $\H_{m}^t = \{I_{m,s}, r_{m,s}\}_{s=1}^{t-1}$. 
%
This dataset $\H_m$ can be collected in several ways: (1) random interactions within $m$, (2) demonstrations from an expert, and (3) rollouts of an algorithm. 
%
%
%
%
%
%
Finally, we train a causal GPT-2 transformer model $\T$ parameterized by $\bTheta$ on this dataset $\Dpr$. 
Specifically, we define $\T_{\bTheta}\left(\cdot \mid \H^t_m\right)$ as the transformer model that observes the dataset $\H^t_m$ till round $t$ and then produces a distribution over the actions.
Our primary novelty lies in our training procedure which we explain in detail in \Cref{sec:training}. 

\textbf{Testing:} We now discuss the testing procedure for our setting. Let $\cTs$ denote the distribution over test tasks $m\in [M_{\text {test }}]$ at the time of testing.
Let $\Dts$ denote a distribution over all possible interactions that can be generated by $\pi^w$ during test time. 
%
%
%
%
%
%
At deployment time, the dataset $\H_m^0 \leftarrow \{\emptyset\}$ is initialized empty. At each round $t$, an action is sampled from the trained transformer model $I_t \sim \T_{\bTheta}(\cdot \mid \H^t_m)$. The sampled action and resulting reward, $r_t$, are then added to $\H^t_m$ to form $\H^{t+1}_m$ and the process repeats for $n$ total rounds.
%
Finally, note that in this testing phase, the model parameter $\bTheta$ is not updated. 
%
Finally, the goal of the learner is to minimize cumulative regret for all task $m\in [\Mts]$ defined as follows: 
$
    \E[R_n] = \frac{1}{\Mts}\sum_{m=1}^{\Mts}\sum_{t=1}^n\max_{a\in\A}f\left( \bx(a), \btheta_{m,*} \right) - f\left( \bx(I_t), \btheta_{m,*} \right)$.

\vspace*{-0.5em}
\subsection{Related In-context Learning Algorithms}
\label{sec:prelim-dpt}
\vspace*{-0.5em}
In this section, we discuss related algorithms for in-context decision-making.
%
For completeness, we describe the \dpt\ and \ad\ training procedure and algorithm now.
During training, \dpt\ first samples $m\sim\cTp$ and then an in-context dataset $\H_m\sim \Dpr(\cdot|,m)$. It adds this $\H_m$ to the training dataset $\Htr$, and repeats to collect $\Mpr$ such training tasks. 
For each task $m$, \dpt\  requires the optimal action $a_{m,*} = \argmax_a f(\bx(m, a),  \btheta_{m, *})$ where $ f(\bx(m, a),  \btheta_{m, *})$ is the expected reward for the action $a$ in task $m$. 
Since the optimal action is usually not known in advance, in \Cref{sec:short-horizon} we introduce a practical variant of \dpt\  that approximates the optimal action with the best action identified during task interaction.
During training \dpt\ minimizes the cross-entropy loss:
\begin{align}
    \L^{\mathrm{\dpt}}_t = \operatorname{cross-entropy}(\T_{\bTheta}(\cdot|\H_m^t), p({a}_{m,*})) \label{eq:loss-dptg}
\end{align}
where $p({a}_{m,*})\!\in \!\triangle^{\A}$ is a one-hot vector such that $p(j) \!=\! 1$ when $j\!=\!{a}_{m,*}$ and $0$ otherwise. This loss is then back-propagated and used to update the model parameter $\bTheta$.

During test time evaluation for online setting the \dpt\ selects $I_t \sim \mathrm{softmax}^\tau_a(\T_{\bTheta}(\cdot|\H^t_m))$ where we define the $\mathrm{softmax}^\tau_a(\bv)$ over a $A$ dimensional vector $\bv\in \R^A$ as $ \mathrm{softmax}^\tau_a(\bv(a)) = \exp(\bv(a)/\tau)/\sum_{a'=1}^A \exp(\bv(a')/\tau)$ which produces a distribution over actions weighted by the temperature parameter $\tau > 0$. Therefore this sampling procedure has a high probability of choosing the predicted optimal action as well as induce sufficient exploration. 
%
%
In the online setting, the \dpt\ observes the reward $r_t(I_t)$ which is added to $\H^t_m$. So the $\H_m$ during online testing consists of $\{I_t, r_t\}_{t=1}^n$ collected during testing. This interaction procedure is conducted for each test task $m\in [M_{\text {test }}]$.
In the testing phase, the model parameter $\bTheta$ is not updated. 

An alternative to \dpt\ that does \textit{not} require knowledge of the optimal action is the \ad\ approach \citep{laskin2022context, lu2023structured}. In \ad, the learner aims 
to predict the next action of the demonstrator. So it
minimizes the cross-entropy loss as follows:
\begin{align}
    \L^{\mathrm{\ad}}_t = \operatorname{cross-entropy}(\T_\bTheta(\cdot|\H_m^t), p({I}_{m,t})) \label{eq:loss-AD}
\end{align}
where  $p({I}_{m,t})$ is a one-hot vector such that $p(j) = 1$ when $j={I}_{m,t}$ (the true action taken by the demonstrator) and $0$ otherwise. 
At deployment time, \ad\ selects $I_t \sim \mathrm{softmax}^\tau_a(\T_{\bTheta}(\cdot|\H^t_m))$.
%
The objective of \ad\ is to match the performance of the demonstrator.
%
%
In the next section, we introduce a new method that can improve upon the demonstrator without knowledge of the optimal action.

\section{The \pred Algorithm}
\label{sec:algo}






We now introduce our main algorithmic contribution, \pred\ (which stands for \textbf{Pre}-trained \textbf{De}cision \textbf{T}ransf\textbf{o}rmer with \textbf{R}eward Estimation). 
%
%
%

\vspace*{-0.8em}
\subsection{Pre-training Next Reward Prediction}
\label{sec:training}
\vspace*{-0.5em}
The key idea behind \pred\ is to leverage the in-context learning ability of transformers to infer the reward of each arm in a given test task.
By training this in-context ability on a set of training tasks, the transformer can implicitly learn structure in the task family and exploit this structure to infer rewards without trying every single arm.
\textcolor{blue}{Hence \pred\ requires access to all the finite set of arms. Note that the \ad, and \dpt\ only require access to the arms selected by the demonstrator.}
In contrast to \dpt\ and \ad\ that output actions directly, \pred\ outputs a scalar value reward prediction for each arm.
%
%
%
To this effect, we append a linear layer of dimension $A$ on top of a causal GPT2 model, denoted by $\rT_{\bTheta}(\cdot | \H_m)$, and use a least-squares loss to train the transformer to predict the reward for each action with these outputs. 
Note that we use $\rT_{\bTheta}(\cdot | \H_m)$ to denote a reward prediction transformer and $\T_{\bTheta}(\cdot | \H_m)$ as the transformer that predicts a distribution over actions (as in \dpt\ and \ad\ ).
%
%
%
%
At every round $t$ the transformer predicts the \emph{next reward} for each of the actions $a\in \A$ for the task $m$ based on $\H_m^t = \{I_{m,s}, r_{m,s}\}_{s=1}^{t-1}$. This predicted reward is denoted by $\wr_{m,t+1}(a)$ for each $a\in\A$. 



\textbf{Loss calculation:} For each training task, $m$, we calculate the loss at each round, $t$, using the transformer's prediction $\hat{r}_{m,t}(I_{m,t})$ and the actual observed reward $r_{m,t}$ that followed action $I_{m,t}$. 
%
%
We use a least-squares loss function:
\begin{align}
    \L_t = \left(\wr_{m,t}(I_{m, t}) -  r_{m,t}\right)^2 \label{eq:loss-transformer}
\end{align}
and hence minimizing this loss will minimize the mean squared-error of the transformer's predictions.
The loss is calculated using \eqref{eq:loss-transformer} and is backpropagated to update the model parameter $\bTheta$.

\textbf{Exploratory Demonstrator:} Observe from the loss definition in \eqref{eq:loss-transformer} that it is calculated from the observed true reward and action from the dataset $\H_m$. 
In order for the transformer to learn accurate reward predictions during training, we require that the weak demonstrator is sufficiently exploratory such that it collects $\H_m$ such that $\H_m$ contains some reward $r_{m,t}$ for each action $a$. We discuss in detail the impact of the demonstrator on \pred\ (\gt) training in \Cref{sec:data-collection}. 

\vspace*{-0.5em}
\subsection{Deploying \pred\ }
\vspace*{-0.5em}
At deployment time, \pred\ learns in-context to predict the mean reward of each arm on an unseen task and acts greedily with respect to this prediction.
%
%
%
That is, at deployment time, a new task is sampled, $m \sim \cTs$, and the dataset $\H_m^0$ is initialized empty. Then at every round $t$,  
\pred\ chooses $I_t = \argmax_{a\in\A} \rT_{\bTheta}\left( \wr_{m,t}(a) \mid \H^t_m\right)$ which is the action with the highest predicted reward and $\wr_{m,t}(a)$ is the predicted reward of action $a$. 
Note that \pred\ is a greedy policy and thus may fail to conduct sufficient exploration. To remedy this potential limitation, we also introduce a soft variant, \predt\, that chooses $I_t \sim \mathrm{softmax}^\tau_a\left(\rT_{\bTheta}\left( \mathbf{\wr}_{m,t}(a) \mid \H^t_m\right)\right)$.
For both \pred\ and \predt, the observed reward $r_t(I_t)$ is added to the dataset $\H_m$ and then used to predict the reward at the next round $t+1$. The full pseudocode of using \pred\ for online interaction is shown in \Cref{alg:pred}.
In \Cref{sec:offline}, we discuss how \pred\ (\gt) can be deployed for offline learning. We also highlight that \pred\ needs to forward-pass $|A|$ times to select the best arm during evaluation, and hence suffers from more computational overhead with a large action space, compared to \ad\ or \dpt.
%
%
\begin{algorithm}[!tbh]
\caption{\textbf{Pre}-trained \textbf{De}cision \textbf{T}ransf\textbf{o}rmer with \textbf{R}eward Estimation (\pred)}
\label{alg:pred}
    \begin{algorithmic}[1]
    \STATE \textbf{Collecting Pretraining Dataset} 
    \STATE Initialize empty pretraining dataset $\Htr$
    \FOR{$i$ in $[\Mpr]$ }
    \STATE Sample task $m \sim \cTp$, in-context dataset $\H_m \sim \Dpr(\cdot | m)$ and add this to $\Htr$.
    \ENDFOR
    \STATE \textbf{Pretraining model on dataset}
    \STATE Initialize model $\rT_{\bTheta}$ with parameters $\bTheta$
    \WHILE{\text{not converged}}
    \STATE Sample $\H_m$ from $\Htr$ and predict $\wr_{m,t}$ for action $(I_{m,t})$ for all $t \in[n]$
    \STATE Compute loss in \eqref{eq:loss-transformer} with respect to $r_{m,t}$ and backpropagate to update model parameter $\bTheta$.
    \ENDWHILE
    %
    \STATE \textbf{Online test-time deployment}
    \STATE Sample unknown task $m \sim \cTs$ and initialize empty $\H_m^{0}=\{\emptyset\}$
    \FOR{$t=1,2,\ldots,n$}
    \STATE Use $\rT_{\bTheta}$ on $m$ at round $t$ to choose 
    \begin{align*}
        I_t \begin{cases}
            = \argmax_{a\in\A} \rT_{\bTheta}\left( \wr_{m,t}(a) \mid \H^t_m\right), & \textbf{\pred} \\
            \sim \textrm{softmax}^\tau_a \rT_{\bTheta}\left( \wr_{m,t}(a) \mid \H^t_m\right), & \textbf{\predt}
        \end{cases}
    \end{align*}
    %
    \STATE Add $\left\{I_t, r_t\right\}$ to $\H_m^t$ to form $\H_m^{t+1}$.
    \ENDFOR
    \end{algorithmic}
\end{algorithm}
\vspace*{-0.5em}

\vspace*{-0.3em}
\section{Empirical Study: Non-Linear Structure}
\vspace*{-0.3em}
\label{sec:short-horizon}
Having introduced \pred, we now investigate its performance in diverse bandit settings compared to other in-context learning algorithms.
In our first set of experiments, we use a bandit setting with a common non-linear structure across tasks.
Ideally, a good learner would leverage the structure, however, we choose the structure such that no existing algorithms are well-suited to the non-linear structure.
This setting is thus a good testbed for establishing that in-context learning can discover and exploit common structure.
%
Moreover, each task only consists of a few rounds of interactions.
This setting is quite common in recommender settings where user interaction with the system lasts only for a few rounds and has an underlying non-linear structure \citep{kwon2022tractable, tomkins2020rapidly}.
We show that \pred\ achieves lower regret than other in-context algorithms for the non-linear structured bandit setting. 
%
%
We study the performance of \pred\ in the large horizon setting in \Cref{sec:horizon}.

\textbf{Baselines:} We first discuss the baselines used in this setting.

\noindent
\textbf{(1) \pred:} This is our proposed method shown in \Cref{alg:pred}.

\textbf{(2) \predt:} This is the proposed exploratory method shown in \Cref{alg:pred} and we fix $\tau=0.05$. 

\textbf{(3) \dptg:} This baseline is the greedy approximation of the \dpt\ algorithm from \citet{lee2023supervised} which is discussed in \Cref{sec:prelim-dpt}.
Note that we choose \dptg\ as a representative example of similar in-context decision-making algorithms studied in \citet{lee2023supervised, sinii2023context, lin2023transformers, ma2023rethinking, liu2023reason, liu2023self} all of which require the optimal action (or its greedy approximation). 
%
\dptg\ estimates the optimal arm using the reward estimates for each arm during each task. 
%
 
\textbf{(4) \ad:} This is the Algorithmic Distillation method \citep{laskin2022context, lu2021low} discussed in \Cref{sec:prelim-dpt}. 

\textbf{(5) \ts:} This baseline is the celebrated stochastic $A$-action bandit Thompson Sampling algorithm from \citet{thompson1933likelihood,agrawal2012analysis,russo2018tutorial,zhu2020thompson}.
We choose \ts\ as the weak demonstrator $\pi^w$ as it does not make use of arm features.
\ts\ is also a stochastic algorithm that induces more exploration in the demonstrations. 

\textbf{(6) \linucb:} (Linear Upper Confidence Bound): This baseline is the Upper Confidence Bound algorithm for the linear bandit setting that leverages the linear structure and feature of the arms to select the most promising action as well as conducting exploration. We choose \linucb\ as a baseline for each test task to show the limitations of algorithms that use linear feedback structure as an underlying assumption to select actions. Note that \linucb\ requires oracle access to features to select actions per task.

\textbf{(7) \mlin:} This is the multi-task linear regression bandit algorithm proposed by \citet{yang2021impact}. This algorithm assumes that there is a common low-dimensional feature extractor shared between the tasks and the reward of each task linearly depends on this feature extractor. We choose \mlin\ as a baseline to show the limitations of algorithms that use linear feedback structure \textit{across tasks} as an underlying assumption to select actions. Note that \mlin\ requires oracle access to the action features to select actions as opposed to \dpt, \ad, and \pred.

We describe in detail the baselines \ts, \linucb, and \mlin\ for interested readers in \Cref{app:baseline-details}.

\textbf{Outcomes:} First, we discuss the main outcomes from our experimental results in this section:

\begin{tcolorbox}
\customfinding \pred\ (\gt)  lowers regret compared to other baselines under unknown, non-linear structure. It learns to exploit the latent structure of the underlying tasks from in-context data even when it is trained without the optimal action $a_{m,*}$ (or its approximation) and without action features $\X$.
\end{tcolorbox}




\textbf{Experimental Result:} These findings are reported in \Cref{fig:expt-short-horizon}.
%
%
%
%
In \Cref{fig:short-horizon-nlm} we show the non-linear bandit setting for horizon $n=50$, $\Mpr = 100000$, $\Mts = 200$, $A=6$, and $d=2$. The demonstrator $\pi^w$ is the \ts\ algorithm. We observe that \pred\ (\gt) has lower cumulative regret than \dptg. 
Note that for this low data regime (short horizon) the \dptg\ does not have a good estimation of $\widehat{a}_{m,*}$ which results in a poor prediction of optimal action $\widehat{a}_{m,t,*}$. This results in higher regret.
The \pred\ (\gt) has lower regret than \linucb, and \mlin, which fail to perform well in this non-linear setting due to their algorithmic design and linear feedback assumption. Finally, \predt\ performs slightly better than \pred\ in both settings as it conducts more exploration. 

In \Cref{fig:short-horizon-nlm-feature} we show the non-linear bandit setting for horizon $n=25$, $\Mpr = 100000$, $\Mts = 200$, $A=6$, and $d=2$ where the norm of the $\btheta_{m,*}$ determines the reward of the actions which also is a non-linear function $\btheta_{m,*}$ and action features. This setting is similar to the wheel bandit setting of \citet{riquelme2018deep}. Again, we observe that \pred\ has lower cumulative regret than all the other baselines.

Finally in \Cref{fig:short-horizon-movielens} and \Cref{fig:short-horizon-yelp} we show the performance of \pred\ against other baselines in real-world datasets Movielens and Yelp. The Movielens dataset consists of more than 32 million ratings of 200,000 users and 80,000 movies \citep{harper2015movielens} where each entry consists of user-id, movie-id, rating, and timestamp. The Yelp dataset \citep{asghar2016yelp} consists of ratings of 1300 business categories by 150,000 users. Each entry is summarized as user-id, business-id, rating, and timestamp. Previously structured bandit works \citep{deshpande2012linear, hong2023multi} directly fit a linear structure or low-rank factorization to estimate the $\btheta_{m,*}$ and simulate the ratings. However, we directly use the user-ids and movie-ids (or business-ids) to build a histogram of ratings per user and calculate the mean rating per movie (or business-id) per task. Define this as the $\{\mu_{m,a}\}_{a=1}^A$. This is then used to simulate the rating for $n$ horizon per movie per task where the data collection algorithm is uniform sampling. 
Note that this does not require estimation of user or movie features, and \pred\ (\gt) learns to exploit the latent structure of user-movie (or business) rating correlations directly from the data.
From \Cref{fig:short-horizon-movielens} and \Cref{fig:short-horizon-yelp} we see that \pred, and \predt\ outperform all the other baselines in these settings. In the next section, we study the simplified linear setting to show that \pred\ is indeed exploiting the latent structure to minimize the cumulative regret.

\begin{figure}[!hbt]
\centering
\vspace*{-1em}
\begin{subfigure}[b]{0.25\textwidth}
    \includegraphics[scale=0.1]{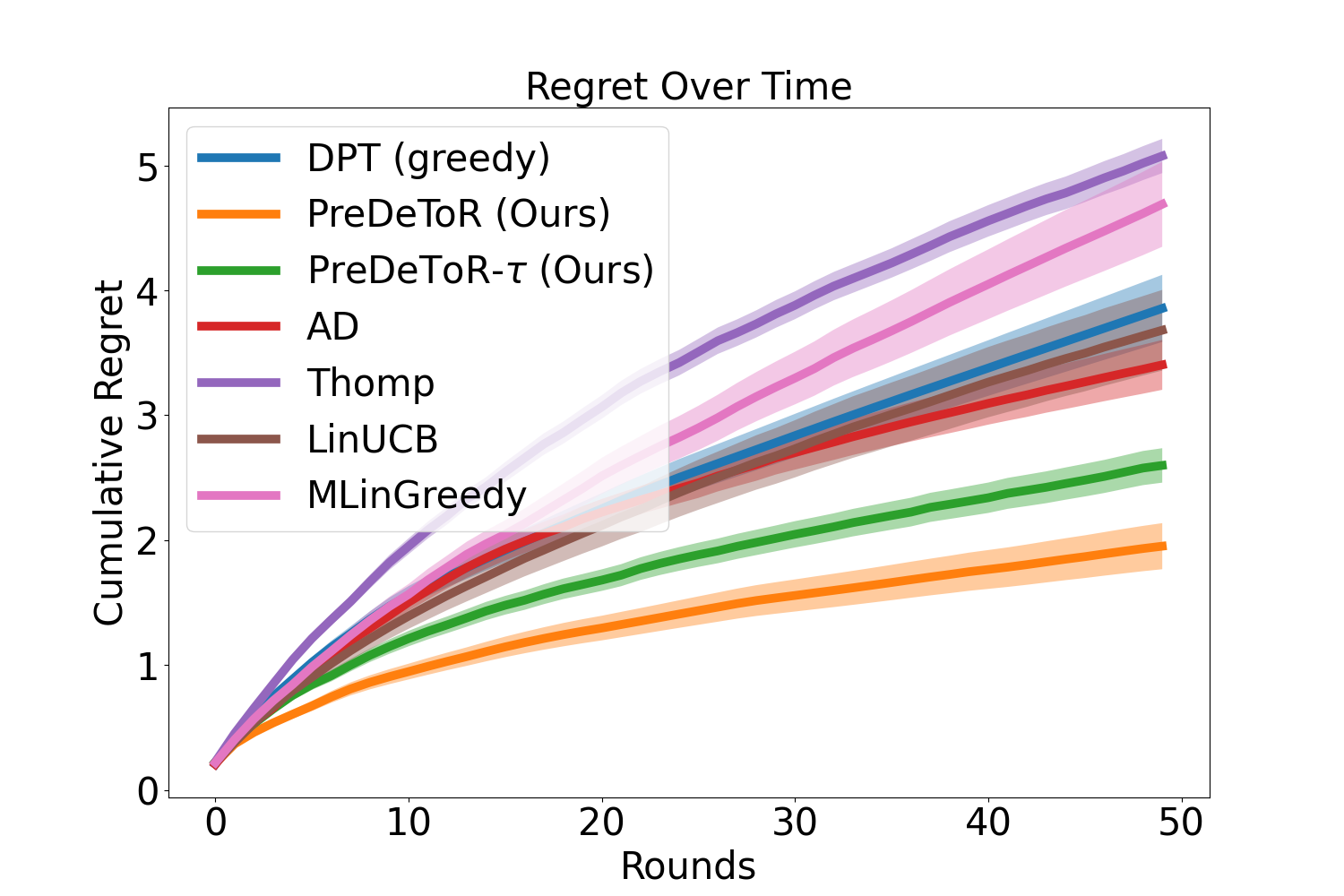}
    \caption{Non-linear bandit}
    \label{fig:short-horizon-nlm}
\end{subfigure}%
\begin{subfigure}[b]{0.25\textwidth}
    \includegraphics[scale=0.1]{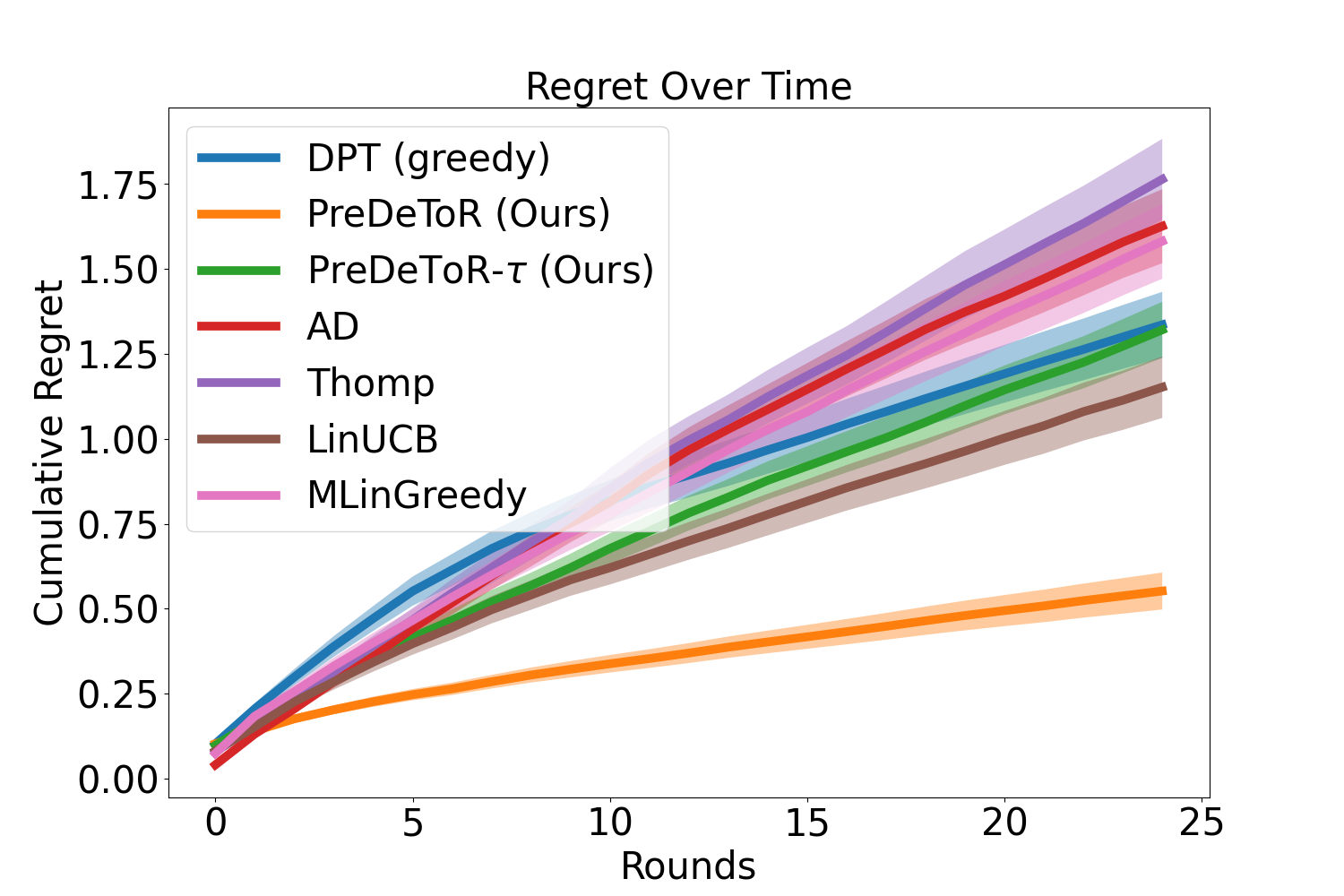}
    \caption{Feature bandit}
    \label{fig:short-horizon-nlm-feature}
\end{subfigure}%
\begin{subfigure}[b]{0.25\textwidth}
    \includegraphics[scale=0.1]{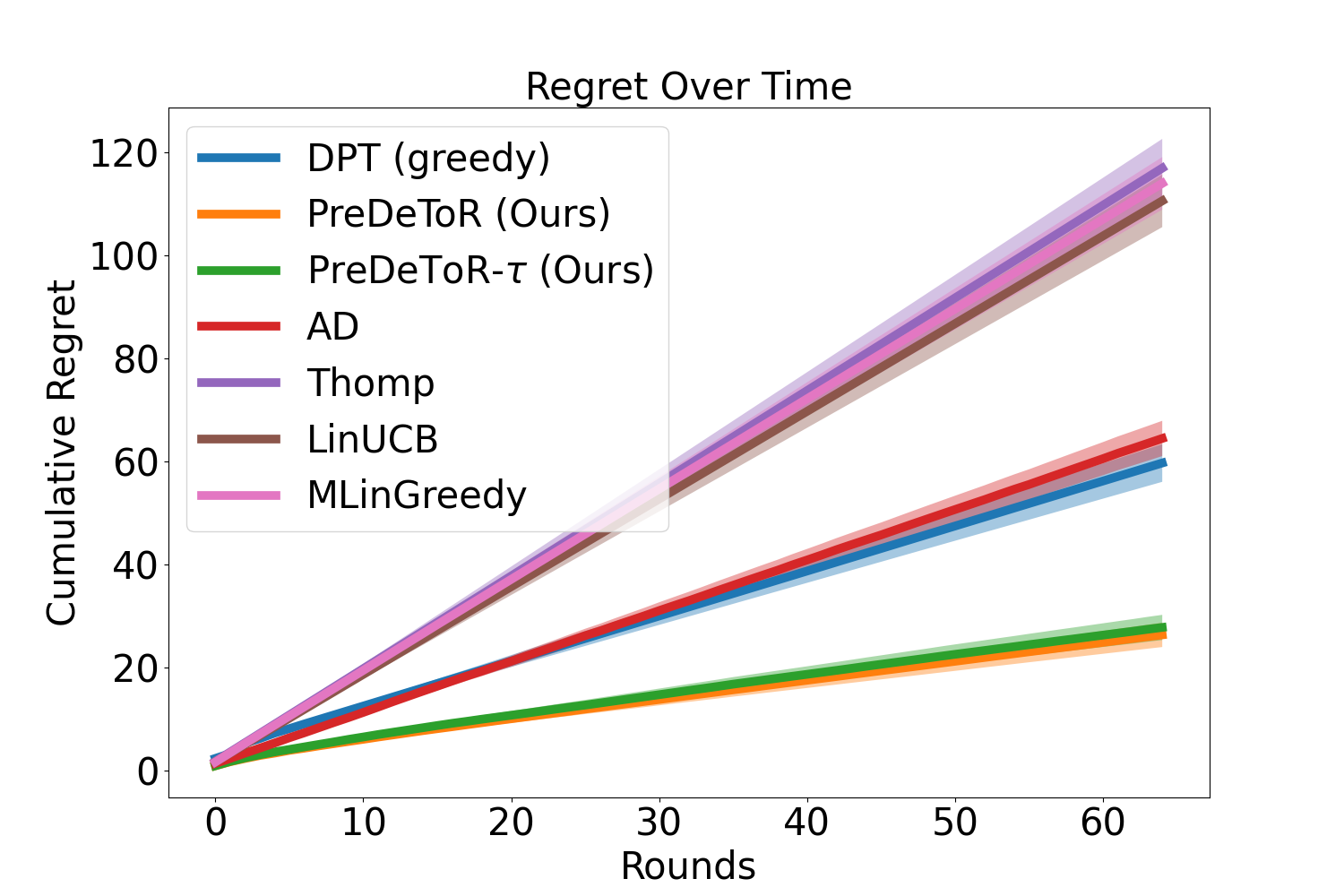}
    \caption{Movielens}
    \label{fig:short-horizon-movielens}
\end{subfigure}%
\begin{subfigure}[b]{0.25\textwidth}
    \includegraphics[scale=0.1]{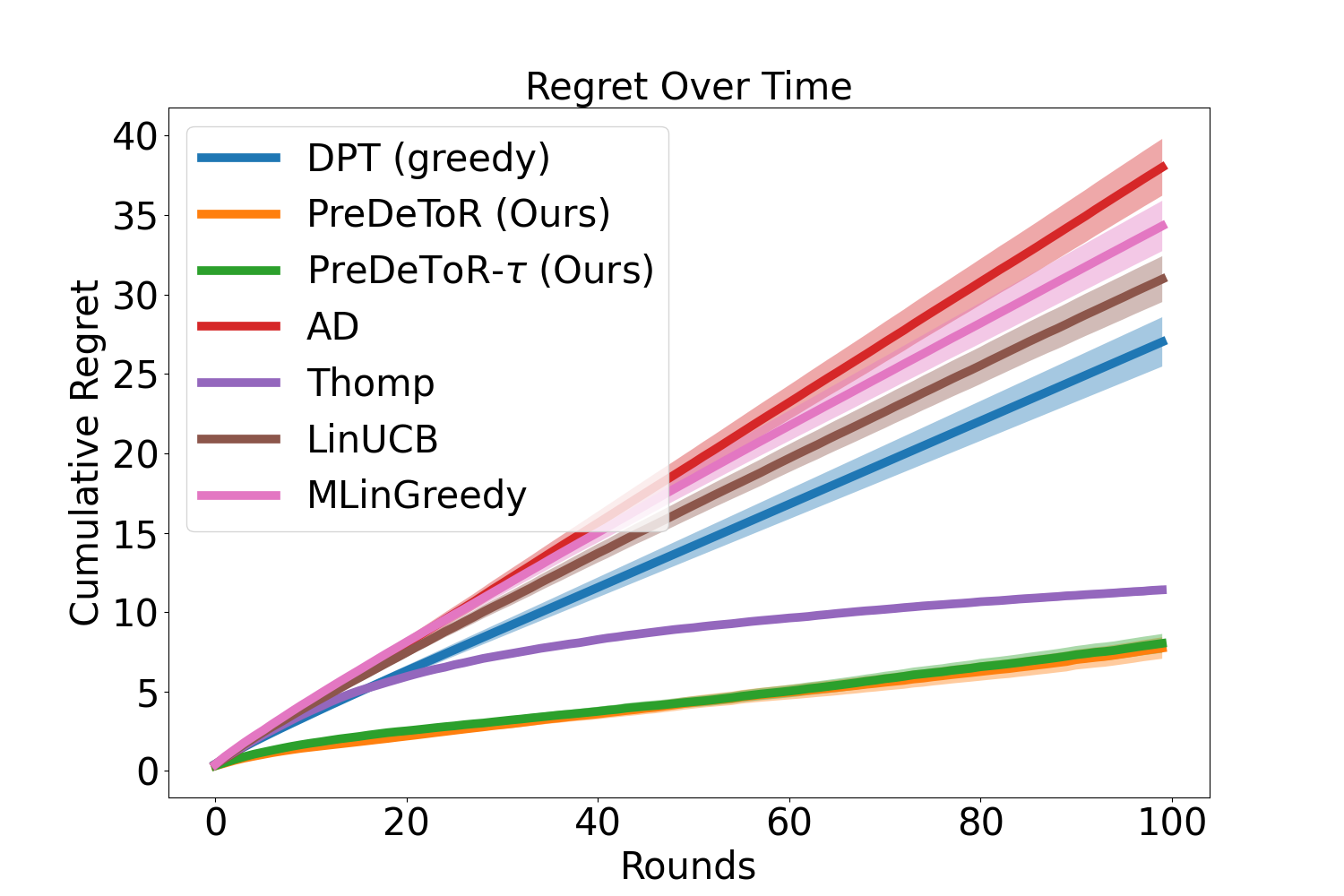}
    \caption{Yelp}
    \label{fig:short-horizon-yelp}
\end{subfigure}
\vspace*{-1em}
\caption{Non-linear regime. The horizontal axis is the number of rounds. Confidence bars show one standard error.}
\label{fig:expt-short-horizon}
\vspace{-0.7em}
\end{figure}

\vspace*{-0.5em}
\section{Empirical Study: Linear Structure and Understanding \pred's Exploration}
\label{sec:linear}
\vspace*{-0.5em}

The previous experiments were conducted in a non-linear structured setting where we are unaware of a provably near-optimal algorithm.
To assess how close \pred's regret is to optimal, in this section, we consider a \textit{linear} setting for which there exist well-understood algorithms \citep{abbasi2011improved,lattimore2020bandit}.
%
Such algorithms provide a strong upper bound for \pred. We summarize the key finding below:

\begin{tcolorbox}
\customfinding \pred\ (\gt) matches the performance of the optimal algorithm \linucb\ in linear bandit setting as it learns to exploit the latent structure across tasks from in-context data and without access to features.
\end{tcolorbox}

In \Cref{fig:short-horizon-lin} we first show the linear bandit setting for horizon $n=25$, $\Mpr = 200000$, $\Mts = 200$, $A=10$, and $d=2$. 
Note that the length of the context (the number of rounds) is an artifact of the transformer architecture and computational complexity. This is because the self-attention takes in as input a length-$n$ sequence of tokens of size $d$, and requires $O\left(d n^2\right)$ time to compute the output \citep{keles2023computational}. 
Further empirical setting details are stated in \Cref{sec:addl-expt}.

We observe from \Cref{fig:short-horizon-lin}  that \pred\ (\gt) has lower cumulative regret than \dptg, and \ad. Note that for this low data (short horizon) regime, the \dptg\ does not have a good estimation of $\widehat{a}_{m,*}$ which results in a poor prediction of optimal action $\widehat{a}_{m,t,*}$. This results in higher regret. Observe that \pred\ (\gt) performs quite similarly to \linucb\ and lowers regret compared to \ts\ which also shows that \pred\ is able to exploit the latent linear structure and reward correlation of the underlying tasks.
Note that \linucb\ is close to the optimal algorithm for this linear bandit setting. 
%
\pred\ outperforms \ad\ as the main objective of \ad\ is to match the performance of its demonstrator.
In this short horizon, we see that \mlin\ performs similarly to \linucb. 

We also show how the prediction error of the optimal action by \pred\ is small compared to \linucb\ in the linear bandit setting. In \Cref{fig:lin-action-dist} we first show how the $10$ actions are distributed in the $\Mts=200$ test tasks. In \Cref{fig:lin-action-dist} for each bar, the frequency indicates the number of tasks where the action (shown in the x-axis) is the optimal action. Then, in \Cref{fig:lin-action-error}, we show the prediction error of \pred\ (\gt) for each task $m\in[\Mts]$. The prediction error is calculated as $(\wmu_{m,n, *}(a) - \mu_{m, *}(a))^2$ where $\wmu_{m,n, *}(a) = \max_a\wtheta_{m,n}^\top\bx_m(a)$ is the empirical mean at the end of round $n$, and $\mu_{*,m}(a)=\max_a\btheta_{m,*}^\top\bx_m(a)$ is the true mean of the optimal action in task $m$. Then we average the prediction error for the action $a\in [A]$ by the number of times the action $a$ is the optimal action in some task $m$. 
From the \Cref{fig:lin-action-error}, we see that for actions $\{2,3,5,6,7,10\}$, the prediction error of \pred\ is either close or smaller than \linucb. 
Note that \linucb\ estimates the empirical mean directly from the test task, whereas \pred\ has a strong prior based on the training data. So \pred\ is able to estimate the reward of the optimal action quite well from the training dataset $\Dpr$.
This shows the power of \pred\ to go beyond the in-context decision-making setting studied in \citet{lee2023supervised, lin2023transformers, ma2023rethinking, sinii2023context, liu2023reason} which require long horizons/trajectories and optimal action during training to learn a near-optimal policy. 

%


\begin{figure}[!hbt]
\centering
\vspace*{-1em}
\begin{subfigure}[b]{0.32\textwidth}
    \includegraphics[scale=0.1]{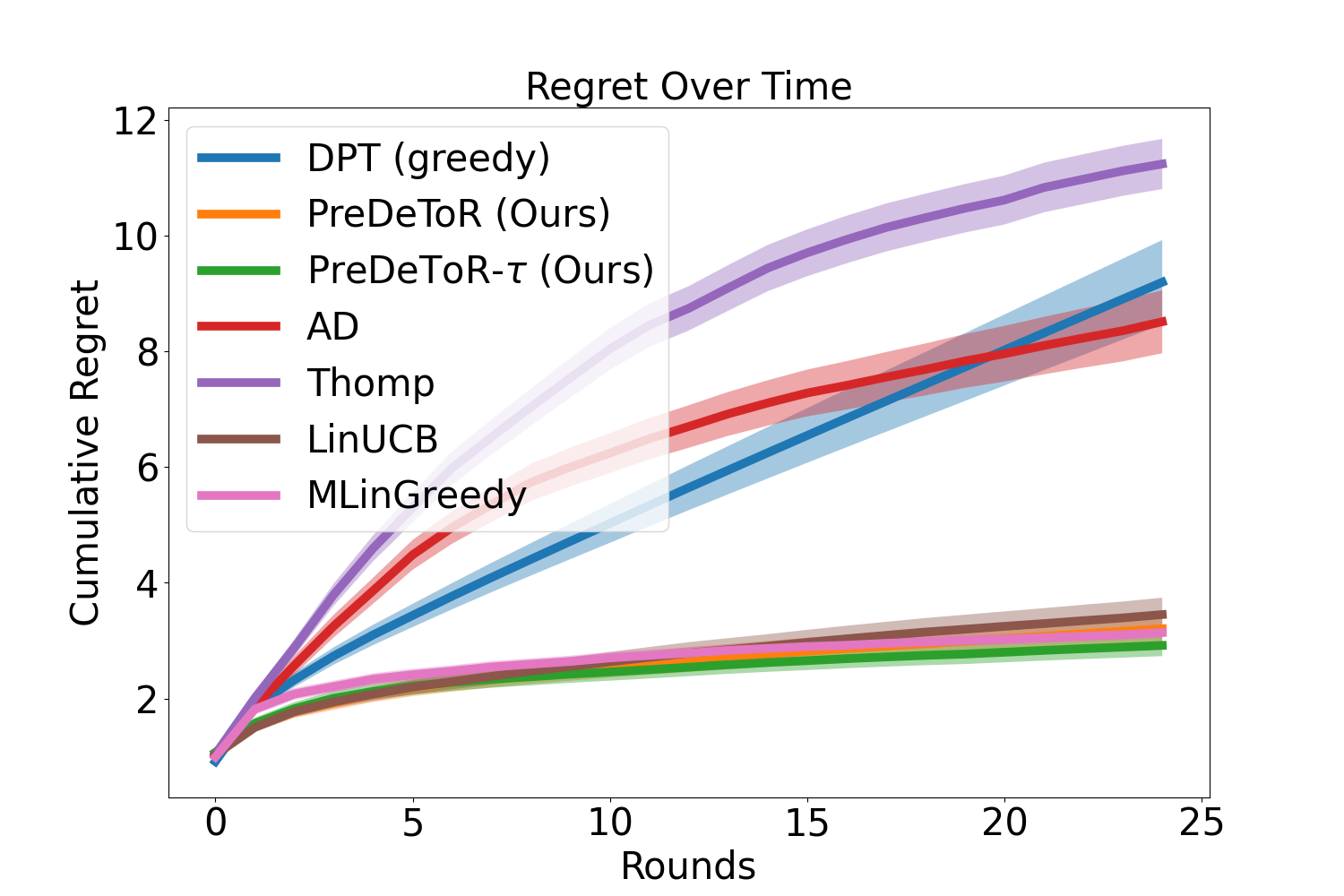}
    \caption{Linear Bandit setting}
    \label{fig:lin}
\end{subfigure}%
\begin{subfigure}[b]{0.32\textwidth}
    \includegraphics[scale=0.21]{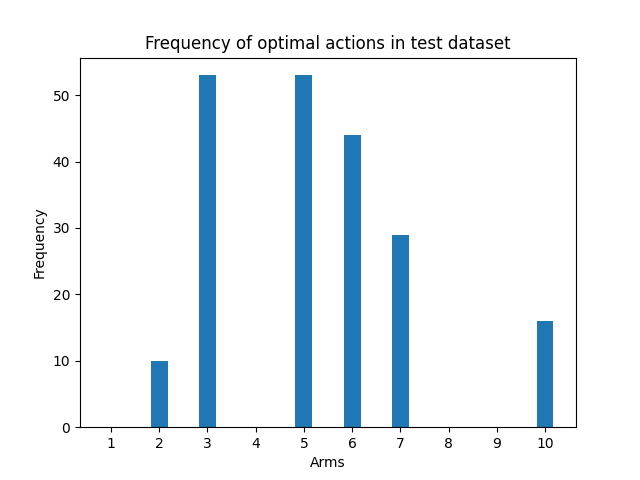}
    \caption{Test action distribution}
    \label{fig:lin-action-dist}
\end{subfigure}%
\begin{subfigure}[b]{0.32\textwidth}
    \includegraphics[scale=0.21]{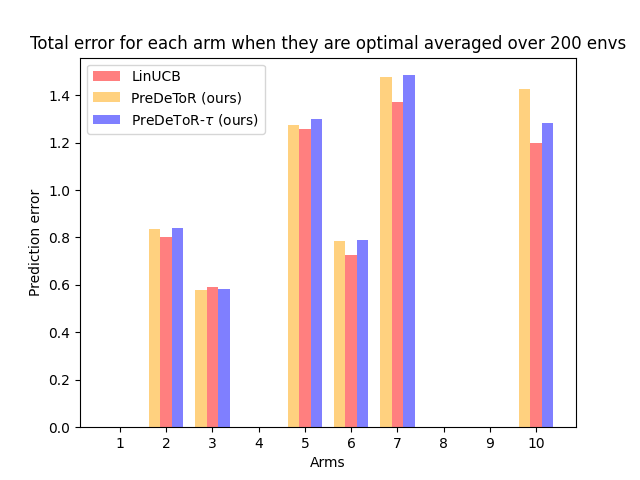}
    \caption{Test Prediction Error}
    \label{fig:lin-action-error}
\end{subfigure}
\vspace*{-1em}
\caption{Linear Expt. The horizontal axis is the number of rounds. Confidence bars show one standard error.}
\label{fig:short-horizon-lin}
\vspace{-0.7em}
\end{figure}



We now state the main finding of our analysis of exploration in the linear bandit setting:

\begin{tcolorbox}
\customfinding The \pred\ (\gt) has an implicit two-phase exploration. In the first phase, it explores with a strong prior over the in-context training data. In the second phase, once the task data has been observed for a few rounds (in-context) it switches to task-based exploration.
\end{tcolorbox}



We first show in \Cref{fig:train-dist} the training distribution of the optimal actions. For each bar, the frequency indicates the number of tasks where the action (shown in the x-axis) is the optimal action.
Then in \Cref{fig:analysis-epxploration} we show how the sampling distribution of \dptg, \pred\, and \predt\ change in the first $10$ and last $10$ rounds for all the tasks where action $5$ is optimal. To plot this graph we first sum over the individual pulls of the action taken by each algorithm over the first $10$ and last $10$ rounds. Then we average these counts over all test tasks where action $5$ is optimal. From the figure \Cref{fig:analysis-epxploration} we see that \pred (\gt) consistently pulls the action $5$ more than \dptg. It also explores other optimal actions like $\{2,3,6,7,10\}$ but discards them quickly in favor of the optimal action $5$ in these tasks. This shows that \pred\ (\gt) only considers the optimal actions seen from the training data. Once sufficient observation have been observed for the task it switches to task-based exploration and samples the optimal action more than \dptg.

Finally, we plot the feasible action set considered by \dptg, \pred, and \predt\ in \Cref{fig:analysis-exploration-time}. To plot this graph again we consider the test tasks where the optimal action is $5$. Then we count the number of distinct actions that are taken from round $t$ up until horizon $n$. Finally we average this over all the considered tasks where the optimal action is $5$. We call this the candidate action set considered by the algorithm. From the \Cref{fig:analysis-exploration-time} we see that \dptg\ explores the least and gets stuck with few actions quickly (by round $10$). Note that the actions \dptg\ samples are sub-optimal and so it suffers a high cumulative regret (see \Cref{fig:short-horizon-lin}).  \pred\ explore slightly more than \dptg, but \predt\ explores the most. 
%
%
\begin{figure}[!hbt]
\vspace*{-1em}
\centering
\begin{subfigure}[b]{0.27\textwidth}
   \includegraphics[scale=0.25]{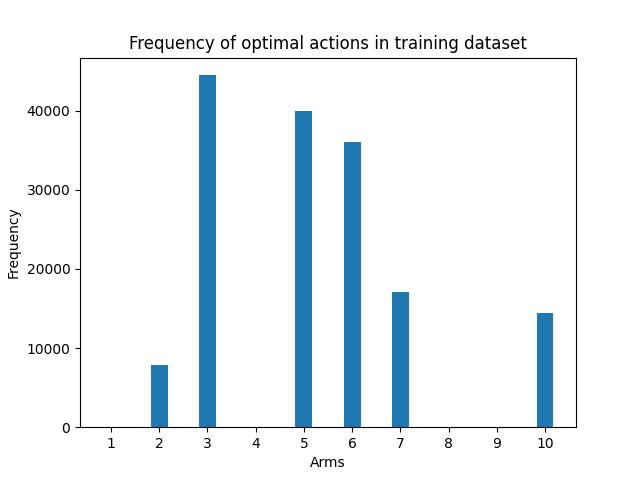}
   \caption{Train Optimal Action Distribution}
   \label{fig:train-dist}
\end{subfigure}%
\hspace*{1em}\begin{subfigure}[b]{0.27\textwidth}
   \includegraphics[scale=0.25]{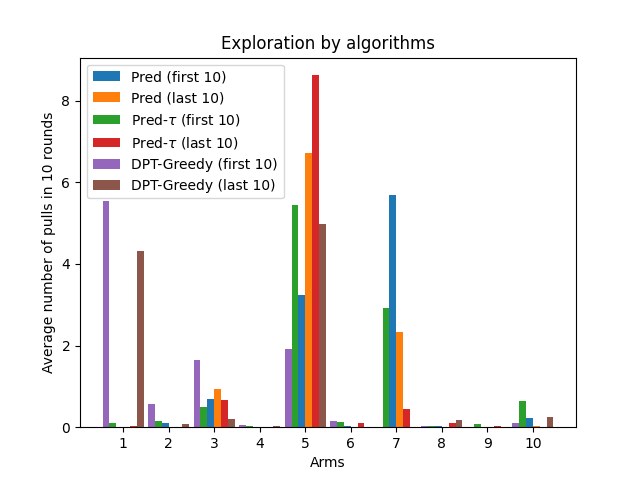}
   \caption{Distribution of action sampling in all test tasks where action $5$ is optimal}
   \label{fig:analysis-epxploration}
\end{subfigure}%
\hspace*{1em}\begin{subfigure}[b]{0.27\textwidth}
   \includegraphics[scale=0.25]{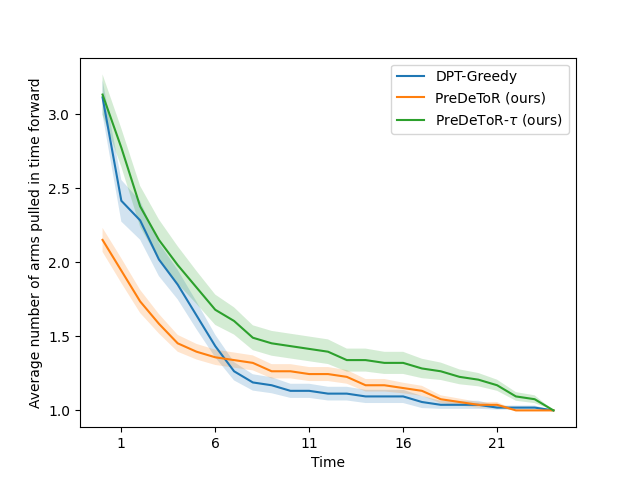}
   \caption{Candidate Action Set in Time averaged over all tasks where action $5$ is optimal}
   \label{fig:analysis-exploration-time}
\end{subfigure}
\caption{Exploration Analysis of \pred (\gt)}
\label{fig:exploration-analysis-time}
\end{figure}
\vspace*{-1em}

\vspace*{-0.5em}
\section{Empirical Study: Importance of Shared Structure and Introducing New Actions}
\label{sec:new-actions}
\vspace*{-0.5em}
One of our central claims is that \pred\ (\gt) internally learns and leverages the shared structure across the training and testing tasks.
To validate this claim, in this section, we consider the introduction of new actions at test time that do \textit{not} follow the structure of training time.
%
These experiments are particularly important as they show the extent to which \pred (\gt) is leveraging the latent structure and the shared correlation between the actions and rewards.

\textbf{Invariant actions:} We denote the set of actions fixed across the different tasks in the pretraining in-context dataset as $\Anc$. Therefore these action features $\bx(a)\in\R^d$ for $a\in \Anc$ are fixed across the different tasks $m$. Note that these invariant actions help the transformer $\T_{\bw}$ to learn the latent structure and the reward correlation across the different tasks. Therefore, as the structure breaks down, \pred\ starts performing worse than other baselines.

\textbf{New actions:} We also want to test whether \pred\ (\gt) exploits shared structure when new actions are introduced that are not seen during training time. To this effect, for each task $m\in [\Mpr]$ and $m\in [\Mts]$ we introduce $A - |\Anc|$ new actions. \textit{That is both for train and test tasks, we introduce new actions.}  For each of these new actions $a\in [A - |\Anc|]$ we choose the features $\bx(m,a)$ randomly from $\X\subseteq\R^d$. Note the transformer now trains on a dataset $\H_m \subseteq \Dpr \neq \Dts$.

\begin{figure}[!hbt]
\centering
\vspace*{-1em}
\begin{subfigure}[b]{0.24\textwidth}
    \includegraphics[scale=0.1]{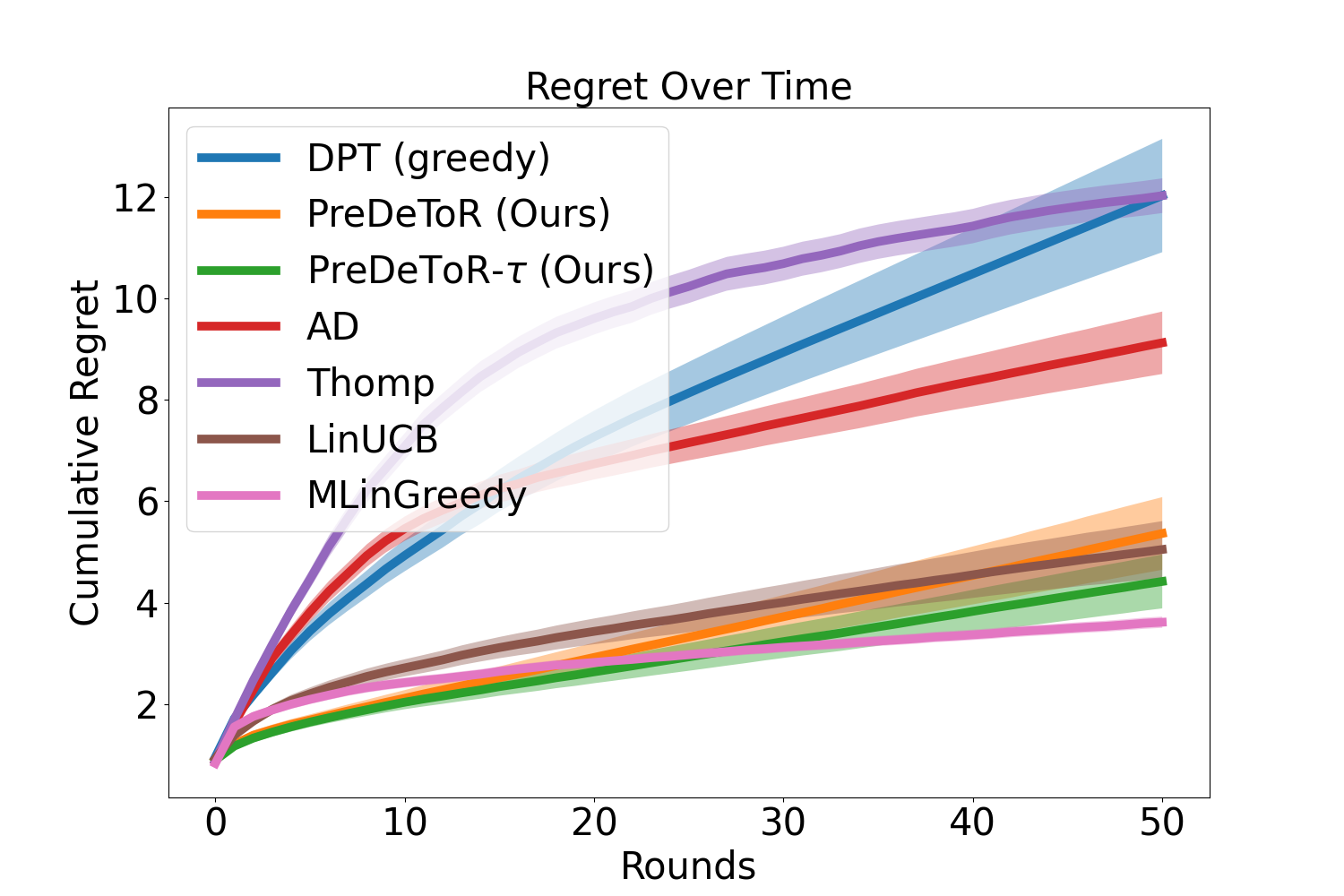}
    \caption{$0$ new action}
    \label{fig:new-lin-1}
\end{subfigure}%
\begin{subfigure}[b]{0.24\textwidth}
    \includegraphics[scale=0.1]{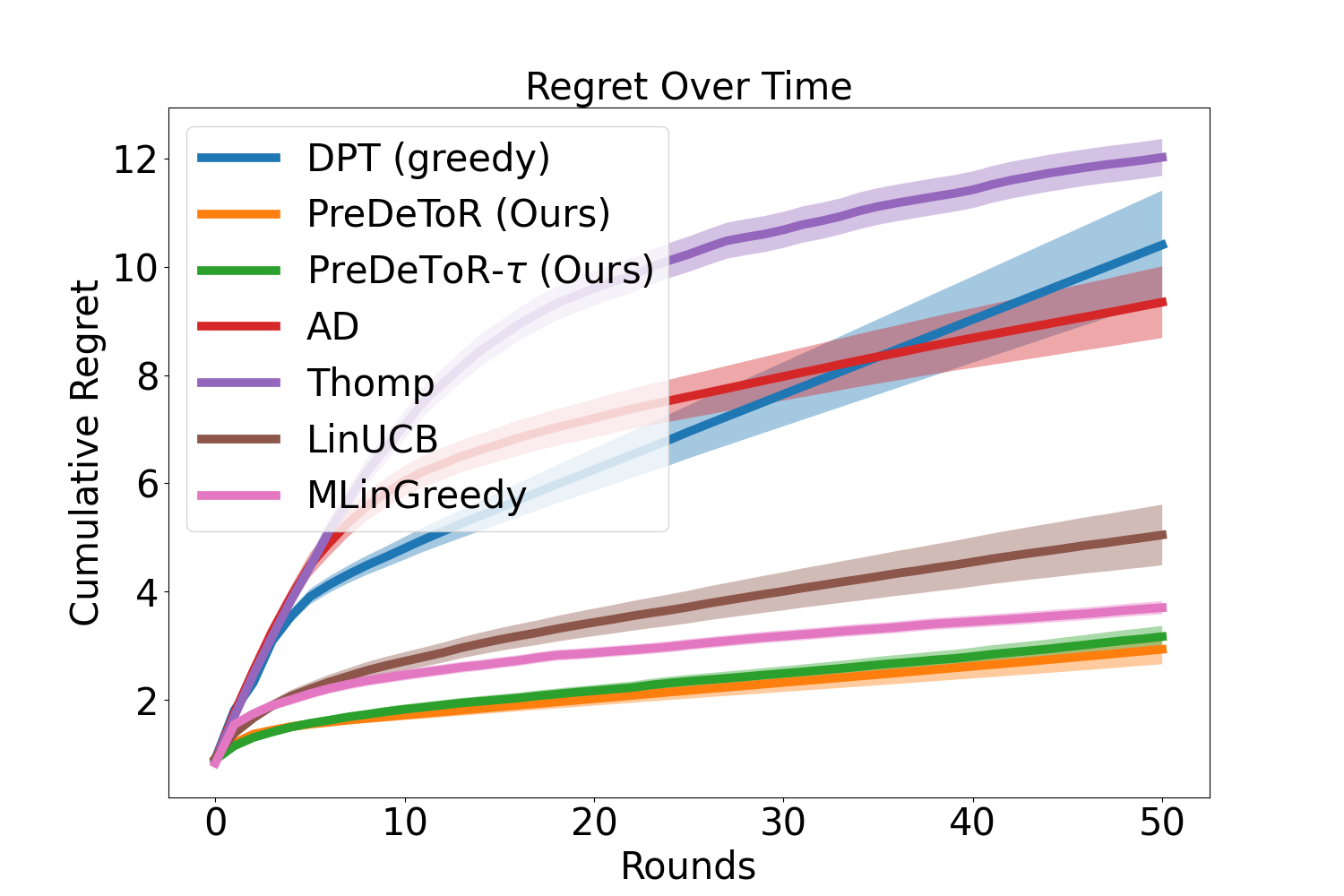}
    \caption{$1$ new action}
    \label{fig:new-lin-2}
\end{subfigure}%
\begin{subfigure}[b]{0.24\textwidth}
    \includegraphics[scale=0.1]{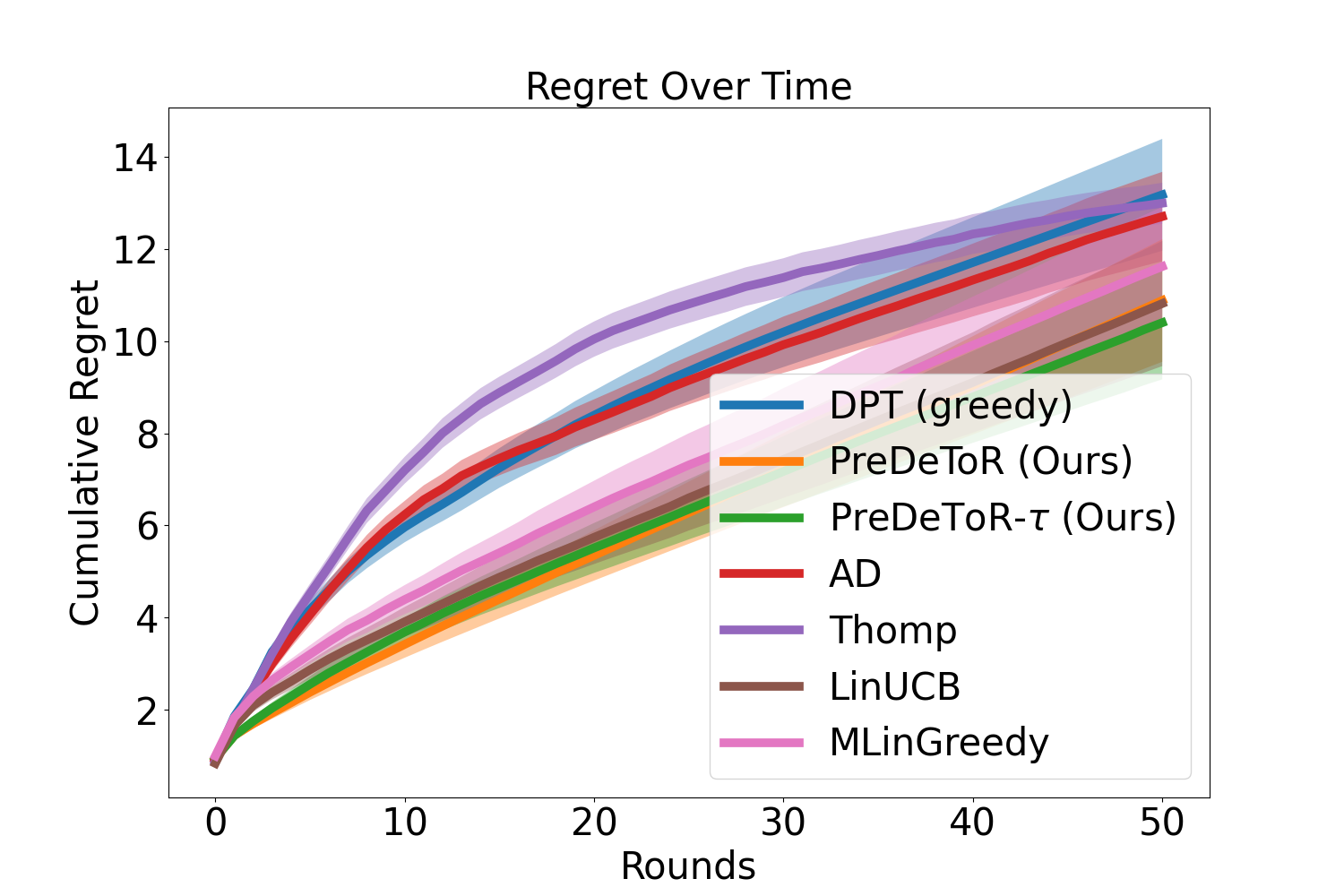}
    \caption{$5$ new actions}
    \label{fig:new-lin-3}
\end{subfigure}%
\begin{subfigure}[b]{0.24\textwidth}
    \includegraphics[scale=0.1]{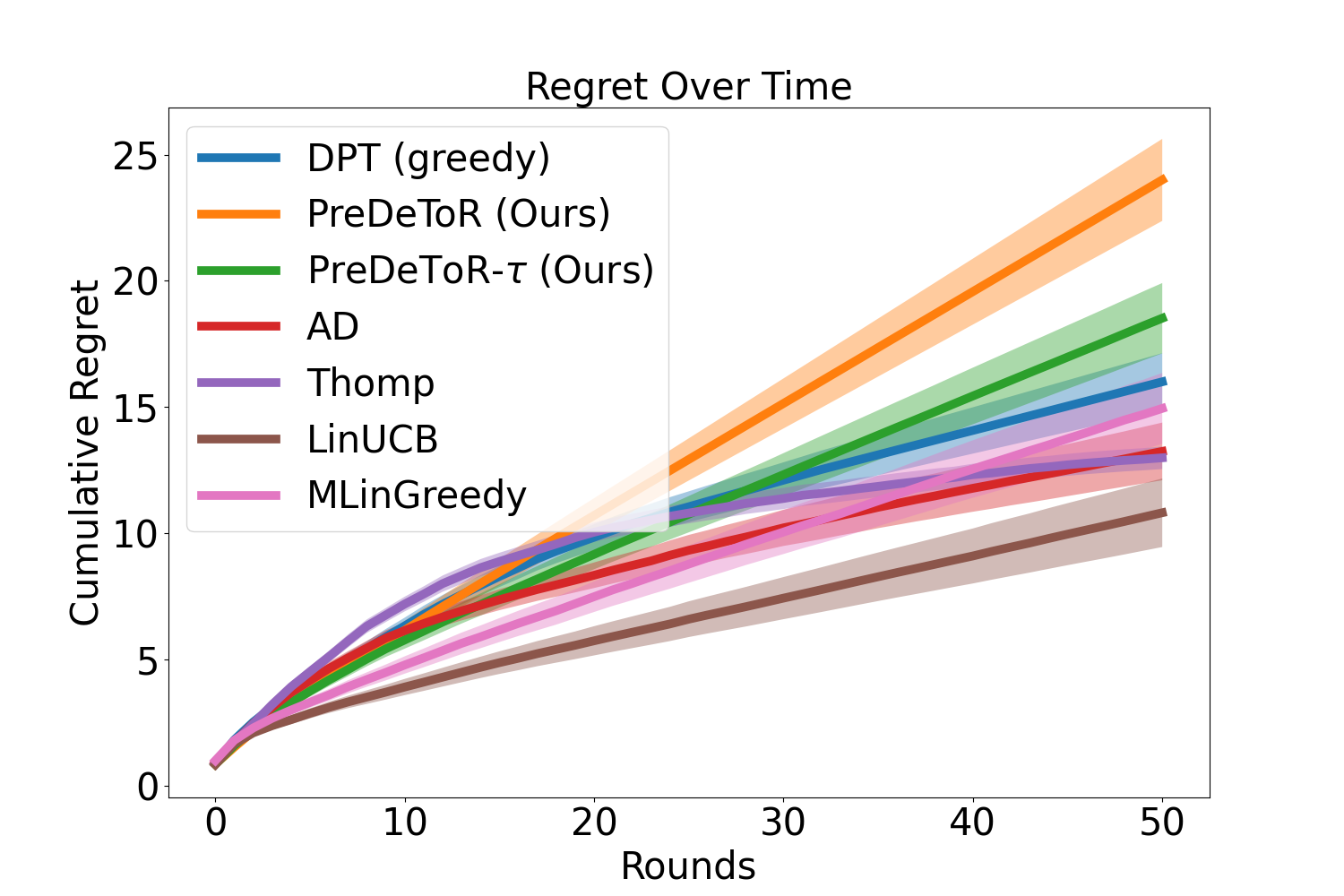}
    \caption{$10$ new actions}
    \label{fig:new-lin-4}
\end{subfigure}
\vspace*{-1.2em}
\caption{Linear new action experiments. The horizontal axis is the number of rounds. Confidence bars show one standard error.}
\label{fig:expt-new-actions-lin}
\vspace{-0.8em}
\end{figure}


\begin{figure}[!hbt]
\centering
\vspace*{-1em}
\begin{subfigure}[b]{0.24\textwidth}
    \includegraphics[scale=0.1]{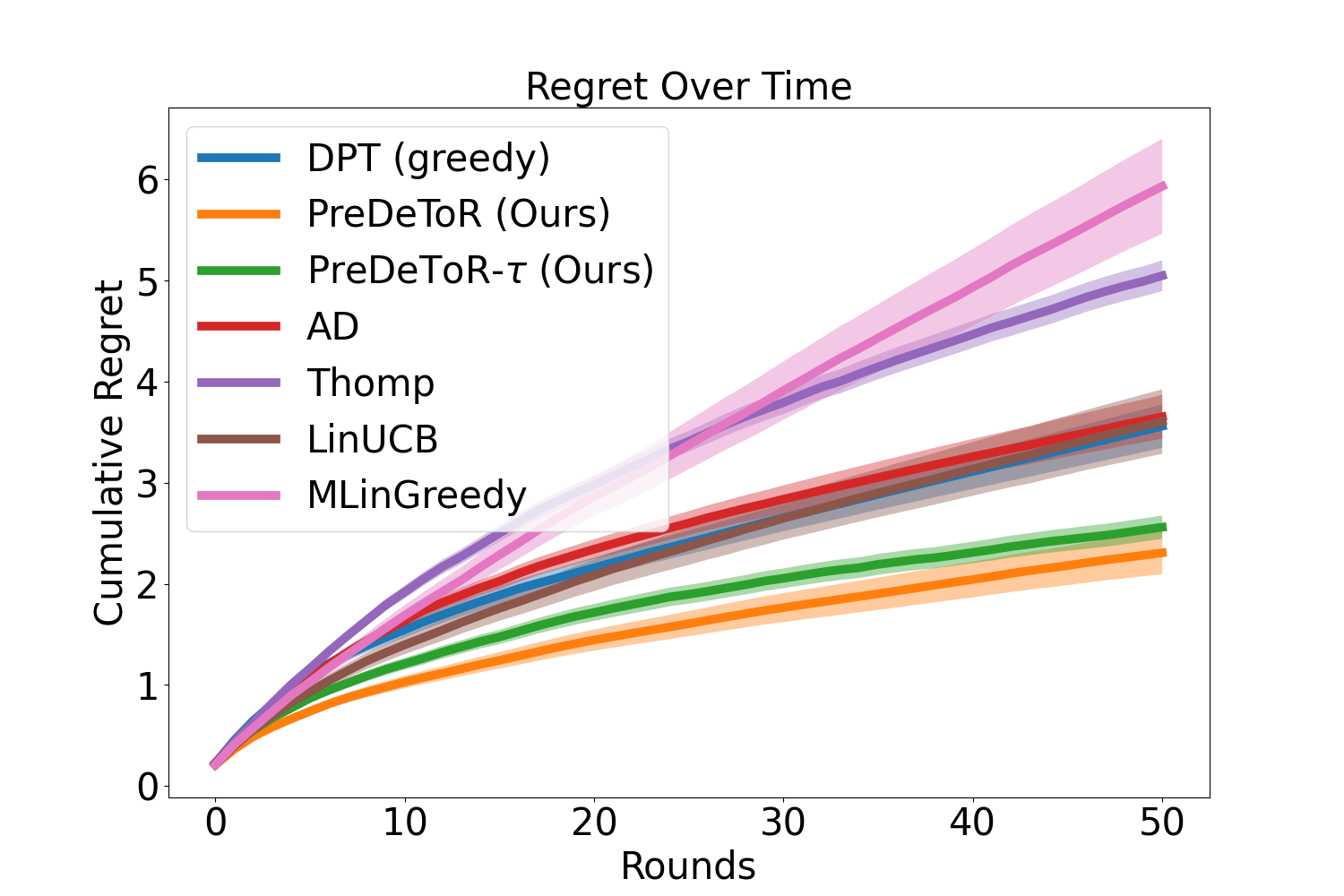}
    \caption{$0$ new action}
    \label{fig:new-nlm-1}
\end{subfigure}%
\begin{subfigure}[b]{0.24\textwidth}
    \includegraphics[scale=0.1]{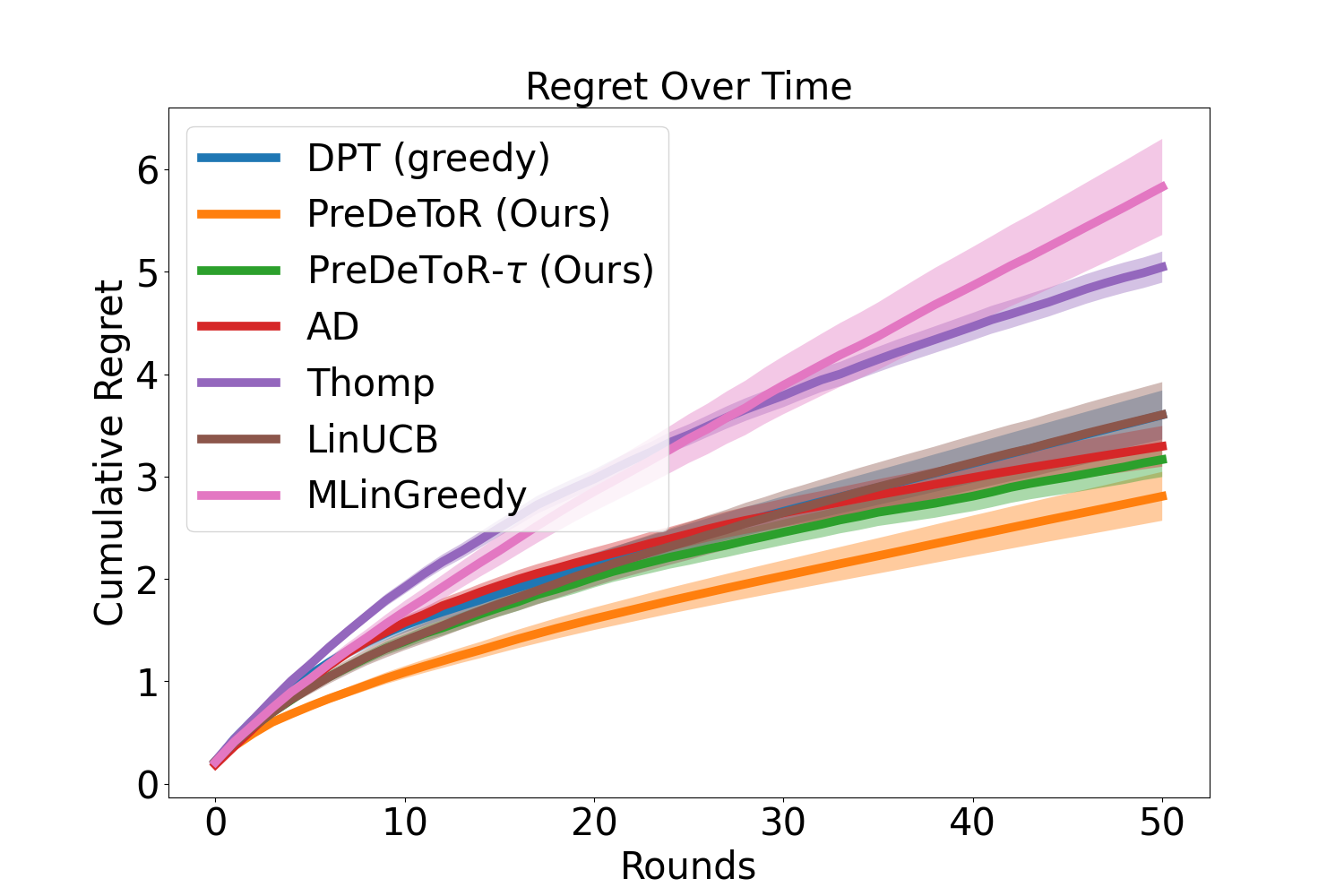}
    \caption{$1$ new action}
    \label{fig:new-nlm-2}
\end{subfigure}%
\begin{subfigure}[b]{0.24\textwidth}
    \includegraphics[scale=0.1]{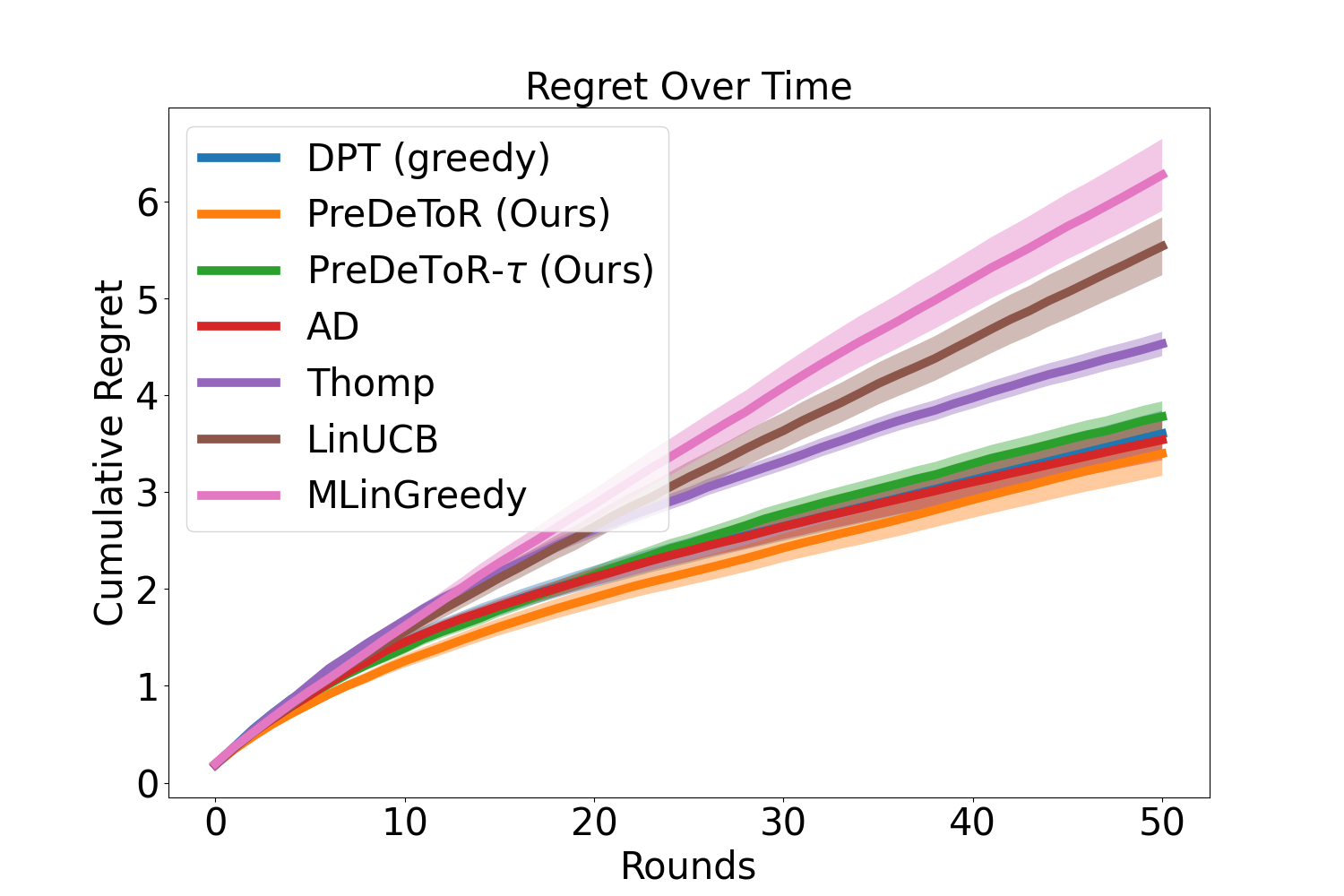}
    \caption{$5$ new actions}
    \label{fig:new-nlm-3}
\end{subfigure}%
\begin{subfigure}[b]{0.24\textwidth}
    \includegraphics[scale=0.1]{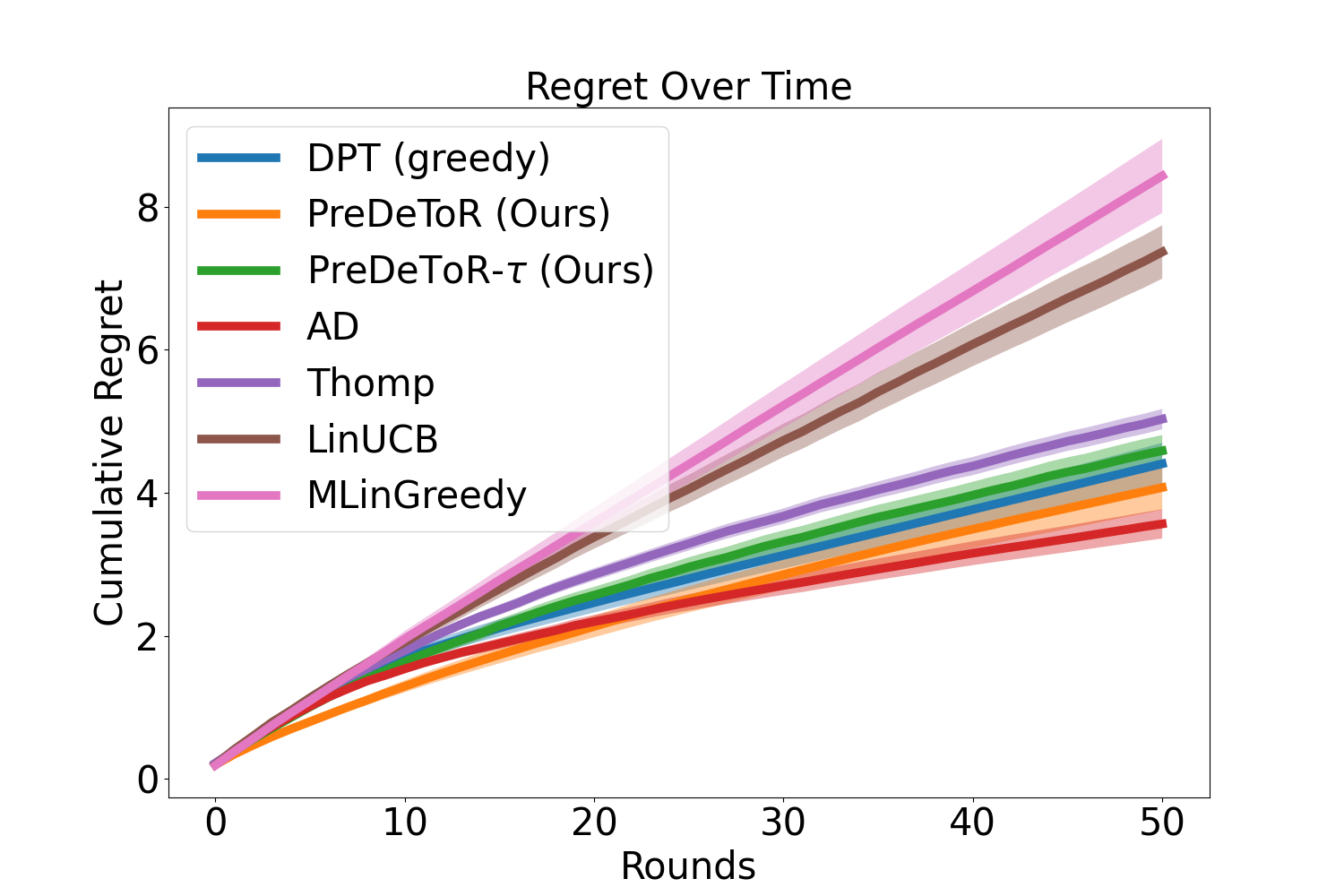}
    \caption{$10$ new actions}
    \label{fig:new-nlm-4}
\end{subfigure}
\vspace*{-1.2em}
\caption{Non-linear new action experiments with non-linear setting.}
\label{fig:expt-new-actions-nlm}
\vspace{-0.8em}
\end{figure}

\textbf{Baselines:} We implement the same baselines discussed in \Cref{sec:short-horizon}. 

\textbf{Outcomes:} Again before presenting the result we discuss the main outcomes from our experimental results of introducing new actions during data collection and evaluation:

\begin{tcolorbox}
\customfinding Shared structure across the tasks is important to learn the reward structure. 
\end{tcolorbox}




\textbf{Experimental Result:} We observe these outcomes in \Cref{fig:expt-new-actions-lin} and \Cref{fig:expt-new-actions-nlm}. We consider the linear and non-linear bandit setting of horizon $n=50$, $\Mpr = 100000$, $\Mts = 200$, $A=10$, and $d=2$. 
Here during data collection and during collecting the test data, we randomly select between $0, 1, 5$, and $10$ new actions from $\R^d$ for each task $m$. 
So the number of invariant actions is $|\Anc| \in \{10, 5, 1, 0\}$. 
Again, the demonstrator $\pi^w$ is the \ts\ algorithm. From \Cref{fig:new-lin-1}, \ref{fig:new-lin-2}, \ref{fig:new-lin-3}, and \ref{fig:new-lin-4}, we observe that when the number of invariant actions is less than \pred\ (\gt) has lower cumulative regret than \dptg, and \ad. 
%
%
Observe that \pred\ (\gt) matches \linucb\ and has lower regret than \dptg, and \ad\ when $\Anc| \in \{10, 5, 1\}$. This shows that \pred\ (\gt) is able to exploit the latent linear structure of the underlying tasks. However, as the number of invariant actions decreases we see that \pred (\gt) performance drops and becomes similar to the unstructured bandits \ts. We also show in \Cref{sec:k_arms_dpt} that in $K$-armed bandit setting when there is no structure across arms \pred\ (\gt) matches the performance of the demonstrator. 

Similarly in \Cref{fig:new-nlm-1}, \ref{fig:new-nlm-2}, \ref{fig:new-nlm-3}, and \ref{fig:new-nlm-4} we show the performance of \pred\ in the non-linear bandit setting. Observe that \linucb, \mlin\ fails to perform well in this non-linear setting due to their assumption of linear rewards. Again note that \pred\ (\gt) has lower regret than \dptg, and \ad\ when $\Anc| \in \{10, 1\}$. This shows that \pred\ (\gt) is able to exploit the latent linear structure of the underlying tasks. However, as the number of invariant actions decreases we see that \pred (\gt) performance drops and becomes similar to \ad.

\vspace*{-0.6em}
\section{Data Collection Analysis}
\label{sec:data-collection}
\vspace*{-0.6em}
In this section, we analyze the performance of \pred, \predt, \dptg, \ad, \ts, and \linucb\ when the weak demonstrator $\pi^w$ is \ts, \linucb, or \unif. We again consider the linear bandit setting discussed in \Cref{sec:short-horizon}. 
We show the cumulative regret by the above baselines in \Cref{fig:TS-collection}, \ref{fig:linucb-collection}, and \ref{fig:linucb-collection} when data is collected through \ts, \linucb, and \unif\ respectively. We first state the main finding below:
%
%
\begin{tcolorbox}
    \customfinding The \pred\ (\gt) excels in exploiting the underlying latent structure and reward correlation from in-context data when the data diversity is high.
\end{tcolorbox}
%
%
%
%
%
\begin{figure}[!hbt]
\centering
\begin{subfigure}[b]{0.32\textwidth}
   \includegraphics[scale=0.1]{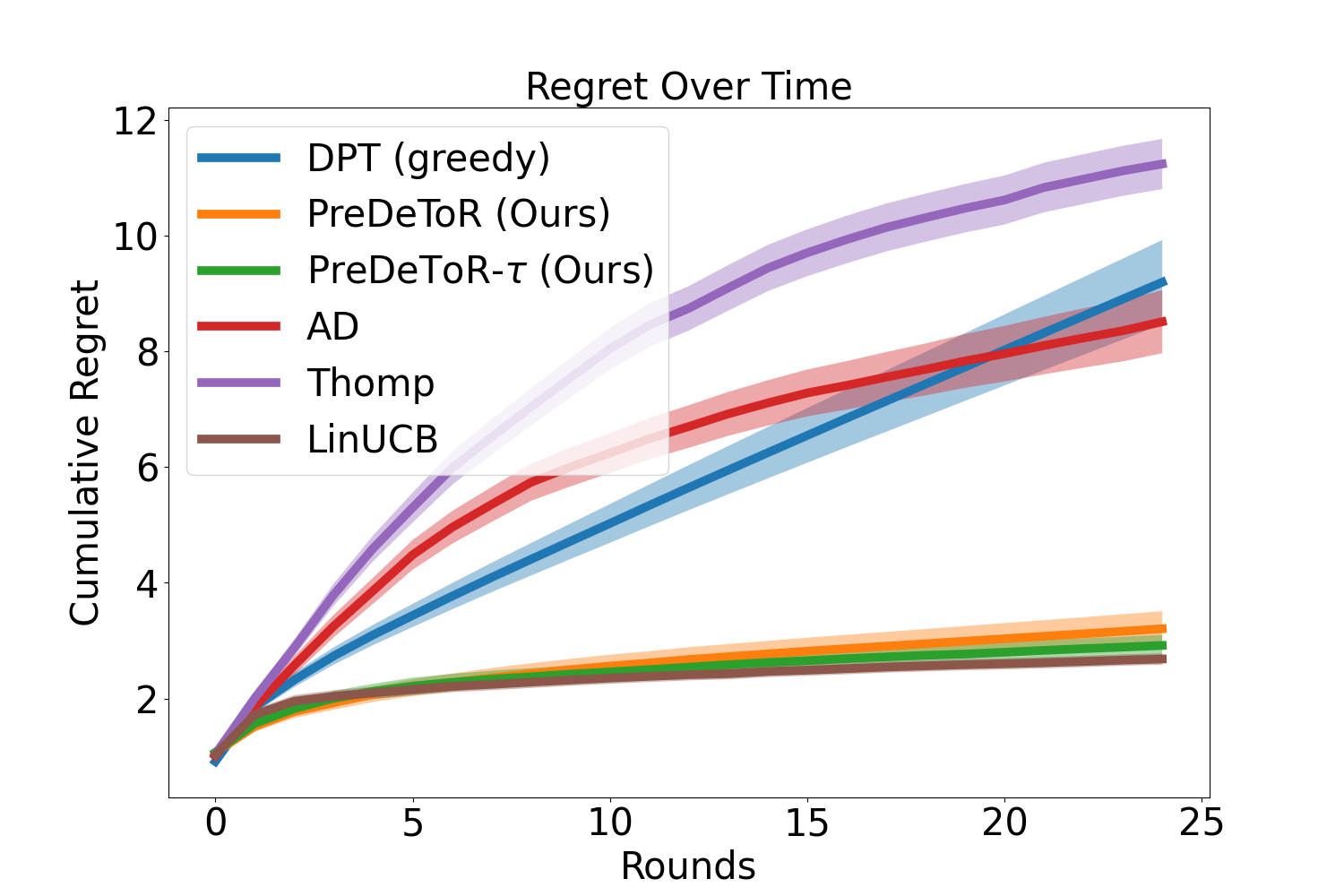}
   \caption{\ts\ data collection}
   \label{fig:TS-collection}
\end{subfigure}%
\begin{subfigure}[b]{0.32\textwidth}
   \includegraphics[scale=0.1]{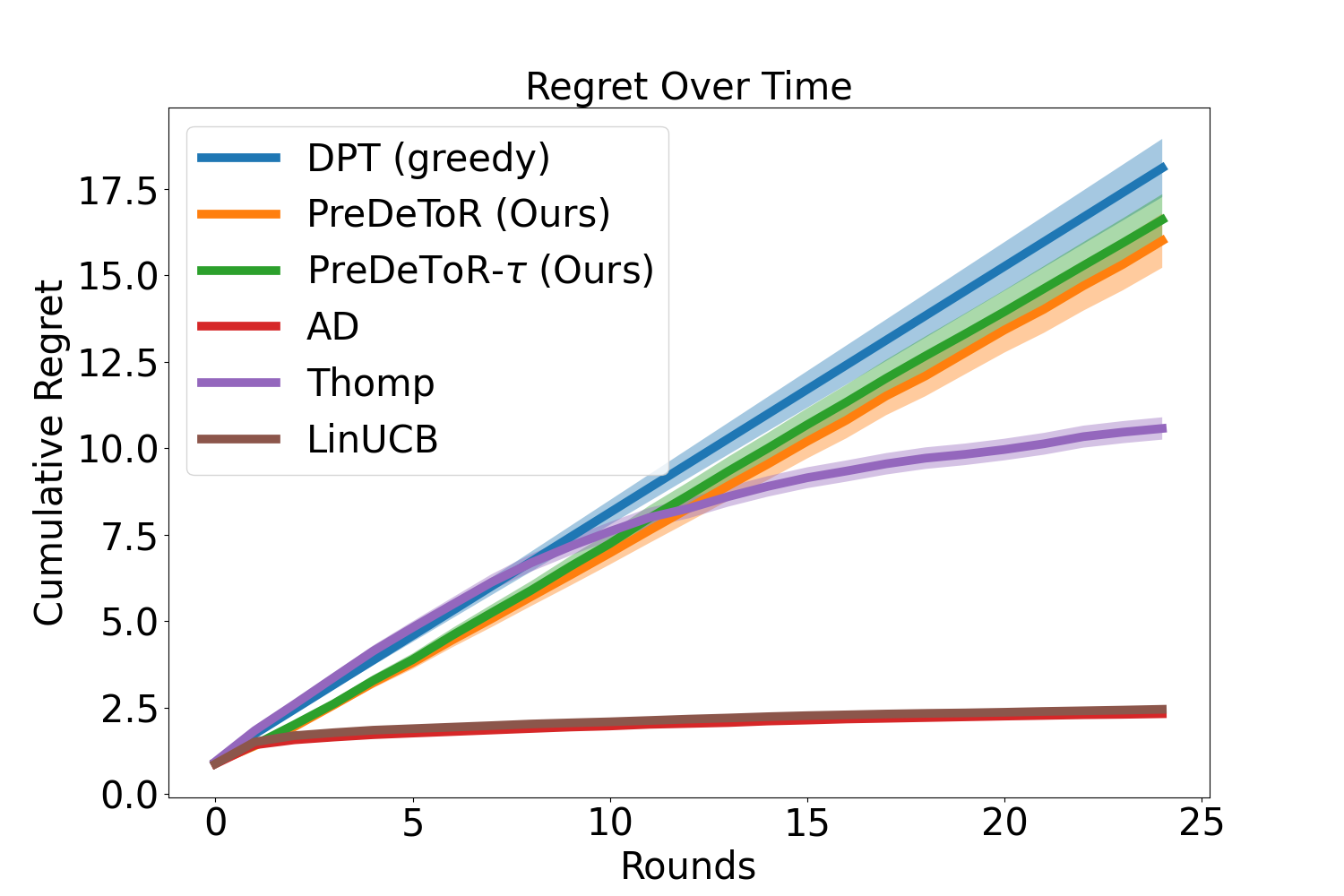}
   \caption{\linucb\ data collection}
   \label{fig:linucb-collection}
\end{subfigure}%
\begin{subfigure}[b]{0.32\textwidth}
   \includegraphics[scale=0.1]{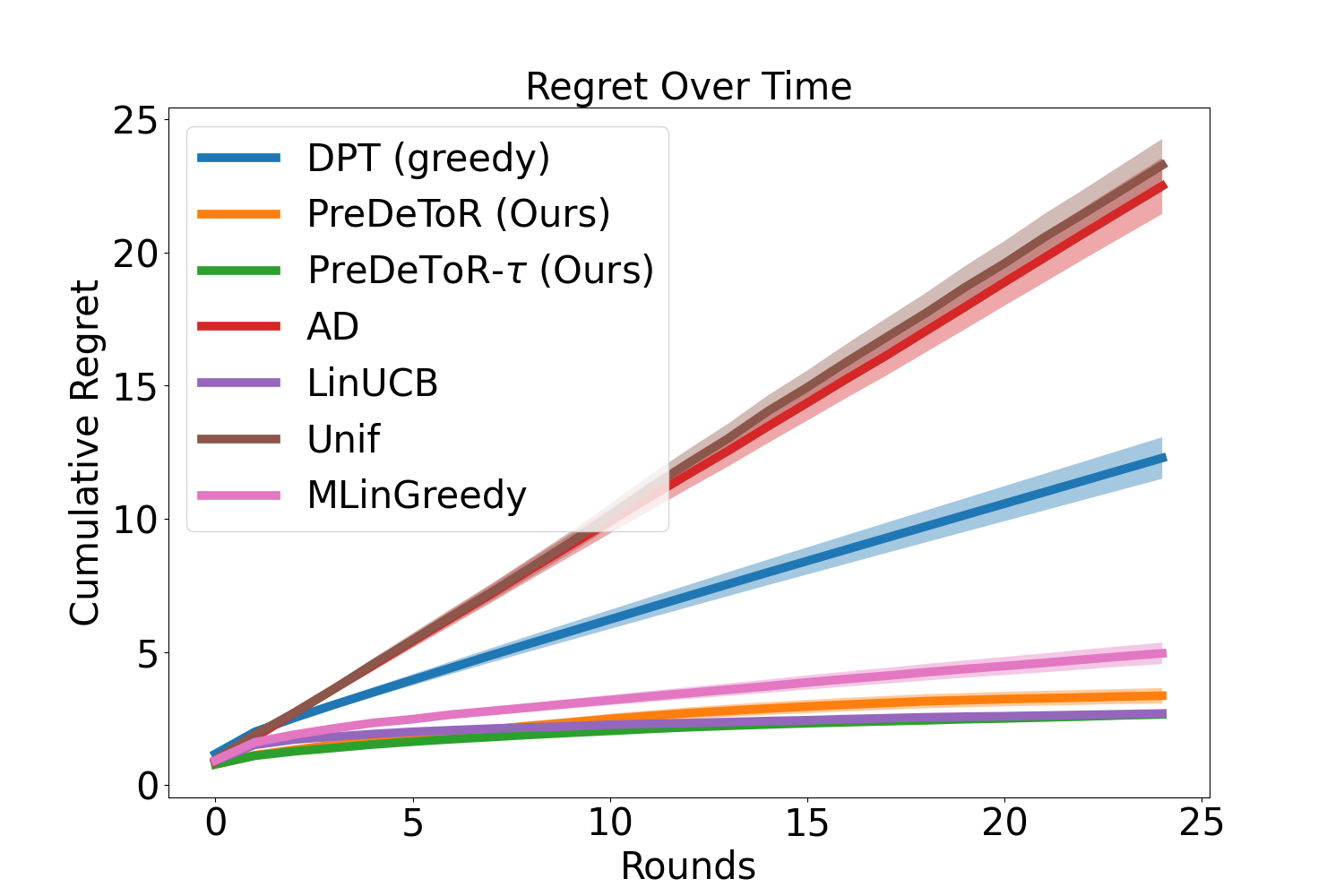}
   \caption{\unif\ data collection}
   \label{fig:unif-collection}
\end{subfigure}
\caption{Data Collection with various algorithms and Performance analysis}
\label{fig:data-collection}
\end{figure}

\textbf{Experimental Result:} We observe these outcomes in \Cref{fig:data-collection}. In \Cref{fig:TS-collection} we see that the $A$-actioned \ts\ is explorative enough as it does not explore with the knowledge of feature representation. 
So it pulls the sub-optimal actions sufficiently high number of times before discarding them in favor of the optimal action.
Therefore the training data is diverse enough so that \pred\ (\gt) can predict the reward vectors for actions sufficiently well. Consequently, \pred\ (\gt) almost matches the \linucb\ algorithm. Both \dptg\ and \ad perform poorly in this setting.

In \Cref{fig:linucb-collection} we see that the \linucb\ algorithm is not explorative enough as it explores with the knowledge of feature representation and quickly discards the sub-optimal actions in favor of the optimal action. Therefore the training data is not diverse enough so that \pred\ (\gt) is not able to correctly predict the reward vectors for actions. Note that \dptg\ also performs poorly in this setting when it is not provided with the optimal action information during training. The \ad\ matches the performance of its demonstrator \linucb\ because of its training procedure of predicting the next action of the demonstrator. 

Finally, in \Cref{fig:unif-collection} we see that the $A$-armed \unif\ is fully explorative as it does not intend to minimize regret (as opposed to \ts) and does not explore with the knowledge of feature representation. Therefore the training data is very diverse which results in \pred\ (\gt) being able to predict the reward vectors for actions very well. Consequently, \pred\ (\gt) perfectly matches the \linucb\ algorithm. Note that \ad\ performs the worst as it matches the performance of its demonstrator whereas the performance of \dptg\ suffers due to the lack of information on the optimal action during training.

We also empirically study the test performance of \pred\ (\gt) in $K$-armed bandit setting when there is no structure across arms in \Cref{sec:k_arms_dpt}, against the original \dpt\ in \Cref{sec:k_arms_dpt}, 
in other \textit{non-linear} bandit settings such as bilinear bandits (\Cref{sec:bilinear}), latent bandits (\Cref{sec:latent}), draw a connection between \pred\ and Bayesian estimators (\Cref{sec:und}), and perform sensitivity and ablation studies in \Cref{sec:lim}, \ref{sec:dimension}, \ref{sec:heads}, \ref{sec:envs}. 
%
Due to space constraints, we refer the interested reader to the relevant section in the appendices.

\vspace*{-0.7em}
\section{Theoretical Analysis of Generalization}
\vspace*{-0.7em}
\label{sec:theory}
In this section, we present a theoretical analysis of how \predt\ generalizes to an unknown target task given a set of source tasks. We observe that \predt's performance hinges on a low excess error on the predicted reward of the actions of the unknown target task based on the in-context data. Thus, in our analysis, we show that, in low-data regimes, \predt\ has a low expected excess risk for the unknown target task as the number of source tasks increases. This is summarized as follows:

\begin{tcolorbox}
\customfinding \pred\ (\gt) has a low expected excess risk for the unknown target task as the number of source tasks increases.  Moreover, the transfer learning risk of \predt\ (once trained on the $M$ source tasks) scales with $\widetilde{O}(1/\sqrt{M})$. 
\end{tcolorbox}

To show this, we proceed as follows: Suppose we have the training data set $\H_{\mathrm{all}} = \{\H_m\}_{m=1}^{\Mpr},$ where the task $m\sim \cT$ with a distribution $\cT$ and the task data $\H_m$ is generated from a distribution $ \Dpr(\cdot | m).$ 
For illustration purposes, here we consider the training data distribution $\Dpr(\cdot | m) $ where the actions are sampled following soft-LinUCB (a stochastic variant of LinUCB) \citep{chu2011contextual}. 
%
Given the loss function in \cref{eq:loss-transformer}, we can define the task $m$ training loss of \predt\ as $\widehat{\L}_m(\rT_{\bTheta}) = \frac{1}{n}\sum_{t=1}^n\ell(r_{m,t},\rT_{\bTheta}(\wr_{m,t}(I_{m, t})|\H^t_m)) = \frac{1}{n}\sum_{t=1}^n(\rT_{\bTheta}(\wr_{m,t}(I_{m, t})|\H^t_m) - r_{m,t})^2$. We drop the notation $\bTheta, \mathbf{r}$ from $\rT_{\bTheta}$ for simplicity and let $M=\Mpr$. We define
\begin{align}
\widehat{\T}= & \underset{\T \in \alg}{\arg \min } \widehat{\L}_{\H_{\mathrm{all}}}(\T):=\frac{1}{M} \sum_{m=1}^M \widehat{\L}_m(\T), \quad \text {(ERM)} \label{eq:erm}
\end{align}
where $\alg$ denotes the space of algorithms induced by the $\T$.
Let $\L_m(\T)=\mathbb{E}_{\H_m}\big[\widehat{\L}_m(\T)\big]$ and $\L_{\mathrm{MTL}}(\T)=\E\big[\widehat{\L}_{\H_{\mathrm{all }}}(\T)\big]=$ $\frac{1}{M} \sum_{m=1}^M \L_m(\T)$ be the corresponding population risks. For the ERM in \eqref{eq:erm}, we want to bound the following excess Multi-Task Learning (MTL) risk of \predt\ 
\begin{align}
\cR^{}_{\mathrm{MTL}}(\widehat{\T})=\L_{\mathrm{MTL}}(\widehat{\T})-\min_{\T \in \alg} \L_{\mathrm{MTL}}(\T) .\label{eq:risk}
\end{align}
%
Note that for in-context learning, a training sample $\left(I_t, r_t\right)$ impacts all future decisions of the algorithm from time step $t+1$ to $n$. Therefore, we need to control the stability of the input perturbation of the learning algorithm learned by the transformer. We introduce the following stability condition.
\begin{assumption}
\label{assm:stability-assumption}
(Error stability \citep{bousquet2002stability, li2023transformers}). Let $\H=\left(I_t, r_t\right)_{t=1}^n$ be a sequence in $[A] \times [0,1]$ with $n \geq 1$ and $\H^{\prime}$ be the sequence where the $t'$th sample of $\H$ is replaced by $\left(I_t^{\prime}, r_t^{\prime}\right)$. Error stability holds for a distribution $(I, r)\sim\D$ if there exists a $K>0$ such that for any $\H,\left(I_t^{\prime}, r_t^{\prime}\right) \in([A] \times [0,1]), t \leq n$, and $\T \in \alg$, we have
\begin{align*}
\big\lvert \mathbb{E}_{(I, r)}\left[\ell(r_{},\T_{}(\wr_{}(I_{})|\H))-\ell\left(r_{},\T_{}(\wr_{}(I_{})|\H^\prime)\right)\right] \big\rvert\, \leq \tfrac{K}{n}.
\end{align*}
Let $\rho$ be a distance metric on $\alg$. Pairwise error stability holds if for all $\T, \T^{\prime} \in \alg$ we have
\begin{align*}
\begin{aligned}
& \big\lvert \mathbb{E}_{(x, y)}\left[\ell(r_{},\T_{}(\wr_{}(I_{})|\H))-\ell\left(r_{},\T'_{}(\wr_{}(I_{})|\H)\right)-\ell(r_{},\T_{}(\wr_{}(I_{})|\H'))+\ell\left(r_{},\T'_{}(\wr_{}(I_{})|\H')\right)\right] \big\rvert \leq \tfrac{K \rho\left(\T, \T^{\prime}\right)}{n}.
\end{aligned}
\end{align*}
\end{assumption}
Now we present the Multi-task learning (MTL) risk of \predt.  
\begin{theorem}
\label{thm:multi-task-risk}\textbf{(\pred\ risk)}
    Suppose error stability \Cref{assm:stability-assumption} holds and assume loss function $\ell(\cdot,\cdot)$ is $C$-Lipschitz for all $r_t \in [0,B]$ and horizon $n\geq 1$. Let $\widehat{\T}$ be the empirical solution of (ERM) and $\mathcal{N}(\mathcal{A}, \rho, \epsilon)$ be the covering number of the algorithm space $\alg$ following Definition \ref{def:covering-number} and \ref{def:alg-dist}. Then with probability at least $1-2 \delta$, the excess MTL risk of \predt\ is bounded by
    \begin{align*} 
        \cR^{}_{\mathrm{MTL}}(\widehat{\T}_{}) \leq 4 \tfrac{C}{\sqrt{nM}} +2(B+K \log n) \sqrt{\tfrac{\log (\N(\alg, \rho, \varepsilon) / \delta)}{c n M}},
    \end{align*}
    where $\N(\alg, \rho, \varepsilon)$ is the covering number of transformer $\widehat{\T}_{}$ and $\epsilon = 1/\sqrt{nM}$. 
\end{theorem}
The proof of Theorem \ref{thm:multi-task-risk} is provided in \Cref{app:generalization}. 
From \Cref{thm:multi-task-risk} we see that in low-data regime with a small horizon $n$, as the number of tasks $M$ increases the MTL risk decreases. We further discuss the stability factor $K$ and covering number $\N(\alg, \rho, \varepsilon)$ in \Cref{remark:stability}, and \ref{remark:covering-number}. We also present the transfer learning risk of \predt\ in \Cref{app:transfer}.

\vspace*{-1.0em}
\section{Conclusions, Limitations and Future Works}
\label{sec:conc}
\vspace*{-0.7em}
In this paper, we studied the supervised pretraining of decision transformers in the multi-task structured bandit setting when the knowledge of the optimal action is unavailable.
Our proposed methods \pred\ (\gt) do not need to know the action representations or the reward structure and learn these with the help of offline data. 
%
\pred\ (\gt) predict the reward for the next action of each action during pretraining and can generalize well in-context in several regimes spanning low-data, new actions, and structured bandit settings like linear, non-linear, bilinear, latent bandits. The \pred\ (\gt) outperforms other in-context algorithms like \ad, \dptg\ in most of the experiments. Finally, we theoretically analyze \predt\ and show that pretraining it in $M$ source tasks leads to a low expected excess error on a target task drawn from the same task distribution $\cT$.
%
%
%
%
In the future, we want to extend our \pred\ (\gt) to the MDP setting \citep{sutton2018reinforcement, agarwal2019reinforcement}, and constrained MDP setting \citep{efroni2020exploration, gu2022review}.

\textbf{Acknowledgments:} Q.\ Xie is supported in part by National Science Foundation (NSF) grants CNS-1955997 and EPCN-2339794 and EPCN-2432546. J.\ Hanna is supported in part by NSF (IIS-2410981) and Sandia National Labs through a University Partnership Award.

\bibliography{biblio}

\begin{thebibliography}{121}
\providecommand{\natexlab}[1]{#1}
\providecommand{\url}[1]{\texttt{#1}}
\expandafter\ifx\csname urlstyle\endcsname\relax
  \providecommand{\doi}[1]{DOI: #1}\else
  \providecommand{\doi}{DOI: \begingroup \urlstyle{rm}\Url}\fi

\bibitem[Abbasi-Yadkori et~al.(2011)Abbasi-Yadkori, P{\'a}l, and Szepesv{\'a}ri]{abbasi2011improved}
Yasin Abbasi-Yadkori, D{\'a}vid P{\'a}l, and Csaba Szepesv{\'a}ri.
\newblock Improved algorithms for linear stochastic bandits.
\newblock \emph{Advances in neural information processing systems}, 24, 2011.

\bibitem[Agarwal et~al.(2019)Agarwal, Jiang, Kakade, and Sun]{agarwal2019reinforcement}
Alekh Agarwal, Nan Jiang, Sham~M Kakade, and Wen Sun.
\newblock Reinforcement learning: Theory and algorithms.
\newblock \emph{CS Dept., UW Seattle, Seattle, WA, USA, Tech. Rep}, 32, 2019.

\bibitem[Agrawal \& Goyal(2012)Agrawal and Goyal]{agrawal2012analysis}
Shipra Agrawal and Navin Goyal.
\newblock Analysis of thompson sampling for the multi-armed bandit problem.
\newblock In \emph{Conference on learning theory}, pp.\  39--1. JMLR Workshop and Conference Proceedings, 2012.

\bibitem[Asghar(2016)]{asghar2016yelp}
Nabiha Asghar.
\newblock Yelp dataset challenge: Review rating prediction.
\newblock \emph{arXiv preprint arXiv:1605.05362}, 2016.

\bibitem[Auer \& Ortner(2010)Auer and Ortner]{auer2010ucb}
Peter Auer and Ronald Ortner.
\newblock Ucb revisited: Improved regret bounds for the stochastic multi-armed bandit problem.
\newblock \emph{Periodica Mathematica Hungarica}, 61\penalty0 (1-2):\penalty0 55--65, 2010.

\bibitem[Auer et~al.(2002)Auer, Cesa-Bianchi, and Fischer]{auer2002finite-time}
Peter Auer, Nicolò Cesa-Bianchi, and Paul Fischer.
\newblock Finite-time {Analysis} of the {Multiarmed} {Bandit} {Problem}.
\newblock \emph{Machine Learning}, 47\penalty0 (2):\penalty0 235--256, May 2002.
\newblock ISSN 1573-0565.
\newblock \doi{10.1023/A:1013689704352}.
\newblock URL \url{https://doi.org/10.1023/A:1013689704352}.

\bibitem[Bengio et~al.(1990)Bengio, Bengio, and Cloutier]{bengio1990learning}
Yoshua Bengio, Samy Bengio, and Jocelyn Cloutier.
\newblock \emph{Learning a synaptic learning rule}.
\newblock Universit{\'e} de Montr{\'e}al, D{\'e}partement d'informatique et de recherche~…, 1990.

\bibitem[Bishop(2006)]{bishop2006pattern}
C~Bishop.
\newblock Pattern recognition and machine learning.
\newblock \emph{Springer google schola}, 2:\penalty0 531--537, 2006.

\bibitem[Bousquet \& Elisseeff(2002)Bousquet and Elisseeff]{bousquet2002stability}
Olivier Bousquet and Andr{\'e} Elisseeff.
\newblock Stability and generalization.
\newblock \emph{The Journal of Machine Learning Research}, 2:\penalty0 499--526, 2002.

\bibitem[Box \& Tiao(2011)Box and Tiao]{box2011bayesian}
George~EP Box and George~C Tiao.
\newblock \emph{Bayesian inference in statistical analysis}.
\newblock John Wiley \& Sons, 2011.

\bibitem[Brandfonbrener et~al.(2022)Brandfonbrener, Bietti, Buckman, Laroche, and Bruna]{brandfonbrener2022does}
David Brandfonbrener, Alberto Bietti, Jacob Buckman, Romain Laroche, and Joan Bruna.
\newblock When does return-conditioned supervised learning work for offline reinforcement learning?
\newblock \emph{Advances in Neural Information Processing Systems}, 35:\penalty0 1542--1553, 2022.

\bibitem[Brohan et~al.(2022)Brohan, Brown, Carbajal, Chebotar, Dabis, Finn, Gopalakrishnan, Hausman, Herzog, Hsu, et~al.]{brohan2022rt}
Anthony Brohan, Noah Brown, Justice Carbajal, Yevgen Chebotar, Joseph Dabis, Chelsea Finn, Keerthana Gopalakrishnan, Karol Hausman, Alex Herzog, Jasmine Hsu, et~al.
\newblock Rt-1: Robotics transformer for real-world control at scale.
\newblock \emph{arXiv preprint arXiv:2212.06817}, 2022.

\bibitem[Bubeck et~al.(2008)Bubeck, Stoltz, Szepesv{\'a}ri, and Munos]{bubeck2008online}
S{\'e}bastien Bubeck, Gilles Stoltz, Csaba Szepesv{\'a}ri, and R{\'e}mi Munos.
\newblock Online optimization in x-armed bandits.
\newblock \emph{Advances in Neural Information Processing Systems}, 21, 2008.

\bibitem[Bubeck et~al.(2011)Bubeck, Stoltz, and Yu]{bubeck2011lipschitz}
S{\'e}bastien Bubeck, Gilles Stoltz, and Jia~Yuan Yu.
\newblock Lipschitz bandits without the lipschitz constant.
\newblock In \emph{Algorithmic Learning Theory: 22nd International Conference, ALT 2011, Espoo, Finland, October 5-7, 2011. Proceedings 22}, pp.\  144--158. Springer, 2011.

\bibitem[Carlin \& Louis(2008)Carlin and Louis]{carlin2008bayesian}
Bradley~P Carlin and Thomas~A Louis.
\newblock \emph{Bayesian methods for data analysis}.
\newblock CRC press, 2008.

\bibitem[Chaudhuri et~al.(2017)Chaudhuri, Jain, and Natarajan]{chaudhuri2017active}
Kamalika Chaudhuri, Prateek Jain, and Nagarajan Natarajan.
\newblock Active heteroscedastic regression.
\newblock In \emph{International Conference on Machine Learning}, pp.\  694--702. PMLR, 2017.

\bibitem[Chen et~al.(2021)Chen, Lu, Rajeswaran, Lee, Grover, Laskin, Abbeel, Srinivas, and Mordatch]{chen2021decision}
Lili Chen, Kevin Lu, Aravind Rajeswaran, Kimin Lee, Aditya Grover, Misha Laskin, Pieter Abbeel, Aravind Srinivas, and Igor Mordatch.
\newblock Decision transformer: Reinforcement learning via sequence modeling.
\newblock \emph{Advances in neural information processing systems}, 34:\penalty0 15084--15097, 2021.

\bibitem[Chowdhury \& Gopalan(2017)Chowdhury and Gopalan]{chowdhury2017kernelized}
Sayak~Ray Chowdhury and Aditya Gopalan.
\newblock On kernelized multi-armed bandits.
\newblock In \emph{International Conference on Machine Learning}, pp.\  844--853. PMLR, 2017.

\bibitem[Chu et~al.(2011)Chu, Li, Reyzin, and Schapire]{chu2011contextual}
Wei Chu, Lihong Li, Lev Reyzin, and Robert Schapire.
\newblock Contextual bandits with linear payoff functions.
\newblock In \emph{Proceedings of the Fourteenth International Conference on Artificial Intelligence and Statistics}, pp.\  208--214. JMLR Workshop and Conference Proceedings, 2011.

\bibitem[Dai et~al.(2022)Dai, Shu, Verma, Fan, Low, and Jaillet]{dai2022federated}
Zhongxiang Dai, Yao Shu, Arun Verma, Flint~Xiaofeng Fan, Bryan Kian~Hsiang Low, and Patrick Jaillet.
\newblock Federated neural bandit.
\newblock \emph{arXiv preprint arXiv:2205.14309}, 2022.

\bibitem[Degenne et~al.(2020)Degenne, M{\'e}nard, Shang, and Valko]{degenne2020gamification}
R{\'e}my Degenne, Pierre M{\'e}nard, Xuedong Shang, and Michal Valko.
\newblock Gamification of pure exploration for linear bandits.
\newblock In \emph{International Conference on Machine Learning}, pp.\  2432--2442. PMLR, 2020.

\bibitem[Deshpande \& Montanari(2012)Deshpande and Montanari]{deshpande2012linear}
Yash Deshpande and Andrea Montanari.
\newblock Linear bandits in high dimension and recommendation systems.
\newblock In \emph{2012 50th Annual Allerton Conference on Communication, Control, and Computing (Allerton)}, pp.\  1750--1754. IEEE, 2012.

\bibitem[Dong et~al.(2021)Dong, Yang, and Ma]{dong2021provable}
Kefan Dong, Jiaqi Yang, and Tengyu Ma.
\newblock Provable model-based nonlinear bandit and reinforcement learning: Shelve optimism, embrace virtual curvature.
\newblock \emph{Advances in neural information processing systems}, 34:\penalty0 26168--26182, 2021.

\bibitem[Du et~al.(2023)Du, Huang, and Sun]{du2023multi}
Yihan Du, Longbo Huang, and Wen Sun.
\newblock Multi-task representation learning for pure exploration in linear bandits.
\newblock \emph{arXiv preprint arXiv:2302.04441}, 2023.

\bibitem[Duan et~al.(2016)Duan, Schulman, Chen, Bartlett, Sutskever, and Abbeel]{duan2016rl}
Yan Duan, John Schulman, Xi~Chen, Peter~L Bartlett, Ilya Sutskever, and Pieter Abbeel.
\newblock Rl $\\^{} 2$: Fast reinforcement learning via slow reinforcement learning.
\newblock \emph{arXiv preprint arXiv:1611.02779}, 2016.

\bibitem[Efroni et~al.(2020)Efroni, Mannor, and Pirotta]{efroni2020exploration}
Yonathan Efroni, Shie Mannor, and Matteo Pirotta.
\newblock Exploration-exploitation in constrained mdps.
\newblock \emph{arXiv preprint arXiv:2003.02189}, 2020.

\bibitem[Fedorov(2013)]{fedorov2013theory}
Valerii~Vadimovich Fedorov.
\newblock \emph{Theory of optimal experiments}.
\newblock Elsevier, 2013.

\bibitem[Filippi et~al.(2010)Filippi, Cappe, Garivier, and Szepesv{\'a}ri]{filippi2010parametric}
Sarah Filippi, Olivier Cappe, Aur{\'e}lien Garivier, and Csaba Szepesv{\'a}ri.
\newblock Parametric bandits: The generalized linear case.
\newblock \emph{Advances in neural information processing systems}, 23, 2010.

\bibitem[Finn et~al.(2017)Finn, Abbeel, and Levine]{finn2017model}
Chelsea Finn, Pieter Abbeel, and Sergey Levine.
\newblock Model-agnostic meta-learning for fast adaptation of deep networks.
\newblock In \emph{International conference on machine learning}, pp.\  1126--1135. PMLR, 2017.

\bibitem[Fran{\c{c}}ois-Lavet et~al.(2018)Fran{\c{c}}ois-Lavet, Henderson, Islam, Bellemare, Pineau, et~al.]{franccois2018introduction}
Vincent Fran{\c{c}}ois-Lavet, Peter Henderson, Riashat Islam, Marc~G Bellemare, Joelle Pineau, et~al.
\newblock An introduction to deep reinforcement learning.
\newblock \emph{Foundations and Trends{\textregistered} in Machine Learning}, 11\penalty0 (3-4):\penalty0 219--354, 2018.

\bibitem[Fu et~al.(2016)Fu, Levine, and Abbeel]{fu2016one}
Justin Fu, Sergey Levine, and Pieter Abbeel.
\newblock One-shot learning of manipulation skills with online dynamics adaptation and neural network priors.
\newblock In \emph{2016 IEEE/RSJ International Conference on Intelligent Robots and Systems (IROS)}, pp.\  4019--4026. IEEE, 2016.

\bibitem[Fujimoto et~al.(2019)Fujimoto, Meger, and Precup]{fujimoto2019off}
Scott Fujimoto, David Meger, and Doina Precup.
\newblock Off-policy deep reinforcement learning without exploration.
\newblock In \emph{International conference on machine learning}, pp.\  2052--2062. PMLR, 2019.

\bibitem[Ge et~al.(2022)Ge, Guo, Yang, Al-Garadi, and Sarker]{ge2022few}
Yao Ge, Yuting Guo, Yuan-Chi Yang, Mohammed~Ali Al-Garadi, and Abeed Sarker.
\newblock Few-shot learning for medical text: A systematic review.
\newblock \emph{arXiv preprint arXiv:2204.14081}, 2022.

\bibitem[Ghasemipour et~al.(2022)Ghasemipour, Gu, and Nachum]{ghasemipour2022so}
Kamyar Ghasemipour, Shixiang~Shane Gu, and Ofir Nachum.
\newblock Why so pessimistic? estimating uncertainties for offline rl through ensembles, and why their independence matters.
\newblock \emph{Advances in Neural Information Processing Systems}, 35:\penalty0 18267--18281, 2022.

\bibitem[Gu et~al.(2022)Gu, Yang, Du, Chen, Walter, Wang, Yang, and Knoll]{gu2022review}
Shangding Gu, Long Yang, Yali Du, Guang Chen, Florian Walter, Jun Wang, Yaodong Yang, and Alois Knoll.
\newblock A review of safe reinforcement learning: Methods, theory and applications.
\newblock \emph{arXiv preprint arXiv:2205.10330}, 2022.

\bibitem[Gupta et~al.(2018)Gupta, Mendonca, Liu, Abbeel, and Levine]{gupta2018meta}
Abhishek Gupta, Russell Mendonca, YuXuan Liu, Pieter Abbeel, and Sergey Levine.
\newblock Meta-reinforcement learning of structured exploration strategies.
\newblock \emph{Advances in neural information processing systems}, 31, 2018.

\bibitem[Gupta et~al.(2020)Gupta, Chaudhari, Mukherjee, Joshi, and Ya{\u{g}}an]{gupta2020unified}
Samarth Gupta, Shreyas Chaudhari, Subhojyoti Mukherjee, Gauri Joshi, and Osman Ya{\u{g}}an.
\newblock A unified approach to translate classical bandit algorithms to the structured bandit setting.
\newblock \emph{IEEE Journal on Selected Areas in Information Theory}, 1\penalty0 (3):\penalty0 840--853, 2020.

\bibitem[Harper \& Konstan(2015)Harper and Konstan]{harper2015movielens}
F~Maxwell Harper and Joseph~A Konstan.
\newblock The movielens datasets: History and context.
\newblock \emph{Acm transactions on interactive intelligent systems (tiis)}, 5\penalty0 (4):\penalty0 1--19, 2015.

\bibitem[Hong et~al.(2020)Hong, Kveton, Zaheer, Chow, Ahmed, and Boutilier]{hong2020latent}
Joey Hong, Branislav Kveton, Manzil Zaheer, Yinlam Chow, Amr Ahmed, and Craig Boutilier.
\newblock Latent bandits revisited.
\newblock \emph{Advances in Neural Information Processing Systems}, 33:\penalty0 13423--13433, 2020.

\bibitem[Hong et~al.(2022{\natexlab{a}})Hong, Kveton, Katariya, Zaheer, and Ghavamzadeh]{hong2022deep}
Joey Hong, Branislav Kveton, Sumeet Katariya, Manzil Zaheer, and Mohammad Ghavamzadeh.
\newblock Deep hierarchy in bandits.
\newblock In \emph{International Conference on Machine Learning}, pp.\  8833--8851. PMLR, 2022{\natexlab{a}}.

\bibitem[Hong et~al.(2022{\natexlab{b}})Hong, Kveton, Zaheer, and Ghavamzadeh]{hong2022hierarchical}
Joey Hong, Branislav Kveton, Manzil Zaheer, and Mohammad Ghavamzadeh.
\newblock Hierarchical bayesian bandits.
\newblock In \emph{International Conference on Artificial Intelligence and Statistics}, pp.\  7724--7741. PMLR, 2022{\natexlab{b}}.

\bibitem[Hong et~al.(2023)Hong, Kveton, Zaheer, Katariya, and Ghavamzadeh]{hong2023multi}
Joey Hong, Branislav Kveton, Manzil Zaheer, Sumeet Katariya, and Mohammad Ghavamzadeh.
\newblock Multi-task off-policy learning from bandit feedback.
\newblock In \emph{International Conference on Machine Learning}, pp.\  13157--13173. PMLR, 2023.

\bibitem[Janner et~al.(2021)Janner, Li, and Levine]{janner2021offline}
Michael Janner, Qiyang Li, and Sergey Levine.
\newblock Offline reinforcement learning as one big sequence modeling problem.
\newblock \emph{Advances in neural information processing systems}, 34:\penalty0 1273--1286, 2021.

\bibitem[Jiang et~al.(2022)Jiang, Liu, Eysenbach, Kolter, and Finn]{jiang2022learning}
Yiding Jiang, Evan Liu, Benjamin Eysenbach, J~Zico Kolter, and Chelsea Finn.
\newblock Learning options via compression.
\newblock \emph{Advances in Neural Information Processing Systems}, 35:\penalty0 21184--21199, 2022.

\bibitem[Johnson et~al.(2002)Johnson, Wichern, et~al.]{johnson2002applied}
Richard~Arnold Johnson, Dean~W Wichern, et~al.
\newblock Applied multivariate statistical analysis.
\newblock 2002.

\bibitem[Jun et~al.(2019)Jun, Willett, Wright, and Nowak]{jun2019bilinear}
Kwang-Sung Jun, Rebecca Willett, Stephen Wright, and Robert Nowak.
\newblock Bilinear bandits with low-rank structure.
\newblock In \emph{International Conference on Machine Learning}, pp.\  3163--3172. PMLR, 2019.

\bibitem[Justus et~al.(2018)Justus, Brennan, Bonner, and McGough]{justus2018predicting}
Daniel Justus, John Brennan, Stephen Bonner, and Andrew~Stephen McGough.
\newblock Predicting the computational cost of deep learning models.
\newblock In \emph{2018 IEEE international conference on big data (Big Data)}, pp.\  3873--3882. IEEE, 2018.

\bibitem[Kang et~al.(2022)Kang, Hsieh, and Lee]{kang2022efficient}
Yue Kang, Cho-Jui Hsieh, and Thomas Chun~Man Lee.
\newblock Efficient frameworks for generalized low-rank matrix bandit problems.
\newblock \emph{Advances in Neural Information Processing Systems}, 35:\penalty0 19971--19983, 2022.

\bibitem[Keles et~al.(2023)Keles, Wijewardena, and Hegde]{keles2023computational}
Feyza~Duman Keles, Pruthuvi~Mahesakya Wijewardena, and Chinmay Hegde.
\newblock On the computational complexity of self-attention.
\newblock In \emph{International Conference on Algorithmic Learning Theory}, pp.\  597--619. PMLR, 2023.

\bibitem[Kumar et~al.(2019)Kumar, Fu, Soh, Tucker, and Levine]{kumar2019stabilizing}
Aviral Kumar, Justin Fu, Matthew Soh, George Tucker, and Sergey Levine.
\newblock Stabilizing off-policy q-learning via bootstrapping error reduction.
\newblock \emph{Advances in Neural Information Processing Systems}, 32, 2019.

\bibitem[Kumar et~al.(2020)Kumar, Zhou, Tucker, and Levine]{kumar2020conservative}
Aviral Kumar, Aurick Zhou, George Tucker, and Sergey Levine.
\newblock Conservative q-learning for offline reinforcement learning.
\newblock \emph{Advances in Neural Information Processing Systems}, 33:\penalty0 1179--1191, 2020.

\bibitem[Kveton et~al.(2017)Kveton, Szepesv{\'a}ri, Rao, Wen, Abbasi-Yadkori, and Muthukrishnan]{kveton2017stochastic}
Branislav Kveton, Csaba Szepesv{\'a}ri, Anup Rao, Zheng Wen, Yasin Abbasi-Yadkori, and S~Muthukrishnan.
\newblock Stochastic low-rank bandits.
\newblock \emph{arXiv preprint arXiv:1712.04644}, 2017.

\bibitem[Kwon et~al.(2022)Kwon, Efroni, Caramanis, and Mannor]{kwon2022tractable}
Jeongyeol Kwon, Yonathan Efroni, Constantine Caramanis, and Shie Mannor.
\newblock Tractable optimality in episodic latent mabs.
\newblock \emph{Advances in Neural Information Processing Systems}, 35:\penalty0 23634--23645, 2022.

\bibitem[Landolfi et~al.(2019)Landolfi, Thomas, and Ma]{landolfi2019model}
Nicholas~C Landolfi, Garrett Thomas, and Tengyu Ma.
\newblock A model-based approach for sample-efficient multi-task reinforcement learning.
\newblock \emph{arXiv preprint arXiv:1907.04964}, 2019.

\bibitem[Laskin et~al.(2022)Laskin, Wang, Oh, Parisotto, Spencer, Steigerwald, Strouse, Hansen, Filos, Brooks, et~al.]{laskin2022context}
Michael Laskin, Luyu Wang, Junhyuk Oh, Emilio Parisotto, Stephen Spencer, Richie Steigerwald, DJ~Strouse, Steven Hansen, Angelos Filos, Ethan Brooks, et~al.
\newblock In-context reinforcement learning with algorithm distillation.
\newblock \emph{arXiv preprint arXiv:2210.14215}, 2022.

\bibitem[Lattimore \& Szepesv{\'a}ri(2019)Lattimore and Szepesv{\'a}ri]{lattimore2019information}
Tor Lattimore and Csaba Szepesv{\'a}ri.
\newblock An information-theoretic approach to minimax regret in partial monitoring.
\newblock In \emph{Conference on Learning Theory}, pp.\  2111--2139. PMLR, 2019.

\bibitem[Lattimore \& Szepesv{\'a}ri(2020)Lattimore and Szepesv{\'a}ri]{lattimore2020bandit}
Tor Lattimore and Csaba Szepesv{\'a}ri.
\newblock \emph{Bandit algorithms}.
\newblock Cambridge University Press, 2020.

\bibitem[Lee et~al.(2023)Lee, Xie, Pacchiano, Chandak, Finn, Nachum, and Brunskill]{lee2023supervised}
Jonathan~N Lee, Annie Xie, Aldo Pacchiano, Yash Chandak, Chelsea Finn, Ofir Nachum, and Emma Brunskill.
\newblock Supervised pretraining can learn in-context reinforcement learning.
\newblock \emph{arXiv preprint arXiv:2306.14892}, 2023.

\bibitem[Lee et~al.(2022)Lee, Nachum, Yang, Lee, Freeman, Guadarrama, Fischer, Xu, Jang, Michalewski, et~al.]{lee2022multi}
Kuang-Huei Lee, Ofir Nachum, Mengjiao~Sherry Yang, Lisa Lee, Daniel Freeman, Sergio Guadarrama, Ian Fischer, Winnie Xu, Eric Jang, Henryk Michalewski, et~al.
\newblock Multi-game decision transformers.
\newblock \emph{Advances in Neural Information Processing Systems}, 35:\penalty0 27921--27936, 2022.

\bibitem[Li et~al.(2020)Li, Yang, and Luo]{li2020focal}
Lanqing Li, Rui Yang, and Dijun Luo.
\newblock Focal: Efficient fully-offline meta-reinforcement learning via distance metric learning and behavior regularization.
\newblock \emph{arXiv preprint arXiv:2010.01112}, 2020.

\bibitem[Li et~al.(2010)Li, Chu, Langford, and Schapire]{li2010contextual}
Lihong Li, Wei Chu, John Langford, and Robert~E Schapire.
\newblock A contextual-bandit approach to personalized news article recommendation.
\newblock In \emph{Proceedings of the 19th international conference on World wide web}, pp.\  661--670, 2010.

\bibitem[Li et~al.(2017)Li, Lu, and Zhou]{li2017provably}
Lihong Li, Yu~Lu, and Dengyong Zhou.
\newblock Provably optimal algorithms for generalized linear contextual bandits.
\newblock In \emph{International Conference on Machine Learning}, pp.\  2071--2080. PMLR, 2017.

\bibitem[Li et~al.(2023)Li, Ildiz, Papailiopoulos, and Oymak]{li2023transformers}
Yingcong Li, Muhammed~Emrullah Ildiz, Dimitris Papailiopoulos, and Samet Oymak.
\newblock Transformers as algorithms: Generalization and stability in in-context learning.
\newblock In \emph{International Conference on Machine Learning}, pp.\  19565--19594. PMLR, 2023.

\bibitem[Lin et~al.(2023)Lin, Bai, and Mei]{lin2023transformers}
Licong Lin, Yu~Bai, and Song Mei.
\newblock Transformers as decision makers: Provable in-context reinforcement learning via supervised pretraining.
\newblock \emph{arXiv preprint arXiv:2310.08566}, 2023.

\bibitem[Liu et~al.(2021)Liu, Raghunathan, Liang, and Finn]{liu2021decoupling}
Evan~Z Liu, Aditi Raghunathan, Percy Liang, and Chelsea Finn.
\newblock Decoupling exploration and exploitation for meta-reinforcement learning without sacrifices.
\newblock In \emph{International conference on machine learning}, pp.\  6925--6935. PMLR, 2021.

\bibitem[Liu et~al.(2023{\natexlab{a}})Liu, Jiao, and Zhang]{liu2023self}
Xiaoqian Liu, Jianbin Jiao, and Junge Zhang.
\newblock Self-supervised pretraining for decision foundation model: Formulation, pipeline and challenges.
\newblock \emph{arXiv preprint arXiv:2401.00031}, 2023{\natexlab{a}}.

\bibitem[Liu et~al.(2023{\natexlab{b}})Liu, McDuff, Kovacs, Galatzer-Levy, Sunshine, Zhan, Poh, Liao, Di~Achille, and Patel]{liu2023large}
Xin Liu, Daniel McDuff, Geza Kovacs, Isaac Galatzer-Levy, Jacob Sunshine, Jiening Zhan, Ming-Zher Poh, Shun Liao, Paolo Di~Achille, and Shwetak Patel.
\newblock Large language models are few-shot health learners.
\newblock \emph{arXiv preprint arXiv:2305.15525}, 2023{\natexlab{b}}.

\bibitem[Liu et~al.(2019)Liu, Swaminathan, Agarwal, and Brunskill]{liu2019off}
Yao Liu, Adith Swaminathan, Alekh Agarwal, and Emma Brunskill.
\newblock Off-policy policy gradient with state distribution correction.
\newblock \emph{arXiv preprint arXiv:1904.08473}, 2019.

\bibitem[Liu et~al.(2020)Liu, Swaminathan, Agarwal, and Brunskill]{liu2020provably}
Yao Liu, Adith Swaminathan, Alekh Agarwal, and Emma Brunskill.
\newblock Provably good batch off-policy reinforcement learning without great exploration.
\newblock \emph{Advances in neural information processing systems}, 33:\penalty0 1264--1274, 2020.

\bibitem[Liu et~al.(2023{\natexlab{c}})Liu, Hu, Zhang, Guo, Ke, Liu, and Wang]{liu2023reason}
Zhihan Liu, Hao Hu, Shenao Zhang, Hongyi Guo, Shuqi Ke, Boyi Liu, and Zhaoran Wang.
\newblock Reason for future, act for now: A principled framework for autonomous llm agents with provable sample efficiency.
\newblock \emph{arXiv preprint arXiv:2309.17382}, 2023{\natexlab{c}}.

\bibitem[Lu et~al.(2023)Lu, Schroecker, Gu, Parisotto, Foerster, Singh, and Behbahani]{lu2023structured}
Chris Lu, Yannick Schroecker, Albert Gu, Emilio Parisotto, Jakob Foerster, Satinder Singh, and Feryal Behbahani.
\newblock Structured state space models for in-context reinforcement learning.
\newblock \emph{arXiv preprint arXiv:2303.03982}, 2023.

\bibitem[Lu et~al.(2021)Lu, Meisami, and Tewari]{lu2021low}
Yangyi Lu, Amirhossein Meisami, and Ambuj Tewari.
\newblock Low-rank generalized linear bandit problems.
\newblock In \emph{International Conference on Artificial Intelligence and Statistics}, pp.\  460--468. PMLR, 2021.

\bibitem[Ma et~al.(2023)Ma, Xiao, Liang, and Hao]{ma2023rethinking}
Yi~Ma, Chenjun Xiao, Hebin Liang, and Jianye Hao.
\newblock Rethinking decision transformer via hierarchical reinforcement learning.
\newblock \emph{arXiv preprint arXiv:2311.00267}, 2023.

\bibitem[Madotto et~al.(2021)Madotto, Lin, Winata, and Fung]{madotto2021few}
Andrea Madotto, Zhaojiang Lin, Genta~Indra Winata, and Pascale Fung.
\newblock Few-shot bot: Prompt-based learning for dialogue systems.
\newblock \emph{arXiv preprint arXiv:2110.08118}, 2021.

\bibitem[Magureanu et~al.(2014)Magureanu, Combes, and Proutiere]{magureanu2014lipschitz}
Stefan Magureanu, Richard Combes, and Alexandre Proutiere.
\newblock Lipschitz bandits: Regret lower bound and optimal algorithms.
\newblock In \emph{Conference on Learning Theory}, pp.\  975--999. PMLR, 2014.

\bibitem[Maillard \& Mannor(2014)Maillard and Mannor]{maillard2014latent}
Odalric-Ambrym Maillard and Shie Mannor.
\newblock Latent bandits.
\newblock In \emph{International Conference on Machine Learning}, pp.\  136--144. PMLR, 2014.

\bibitem[Min et~al.(2022)Min, Lyu, Holtzman, Artetxe, Lewis, Hajishirzi, and Zettlemoyer]{min2022rethinking}
Sewon Min, Xinxi Lyu, Ari Holtzman, Mikel Artetxe, Mike Lewis, Hannaneh Hajishirzi, and Luke Zettlemoyer.
\newblock Rethinking the role of demonstrations: What makes in-context learning work?
\newblock \emph{arXiv preprint arXiv:2202.12837}, 2022.

\bibitem[Mirchandani et~al.(2023)Mirchandani, Xia, Florence, Ichter, Driess, Arenas, Rao, Sadigh, and Zeng]{mirchandani2023large}
Suvir Mirchandani, Fei Xia, Pete Florence, Brian Ichter, Danny Driess, Montserrat~Gonzalez Arenas, Kanishka Rao, Dorsa Sadigh, and Andy Zeng.
\newblock Large language models as general pattern machines.
\newblock \emph{arXiv preprint arXiv:2307.04721}, 2023.

\bibitem[Mishra et~al.(2017)Mishra, Rohaninejad, Chen, and Abbeel]{mishra2017simple}
Nikhil Mishra, Mostafa Rohaninejad, Xi~Chen, and Pieter Abbeel.
\newblock A simple neural attentive meta-learner.
\newblock \emph{arXiv preprint arXiv:1707.03141}, 2017.

\bibitem[Mitchell et~al.(2021)Mitchell, Rafailov, Peng, Levine, and Finn]{mitchell2021offline}
Eric Mitchell, Rafael Rafailov, Xue~Bin Peng, Sergey Levine, and Chelsea Finn.
\newblock Offline meta-reinforcement learning with advantage weighting.
\newblock In \emph{International Conference on Machine Learning}, pp.\  7780--7791. PMLR, 2021.

\bibitem[Mnih et~al.(2013)Mnih, Kavukcuoglu, Silver, Graves, Antonoglou, Wierstra, and Riedmiller]{mnih2013playing}
Volodymyr Mnih, Koray Kavukcuoglu, David Silver, Alex Graves, Ioannis Antonoglou, Daan Wierstra, and Martin Riedmiller.
\newblock Playing atari with deep reinforcement learning.
\newblock \emph{arXiv preprint arXiv:1312.5602}, 2013.

\bibitem[Mukherjee et~al.(2023)Mukherjee, Xie, Hanna, and Nowak]{mukherjee2023multi}
Subhojyoti Mukherjee, Qiaomin Xie, Josiah~P Hanna, and Robert Nowak.
\newblock Multi-task representation learning for pure exploration in bilinear bandits.
\newblock \emph{arXiv preprint arXiv:2311.00327}, 2023.

\bibitem[M{\"u}ller et~al.(2021)M{\"u}ller, Hollmann, Arango, Grabocka, and Hutter]{muller2021transformers}
Samuel M{\"u}ller, Noah Hollmann, Sebastian~Pineda Arango, Josif Grabocka, and Frank Hutter.
\newblock Transformers can do bayesian inference.
\newblock \emph{arXiv preprint arXiv:2112.10510}, 2021.

\bibitem[Nagabandi et~al.(2018)Nagabandi, Clavera, Liu, Fearing, Abbeel, Levine, and Finn]{nagabandi2018learning}
Anusha Nagabandi, Ignasi Clavera, Simin Liu, Ronald~S Fearing, Pieter Abbeel, Sergey Levine, and Chelsea Finn.
\newblock Learning to adapt in dynamic, real-world environments through meta-reinforcement learning.
\newblock \emph{arXiv preprint arXiv:1803.11347}, 2018.

\bibitem[Neyshabur et~al.(2017)Neyshabur, Tomioka, Salakhutdinov, and Srebro]{neyshabur2017geometry}
Behnam Neyshabur, Ryota Tomioka, Ruslan Salakhutdinov, and Nathan Srebro.
\newblock Geometry of optimization and implicit regularization in deep learning.
\newblock \emph{arXiv preprint arXiv:1705.03071}, 2017.

\bibitem[Pal et~al.(2023)Pal, Suggala, Shanmugam, and Jain]{pal2023optimal}
Soumyabrata Pal, Arun~Sai Suggala, Karthikeyan Shanmugam, and Prateek Jain.
\newblock Optimal algorithms for latent bandits with cluster structure.
\newblock In \emph{International Conference on Artificial Intelligence and Statistics}, pp.\  7540--7577. PMLR, 2023.

\bibitem[Perkins \& Precup(1999)Perkins and Precup]{perkins1999using}
Theodore~J Perkins and Doina Precup.
\newblock Using options for knowledge transfer in reinforcement learning title2, 1999.

\bibitem[Pong et~al.(2022)Pong, Nair, Smith, Huang, and Levine]{pong2022offline}
Vitchyr~H Pong, Ashvin~V Nair, Laura~M Smith, Catherine Huang, and Sergey Levine.
\newblock Offline meta-reinforcement learning with online self-supervision.
\newblock In \emph{International Conference on Machine Learning}, pp.\  17811--17829. PMLR, 2022.

\bibitem[Pukelsheim(2006)]{pukelsheim2006optimal}
Friedrich Pukelsheim.
\newblock \emph{Optimal design of experiments}.
\newblock SIAM, 2006.

\bibitem[Rakelly et~al.(2019)Rakelly, Zhou, Finn, Levine, and Quillen]{rakelly2019efficient}
Kate Rakelly, Aurick Zhou, Chelsea Finn, Sergey Levine, and Deirdre Quillen.
\newblock Efficient off-policy meta-reinforcement learning via probabilistic context variables.
\newblock In \emph{International conference on machine learning}, pp.\  5331--5340. PMLR, 2019.

\bibitem[Reed et~al.(2022)Reed, Zolna, Parisotto, Colmenarejo, Novikov, Barth-Maron, Gimenez, Sulsky, Kay, Springenberg, et~al.]{reed2022generalist}
Scott Reed, Konrad Zolna, Emilio Parisotto, Sergio~Gomez Colmenarejo, Alexander Novikov, Gabriel Barth-Maron, Mai Gimenez, Yury Sulsky, Jackie Kay, Jost~Tobias Springenberg, et~al.
\newblock A generalist agent.
\newblock \emph{arXiv preprint arXiv:2205.06175}, 2022.

\bibitem[Riquelme et~al.(2018)Riquelme, Tucker, and Snoek]{riquelme2018deep}
Carlos Riquelme, George Tucker, and Jasper Snoek.
\newblock Deep bayesian bandits showdown: An empirical comparison of bayesian deep networks for thompson sampling.
\newblock \emph{arXiv preprint arXiv:1802.09127}, 2018.

\bibitem[Roberts et~al.(2019)Roberts, Raffel, Lee, Matena, Shazeer, Liu, Narang, Li, and Zhou]{roberts2019exploring}
Adam Roberts, Colin Raffel, Katherine Lee, Michael Matena, Noam Shazeer, Peter~J Liu, Sharan Narang, Wei Li, and Yanqi Zhou.
\newblock Exploring the limits of transfer learning with a unified text-to-text transformer.
\newblock 2019.

\bibitem[Rothfuss et~al.(2018)Rothfuss, Lee, Clavera, Asfour, and Abbeel]{rothfuss2018promp}
Jonas Rothfuss, Dennis Lee, Ignasi Clavera, Tamim Asfour, and Pieter Abbeel.
\newblock Promp: Proximal meta-policy search.
\newblock \emph{arXiv preprint arXiv:1810.06784}, 2018.

\bibitem[Russo et~al.(2018)Russo, Van~Roy, Kazerouni, Osband, Wen, et~al.]{russo2018tutorial}
Daniel~J Russo, Benjamin Van~Roy, Abbas Kazerouni, Ian Osband, Zheng Wen, et~al.
\newblock A tutorial on thompson sampling.
\newblock \emph{Foundations and Trends{\textregistered} in Machine Learning}, 11\penalty0 (1):\penalty0 1--96, 2018.

\bibitem[Schaul \& Schmidhuber(2010)Schaul and Schmidhuber]{schaul2010metalearning}
Tom Schaul and J{\"u}rgen Schmidhuber.
\newblock Metalearning.
\newblock \emph{Scholarpedia}, 5\penalty0 (6):\penalty0 4650, 2010.

\bibitem[Semnani et~al.(2023)Semnani, Yao, Zhang, and Lam]{semnani2023wikichat}
Sina Semnani, Violet Yao, Heidi Zhang, and Monica Lam.
\newblock Wikichat: Stopping the hallucination of large language model chatbots by few-shot grounding on wikipedia.
\newblock In \emph{Findings of the Association for Computational Linguistics: EMNLP 2023}, pp.\  2387--2413, 2023.

\bibitem[Shafiullah et~al.(2022)Shafiullah, Cui, Altanzaya, and Pinto]{shafiullah2022behavior}
Nur~Muhammad Shafiullah, Zichen Cui, Ariuntuya~Arty Altanzaya, and Lerrel Pinto.
\newblock Behavior transformers: Cloning $ k $ modes with one stone.
\newblock \emph{Advances in neural information processing systems}, 35:\penalty0 22955--22968, 2022.

\bibitem[Siegel et~al.(2020)Siegel, Springenberg, Berkenkamp, Abdolmaleki, Neunert, Lampe, Hafner, Heess, and Riedmiller]{siegel2020keep}
Noah~Y Siegel, Jost~Tobias Springenberg, Felix Berkenkamp, Abbas Abdolmaleki, Michael Neunert, Thomas Lampe, Roland Hafner, Nicolas Heess, and Martin Riedmiller.
\newblock Keep doing what worked: Behavioral modelling priors for offline reinforcement learning.
\newblock \emph{arXiv preprint arXiv:2002.08396}, 2020.

\bibitem[Sinii et~al.(2023)Sinii, Nikulin, Kurenkov, Zisman, and Kolesnikov]{sinii2023context}
Viacheslav Sinii, Alexander Nikulin, Vladislav Kurenkov, Ilya Zisman, and Sergey Kolesnikov.
\newblock In-context reinforcement learning for variable action spaces.
\newblock \emph{arXiv preprint arXiv:2312.13327}, 2023.

\bibitem[Soudry et~al.(2018)Soudry, Hoffer, Nacson, Gunasekar, and Srebro]{soudry2018implicit}
Daniel Soudry, Elad Hoffer, Mor~Shpigel Nacson, Suriya Gunasekar, and Nathan Srebro.
\newblock The implicit bias of gradient descent on separable data.
\newblock \emph{Journal of Machine Learning Research}, 19\penalty0 (70):\penalty0 1--57, 2018.

\bibitem[Sutton \& Barto(2018)Sutton and Barto]{sutton2018reinforcement}
Richard~S Sutton and Andrew~G Barto.
\newblock \emph{Reinforcement learning: An introduction}.
\newblock MIT press, 2018.

\bibitem[Thompson(1933)]{thompson1933likelihood}
William~R Thompson.
\newblock On the likelihood that one unknown probability exceeds another in view of the evidence of two samples.
\newblock \emph{Biometrika}, 25\penalty0 (3-4):\penalty0 285--294, 1933.

\bibitem[Tomkins et~al.(2020)Tomkins, Liao, Klasnja, Yeung, and Murphy]{tomkins2020rapidly}
Sabina Tomkins, Peng Liao, Predrag Klasnja, Serena Yeung, and Susan Murphy.
\newblock Rapidly personalizing mobile health treatment policies with limited data.
\newblock \emph{arXiv preprint arXiv:2002.09971}, 2020.

\bibitem[Valko et~al.(2013)Valko, Korda, Munos, Flaounas, and Cristianini]{valko2013finite}
Michal Valko, Nathaniel Korda, R{\'e}mi Munos, Ilias Flaounas, and Nelo Cristianini.
\newblock Finite-time analysis of kernelised contextual bandits.
\newblock \emph{arXiv preprint arXiv:1309.6869}, 2013.

\bibitem[Vaswani et~al.(2017)Vaswani, Shazeer, Parmar, Uszkoreit, Jones, Gomez, Kaiser, and Polosukhin]{vaswani2017attention}
Ashish Vaswani, Noam Shazeer, Niki Parmar, Jakob Uszkoreit, Llion Jones, Aidan~N Gomez, {L}ukasz Kaiser, and Illia Polosukhin.
\newblock Attention is all you need.
\newblock \emph{Advances in neural information processing systems}, 30, 2017.

\bibitem[Wang et~al.(2016)Wang, Kurth-Nelson, Tirumala, Soyer, Leibo, Munos, Blundell, Kumaran, and Botvinick]{wang2016learning}
Jane~X Wang, Zeb Kurth-Nelson, Dhruva Tirumala, Hubert Soyer, Joel~Z Leibo, Remi Munos, Charles Blundell, Dharshan Kumaran, and Matt Botvinick.
\newblock Learning to reinforcement learn.
\newblock \emph{arXiv preprint arXiv:1611.05763}, 2016.

\bibitem[Wu et~al.(2019)Wu, Tucker, and Nachum]{wu2019behavior}
Yifan Wu, George Tucker, and Ofir Nachum.
\newblock Behavior regularized offline reinforcement learning.
\newblock \emph{arXiv preprint arXiv:1911.11361}, 2019.

\bibitem[Xie et~al.(2021)Xie, Raghunathan, Liang, and Ma]{xie2021explanation}
Sang~Michael Xie, Aditi Raghunathan, Percy Liang, and Tengyu Ma.
\newblock An explanation of in-context learning as implicit bayesian inference.
\newblock \emph{arXiv preprint arXiv:2111.02080}, 2021.

\bibitem[Yang et~al.(2023)Yang, Robeyns, Wang, and Aitchison]{yang2023bayesian}
Adam~X Yang, Maxime Robeyns, Xi~Wang, and Laurence Aitchison.
\newblock Bayesian low-rank adaptation for large language models.
\newblock \emph{arXiv preprint arXiv:2308.13111}, 2023.

\bibitem[Yang et~al.(2020)Yang, Hu, Lee, and Du]{yang2020impact}
Jiaqi Yang, Wei Hu, Jason~D Lee, and Simon~S Du.
\newblock Impact of representation learning in linear bandits.
\newblock \emph{arXiv preprint arXiv:2010.06531}, 2020.

\bibitem[Yang et~al.(2021)Yang, Hu, Lee, and Du]{yang2021impact}
Jiaqi Yang, Wei Hu, Jason~D Lee, and Simon~Shaolei Du.
\newblock Impact of representation learning in linear bandits.
\newblock In \emph{International Conference on Learning Representations}, 2021.

\bibitem[Yang et~al.(2022{\natexlab{a}})Yang, Lei, Lee, and Du]{yang2022nearly}
Jiaqi Yang, Qi~Lei, Jason~D Lee, and Simon~S Du.
\newblock Nearly minimax algorithms for linear bandits with shared representation.
\newblock \emph{arXiv preprint arXiv:2203.15664}, 2022{\natexlab{a}}.

\bibitem[Yang \& Wang(2020)Yang and Wang]{yang2020reinforcement}
Lin Yang and Mengdi Wang.
\newblock Reinforcement learning in feature space: Matrix bandit, kernels, and regret bound.
\newblock In \emph{International Conference on Machine Learning}, pp.\  10746--10756. PMLR, 2020.

\bibitem[Yang et~al.(2022{\natexlab{b}})Yang, Schuurmans, Abbeel, and Nachum]{yang2022dichotomy}
Mengjiao Yang, Dale Schuurmans, Pieter Abbeel, and Ofir Nachum.
\newblock Dichotomy of control: Separating what you can control from what you cannot.
\newblock \emph{arXiv preprint arXiv:2210.13435}, 2022{\natexlab{b}}.

\bibitem[Yu et~al.(2021)Yu, Kumar, Rafailov, Rajeswaran, Levine, and Finn]{yu2021combo}
Tianhe Yu, Aviral Kumar, Rafael Rafailov, Aravind Rajeswaran, Sergey Levine, and Chelsea Finn.
\newblock Combo: Conservative offline model-based policy optimization.
\newblock \emph{Advances in neural information processing systems}, 34:\penalty0 28954--28967, 2021.

\bibitem[Zambaldi et~al.(2018)Zambaldi, Raposo, Santoro, Bapst, Li, Babuschkin, Tuyls, Reichert, Lillicrap, Lockhart, et~al.]{zambaldi2018relational}
Vinicius Zambaldi, David Raposo, Adam Santoro, Victor Bapst, Yujia Li, Igor Babuschkin, Karl Tuyls, David Reichert, Timothy Lillicrap, Edward Lockhart, et~al.
\newblock Relational deep reinforcement learning.
\newblock \emph{arXiv preprint arXiv:1806.01830}, 2018.

\bibitem[Zhao et~al.(2023)Zhao, Huang, Xu, Lin, Liu, and Huang]{zhao2023expel}
Andrew Zhao, Daniel Huang, Quentin Xu, Matthieu Lin, Yong-Jin Liu, and Gao Huang.
\newblock Expel: Llm agents are experiential learners.
\newblock \emph{arXiv preprint arXiv:2308.10144}, 2023.

\bibitem[Zhou et~al.(2020)Zhou, Li, and Gu]{zhou2020neural}
Dongruo Zhou, Lihong Li, and Quanquan Gu.
\newblock Neural contextual bandits with ucb-based exploration.
\newblock In \emph{International Conference on Machine Learning}, pp.\  11492--11502. PMLR, 2020.

\bibitem[Zhu \& Tan(2020)Zhu and Tan]{zhu2020thompson}
Qiuyu Zhu and Vincent Tan.
\newblock Thompson sampling algorithms for mean-variance bandits.
\newblock In \emph{International Conference on Machine Learning}, pp.\  11599--11608. PMLR, 2020.

\bibitem[Zintgraf et~al.(2019)Zintgraf, Shiarlis, Igl, Schulze, Gal, Hofmann, and Whiteson]{zintgraf2019varibad}
Luisa Zintgraf, Kyriacos Shiarlis, Maximilian Igl, Sebastian Schulze, Yarin Gal, Katja Hofmann, and Shimon Whiteson.
\newblock Varibad: A very good method for bayes-adaptive deep rl via meta-learning.
\newblock \emph{arXiv preprint arXiv:1910.08348}, 2019.

\end{thebibliography}
\bibliographystyle{rlj}

\appendix
\onecolumn
\section{Appendix}
\label{sec:appendix}

\subsection{Related Works}
\label{sec:related}
In this section, we briefly discuss related works.

In-context decision making \citep{laskin2022context,lee2023supervised} has emerged as an attractive alternative in Reinforcement Learning (RL) compared to updating the model parameters after collection of new data \citep{mnih2013playing, franccois2018introduction}. 
In RL the contextual data takes the form
of state-action-reward tuples representing a dataset of interactions with an unknown environment (task). In this paper, we will refer to this as the in-context data.
%
%
%
Recall that in many real-world settings, the underlying task can be structured with correlated features, and the reward can be highly non-linear. So specialized bandit algorithms fail to learn in these tasks.
%
To circumvent this issue, a learner can first collect in-context data consisting of just action indices $I_t$ and rewards $r_t$. Then it can leverage the representation learning capability of deep neural networks to learn a pattern across the in-context data and subsequently derive a near-optimal policy \citep{lee2023supervised, mirchandani2023large}. We refer to this learning framework as an in-context decision-making setting.

The in-context decision-making setting of \citet{sinii2023context} also allows changing the action space by learning an embedding over the action space yet also requires the optimal action during training. In contrast we do not require the optimal action as well as show that we can generalize to new actions without learning an embedding over them. Similarly, \citet{lin2023transformers} study the in-context decision-making setting of \citet{laskin2022context, lee2023supervised}, but they also require a greedy approximation of the optimal action. The \citet{ma2023rethinking} also studies a similar setting for hierarchical RL where they stitch together sub-optimal trajectories and predict the next action during test time. Similarly, \citet{liu2023reason} studies the in-context decision-making setting to predict action instead of learning a reward correlation from a short horizon setting.
In contrast we do not require a greedy approximation of the optimal action, deal with short horizon setting and changing action sets during training and testing, and predict the estimated means of the actions instead of predicting the optimal action.
A survey of the in-context decision-making approaches can be found in \citet{liu2023self}.

In the in-context decision-making setting, the learning model is first trained on supervised input-output examples with the in-context data during training. Then during test time, the model is asked to complete a new input (related to the context provided) without any update to the model parameters \citep{xie2021explanation, min2022rethinking}. Motivated by this, \citet{lee2023supervised} recently proposed the Decision Pretrained Transformers (\dpt) that exhibit the following properties: \textbf{(1)} During supervised pretraining of \dpt, predicting optimal actions alone gives rise to near-optimal decision-making algorithms for unforeseen task during test time. Note that \dpt\ does not update model parameters during test time and, therefore, conducts in-context learning on the unforeseen task. \textbf{(2)} \dpt\ improves over the in-context data used to pretrain it by exploiting latent structure. However, \dpt\ either requires the optimal action during training or if it needs to approximate the optimal action. For approximating the optimal action, it requires a large amount of data from the underlying task.

At the same time, learning the underlying data pattern from a few examples during training is becoming more relevant in many domains like chatbot interaction \citep{madotto2021few,semnani2023wikichat}, recommendation systems, healthcare \citep{ge2022few, liu2023large}, etc. This is referred to as few-shot learning.
However, most current RL decision-making systems (including in-context learners like \dpt) require an enormous amount of data to learn a good policy.

The in-context learning framework is related to the meta-learning framework \citep{bengio1990learning, schaul2010metalearning}.  Broadly, these techniques aim to learn the underlying latent shared structure within the training distribution of tasks, facilitating faster learning of novel tasks during test time. In the context of decision-making and reinforcement learning (RL), there exists a frequent choice regarding the specific 'structure' to be learned, be it the task dynamics \citep{fu2016one, nagabandi2018learning, landolfi2019model}, a task context identifier \citep{rakelly2019efficient, zintgraf2019varibad, liu2021decoupling}, or temporally extended skills and options \citep{perkins1999using, gupta2018meta, jiang2022learning}.

However, as we noted in the \Cref{sec:intro}, one can do a greedy approximation of the optimal action from the historical data using a weak demonstrator and a neural network policy \citep{finn2017model, rothfuss2018promp}. Moreover, the in-context framework generally is more agnostic where it learns the policy of the demonstrator \citep{duan2016rl, wang2016learning, mishra2017simple}. Note that both \dptg\ and \pred\ are different than algorithmic distillation \citep{laskin2022context, lu2023structured} as they do not distill an existing RL algorithm. moreover, in contrast to \dptg\ which is trained to predict the optimal action, the \pred\ is trained to predict the reward for each of the actions. This enables the \pred\ (similar to \dptg) to show to potentially emergent online and offline strategies at test time that automatically align with the task structure, resembling posterior sampling.

As we discussed in the \Cref{sec:intro}, in decision-making, RL, and imitation learning the transformer models are trained using autoregressive action prediction \citep{yang2023bayesian}. Similar methods have also been used in Large language models \citep{vaswani2017attention, roberts2019exploring}. One of the more notable examples is the Decision Transformers (abbreviated as DT) which utilizes a transformer to autoregressively model sequences of actions from offline experience data, conditioned on the achieved return \citep{chen2021decision, janner2021offline}. This approach has also been shown to be effective for multi-task settings \citep{lee2022multi}, and multi-task imitation learning with transformers \citep{reed2022generalist, brohan2022rt, shafiullah2022behavior}. However, the DT methods are not known to improve upon their in-context data, which is the main thrust of this paper \citep{brandfonbrener2022does, yang2022dichotomy}.

Our work is also closely related to the offline RL setting. In offline RL, the algorithms can formulate a policy from existing data sets of state, action, reward, and next-state interactions.
Recently, the idea of pessimism has also been introduced in an offline setting to address the challenge of distribution shift \citep{kumar2020conservative, yu2021combo, liu2020provably, ghasemipour2022so}. Another approach to solve this issue is policy regularization \citep{fujimoto2019off, kumar2019stabilizing, wu2019behavior, siegel2020keep, liu2019off}, or reuse data for related task \citep{li2020focal, mitchell2021offline}, or additional collection of data along with offline data \citep{pong2022offline}. However, all of these approaches still have to take into account the issue of distributional shifts. In contrast \pred\ and \dptg\ leverages the decision transformers to avoid these issues.
Both of these methods can also be linked to posterior sampling. Such connections between sequence modeling with transformers and posterior sampling have also been made in \citet{chen2021decision,muller2021transformers,lee2023supervised, yang2023bayesian}.

\subsection{Experimental Setting Information and Details of Baselines}
\label{sec:addl-expt}
In this section, we describe in detail the experimental settings and some baselines.

\subsubsection{Experimental Details}

\textbf{Linear Bandit:} We consider the setting when $f(\bx, \btheta_*) = \bx^\top\btheta_*$. Here $\bx \in \R^d$ is the action feature and $\btheta_*\in\R^d$ is the hidden parameter. 
For every experiment, we first generate tasks from $\cTp$. Then we sample a fixed set of actions from  $\N\left(\mathbf{0}, \bI_d / d\right)$ in $\R^d$ and this constitutes the features. 
Then for each task $m\in [M]$ we sample $\btheta_{m ,*} \sim \N\left(\mathbf{0}, \bI_d / d\right)$ to produce the means $\mu(m,a)=\left\langle\btheta_{m ,*}, \bx(m,a)\right\rangle$ for $a \in \A$ and $m\in [M]$.
Finally, note that we do not shuffle the data as the order matters. Also in this setting $\bx(m, a)$ for each $a\in\A$ is fixed for all tasks $m$.

\textbf{Non-Linear Bandit:} We now consider the setting when $f(\bx, \btheta_*) = 1/(1 + 0.5\cdot\exp(2\cdot\exp(- \bx^\top \btheta_*)))$. Again, here $\bx \in \R^d$ is the action feature, and $\btheta_*\in\R^d$ is the hidden parameter. 
Note that this is different than the generalized linear bandit setting \citep{filippi2010parametric, li2017provably}.
Again for every experiment, we first generate tasks from $\cTp$. Then we sample a fixed set of actions from  $\N\left(\mathbf{0}, \bI_d / d\right)$ in $\R^d$ and this constitutes the features. 
Then for each task $m\in [M]$ we sample $\btheta_{m ,*} \sim \N\left(\mathbf{0}, \bI_d / d\right)$ to produce the means $\mu(m,a)=1/(1 + 0.5\cdot\exp(2\cdot\exp(- \bx(m,a)^\top \btheta_{m,*})))$ for $a \in \A$ and $m\in [M]$.
Again note that in this setting $\bx(m, a)$ for each $a\in\A$ is fixed for all tasks $m$.

We use NVIDIA GeForce RTX 3090 GPU with 24GB RAM to load the GPT 2 Large Language Model. This requires less than 2GB RAM without data, and with large context may require as much as 20GB RAM.

\subsubsection{Details of Baselines}
\label{app:baseline-details}

\textbf{(1) \ts:} This baseline is the stochastic $A$-action bandit Thompson Sampling algorithm from \citet{thompson1933likelihood,agrawal2012analysis,russo2018tutorial,zhu2020thompson}. We briefly describe the algorithm below: At every round $t$ and each action $a$, \ts\  samples $\gamma_{m,t}(a) \sim \N(\wmu_{m,t-1}(a), \sigma^2/N_{m,t-1}(a))$, where $N_{m,t-1}(a)$ is the number of times the action $a$ has been selected till $t-1$, and $\wmu_{m,t-1}(a) = \frac{\sum_{s=1}^{t-1} \wr_{m,s}\mathbf{1}(I_s =a)}{N_{m,t-1}(a)}$ is the empirical mean. Then the action selected at round $t$ is $I_{t} = \argmax_a \gamma_{m,t}(a)$. 
Observe that \ts\ is not a deterministic algorithm like \ucb\ \citep{auer2002finite-time}. So we choose \ts\ as the weak demonstrator $\pi^w$ because it is more exploratory than \ucb\ and also chooses the optimal action, $a_{m, *}$, a sufficiently large number of times. 
\ts\ is a weak demonstrator as it does not have access to the feature set $\X$ for any task $m$. 

\textbf{(2) \linucb:} (Linear Upper Confidence Bound): This baseline is the Upper Confidence Bound algorithm for the linear bandit setting that selects the action $I_t$ at round $t$ for task $m$ that is most optimistic and reduces the uncertainty of the task unknown parameter $\btheta_{m,*}$.
To balance exploitation and exploration between choosing different items the \linucb\ computes an upper confidence value to the estimated mean of each action $\bx_{m,a} \in \X$. 
This is done as follows: At every round $t$ for task $m$, it calculates the ucb value $B_{m,a,t}$ for each action $\bx_{m,a} \in \X$ such that $B_{m,a,t} = \bx_{m,a}^\top \wtheta_{m,t-1} + \alpha\|\bx_{m,a}\|_{\bSigma_{m,t-1}^{-1}}$ where $\alpha > 0$ is a constant and $\wtheta_{m,t}$ is the estimate of the model parameter $\btheta_{m, *}$ at round $t$. 
Here, $\bSigma_{m,t-1} = \sum_{s=1}^{t-1}\bx_{m,s}\bx_{m,s}^\top +\lambda\bI_d$ is the data covariance matrix or the arms already tried.
Then it chooses $I_t = \argmax_{a}B_{m,a,t}$. 
Note that \linucb\ is a \textit{strong} demonstrator that we give oracle access to the features of each action; other algorithms do not observe the features.
Hence, in linear bandits, \linucb\ provides an approximate upper bound on the performance of all algorithms.

\textbf{(3) \mlin:} This is the multi-task linear regression bandit algorithm proposed by \citet{yang2021impact}. This algorithm assumes that there is a common low dimensional feature extractor $\mathbf{B}\in\R^{k\times d}$, $k\leq d$ shared between the tasks and the rewards per task $m$ are linearly dependent on a hidden parameter $\btheta_{m,*}$. 
Under a diversity assumption (which may not be satisfied in real data) and $\bW=\left[\bw_1, \ldots, \bw_M\right]$ they assume $\bTheta =\left[\btheta_{1,*}, \ldots, \btheta_{M,*}\right]=\bB \bW$.
During evaluation \mlin\ estimates the $\mathbf{\wB}$ and $\widehat{\mathbf{W}}$ from training data and fit $\wtheta_{m}=\mathbf{\wB} \widehat{\bw}_m$ per task and selects action greedily based on $I_{m,t} = \argmax_a \bx_{m,a}^\top\wtheta_{m,*}$. Finally, note that \mlin\ requires access to the action features to estimate $\wtheta_{m}$ and select actions as opposed to \dpt, \ad, and \pred.




\subsection{Empirical Study: Comparison against K-armed bandits and \dpt}
\label{sec:k_arms_dpt}
In this section, we discuss the performance of \pred\ (\gt) when there is no latent structure in the data, that is the $K$-armed bandits. Then we compare the performance of \pred\ (\gt) against \dpt.

\textbf{Baselines:} In the K-armed bandits We implement the same baselines discussed in \Cref{sec:short-horizon}. The baselines are \pred, \predt, \dptg, \ad, \ts, and \linucb. In the linear and non-linear setting, we compare against \dpt\ instead of \dptg. 

\textbf{Settings:} In the K-armed bandit setting we consider $d=6$, and the arms as canonical vectors $\be_1, \be_2, \ldots, \be_6$. For each task $m$, we choose the hidden parameter $\btheta_{m,*}$ similar to the linear bandit setting discussed in \Cref{sec:linear}. Note that this results in a K-armed bandit setting. For the linear and non-linear setting comparison, we use the same setting as \Cref{sec:linear}, and \ref{sec:short-horizon}.

\textbf{Outcomes:} We first discuss the main outcomes from our experimental results in K-armed bandits and then in comparison against \dpt\ in linear and non-linear settings.

\begin{tcolorbox}
\customfinding \pred\ (\gt) matches the performance of the demonstrator when there is no structure (K-armed bandits). \pred\ (\gt) performs close to \dpt\ in the linear and non-linear setting showing the usefulness of learning the reward structure.
\end{tcolorbox}

\begin{figure}[!hbt]
\centering
\begin{subfigure}[b]{0.3\textwidth}
   \includegraphics[scale=0.13]{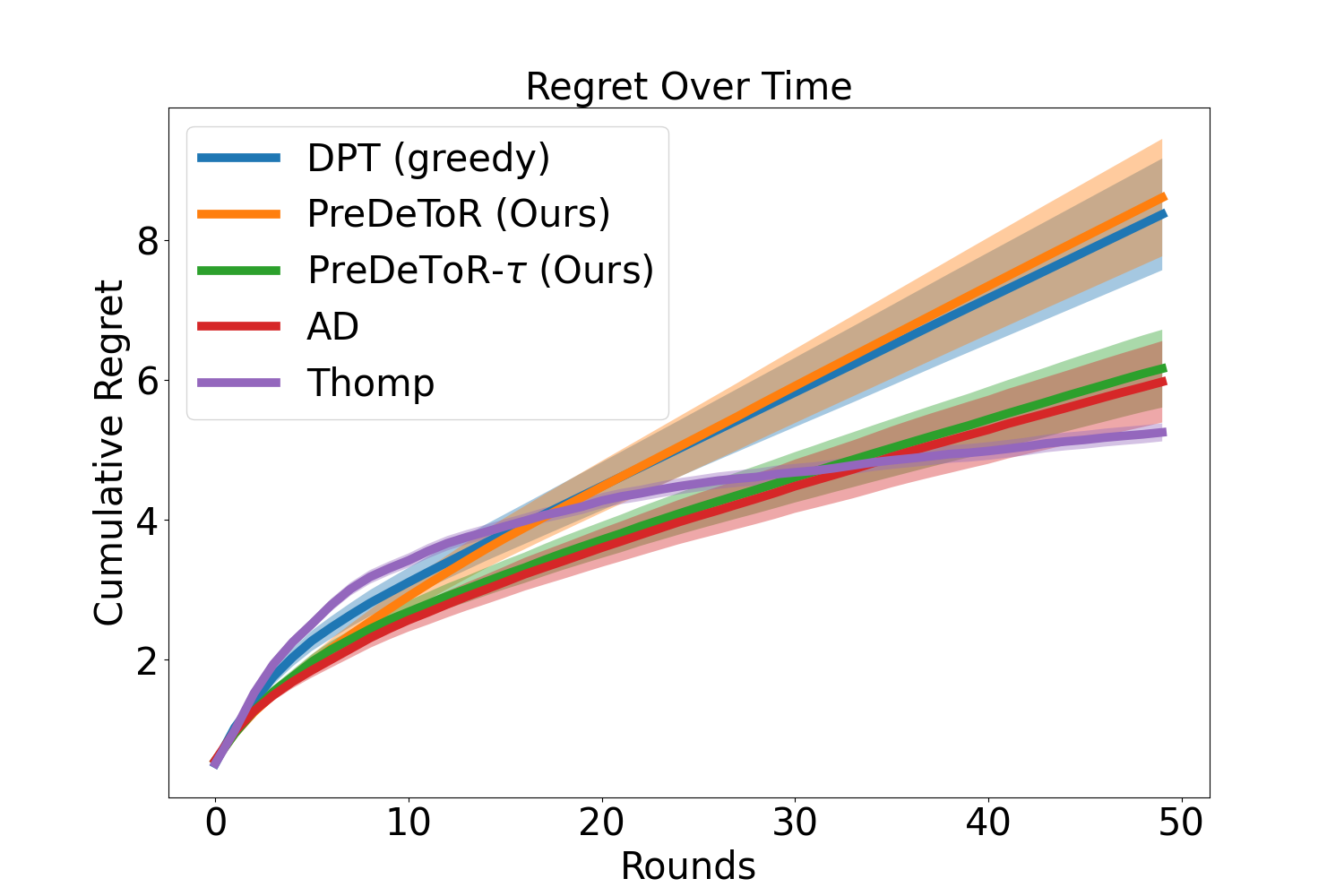}
   \caption{K-armed Bandit}
   \label{fig:dpt-1}
\end{subfigure}%
\begin{subfigure}[b]{0.3\textwidth}
   \includegraphics[scale=0.13]{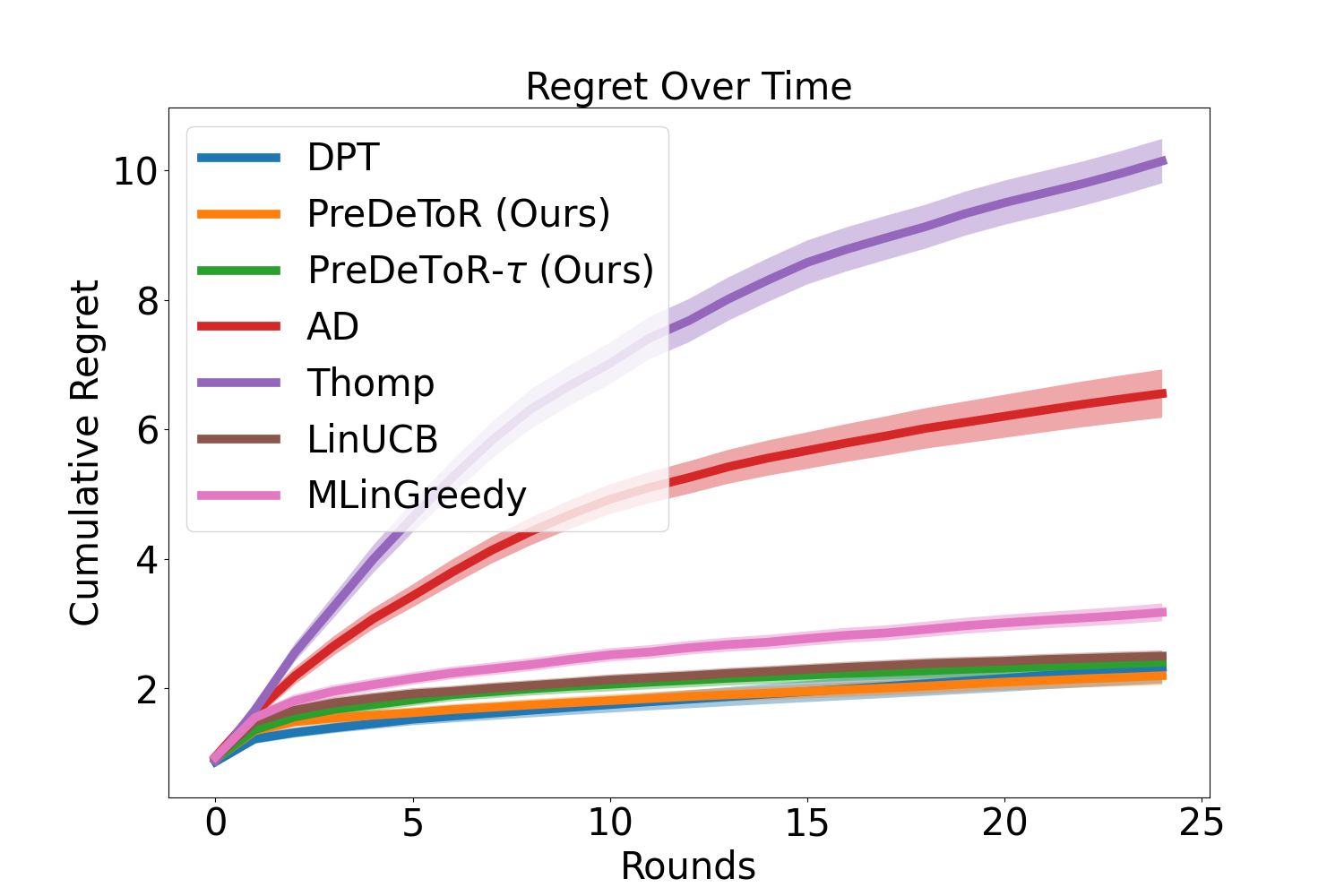}
   \caption{Comparison against \dpt\ in linear setting}
   \label{fig:dpt-2}
\end{subfigure}%
\begin{subfigure}[b]{0.3\textwidth}
   \includegraphics[scale=0.13]{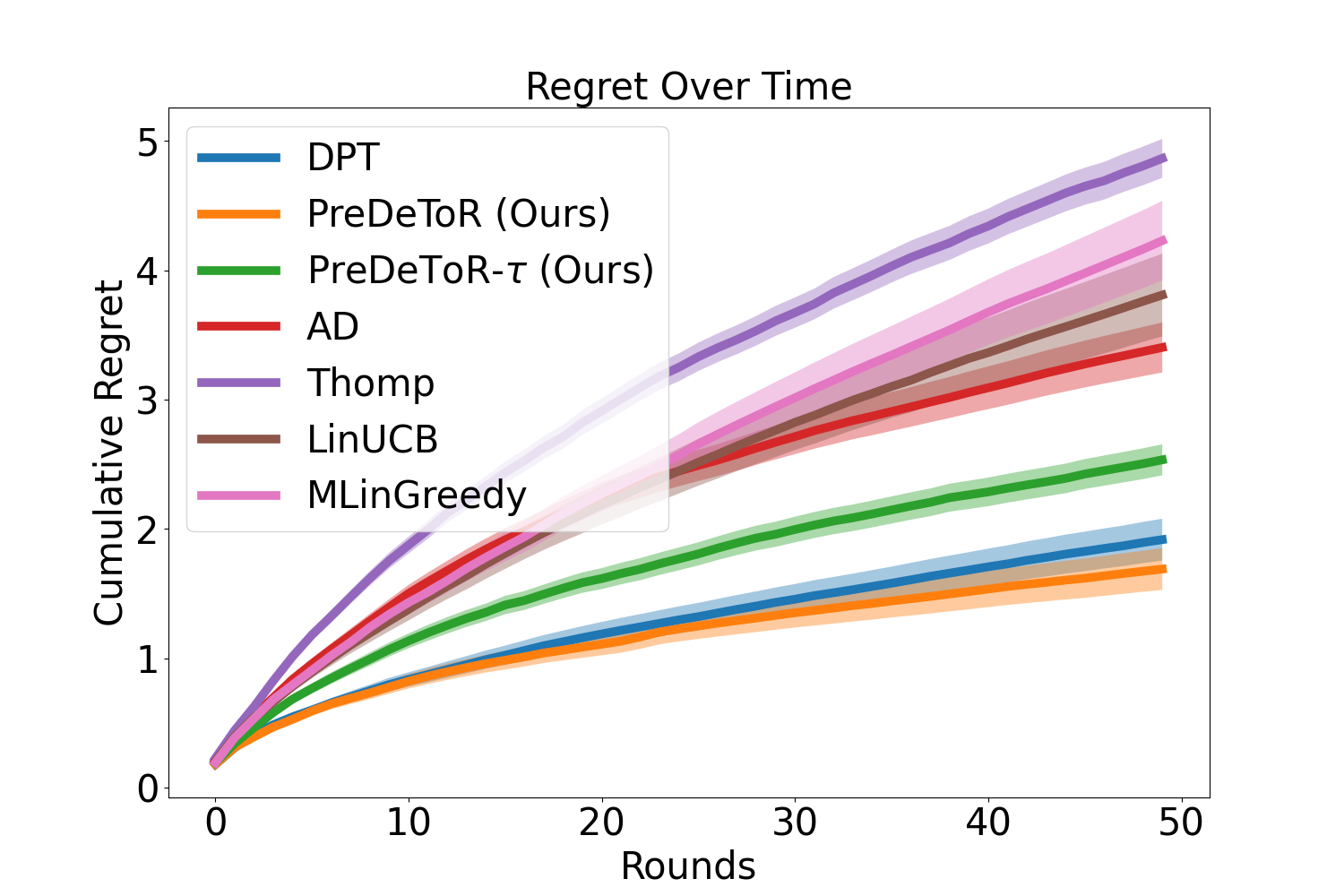}
   \caption{Comparison against \dpt\ in non-linear setting}
   \label{fig:dpt-3}
\end{subfigure}
\vspace*{-1em}
\caption{Experiment with k-armed bandits and \dpt\ (original). The y-axis shows the cumulative regret.}
\label{fig:expt-k-arms-dpt}
\end{figure}

\textbf{Experimental Result:} We observe these outcomes in \Cref{fig:expt-k-arms-dpt}. In \Cref{fig:dpt-1} the demonstrator $\pi^w$ is the \ts\ algorithm. \textit{Note that there is no structure across arms now, and sampling one arm gives no information about other arms in a task.} We observe that \predt\ performs similarly to the demonstrator \ts, and also shows that incorporating exploration is a sound technique. Also, \ad\ performs similarly to the demonstrator \ts. Both \dptg\ and \pred\ fail to learn the latent structure across the tasks and therefore do not learn any exploration strategy.

In \Cref{fig:dpt-2} we show the linear bandit setting discussed in \Cref{sec:addl-expt}. We observe that \pred\ (\gt) matches the performance of \dpt, and \linucb. Note that \dpt\ has access to the optimal action per task, and \linucb\ is the optimal oracle algorithm that leverages the structure information. 

In \Cref{fig:dpt-3} we show the non-linear bandit setting discussed in \Cref{sec:addl-expt}. 
Again we observe that \pred\ (\gt) matches the performance of \dpt\ and has lower cumulative regret than \ad\ and \linucb\ which fails to perform well in this non-linear setting due to its algorithmic design.

\subsection{Empirical Study: Bilinear Bandits}
\label{sec:bilinear}
In this section, we discuss the performance of \pred\ against the other baselines in the bilinear setting. 
Again note that the number of tasks $\Mpr \gg A \geq n$.
%
%
Through this experiment, we want to evaluate the performance of \pred\ to exploit the underlying latent structure and reward correlation when the horizon is small, the number of tasks is large, and understand its performance in the bilinear bandit setting \citep{jun2019bilinear, lu2021low, kang2022efficient, mukherjee2023multi}.
Note that this setting also goes beyond the linear feedback model \citep{abbasi2011improved, lattimore2020bandit} and is related to matrix bandits \citep{yang2020reinforcement}.

\textbf{Bilinear bandit setting:} In the bilinear bandits the learner is provided with two sets of action sets, $\X\subseteq\R^{d_1}$ and $\Z\subseteq\R^{d_2}$ which are referred to as the left and right action sets. At every round $t$ the learner chooses a pair of actions $\bx_t\in\X$ and $\bz_t\in\Z$ and observes a reward 
\begin{align*}
    r_t = \bx_t^\top \bTheta_* \bz_t + \eta_t
\end{align*}
where $\bTheta_*\in\R^{d_1\times d_2}$ is the unknown hidden matrix which is also low-rank. The $\eta_t$ is a $\sigma^2$ sub-Gaussian noise. In the multi-task bilinear bandit setting we now have a set of $M$ tasks where the reward for the $m$-th task at round $t$ is given by
\begin{align*}
    r_{m,t} = \bx_{m,t}^\top \bTheta_{m,*} \bz_{m,t} + \eta_{m,t}.
\end{align*}
Here  $\bTheta_{m,*}\in\R^{d_1\times d_2}$ is the unknown hidden matrix for each task $m$, which is also low-rank. The $\eta_{m,t}$ is a $\sigma^2$ sub-Gaussian noise. Let $\kappa$ be the rank of each of these matrices $\bTheta_{m, *}$.

A special case is the rank $1$ structure where $\bTheta_{m, *} = \btheta_{m,*}\btheta_{m,*}^\top$ where $\bTheta_{m,*} \in \R^{d \times d}$ and $\btheta_{m,*}\in\R^d$ for each task $m$. Let the left and right action sets be also same such that $\bx_{m,t} \in \X \subseteq \R^{d}$. Observe then that the reward for the $m$-th task at round $t$ is given by
\begin{align*}
    r_{m,t} = \bx_{m,t}^\top \bTheta_{m,*} \bx_{m,t} + \eta_{m,t} = (\bx_{m,t}^\top\btheta_{m, *})^2 + \eta_{m,t}.
\end{align*}
This special case is studied in \citet{chaudhuri2017active}.

\textbf{Baselines:} We again implement the same baselines discussed in \Cref{sec:short-horizon}. The baselines are \pred, \predt, \dptg, and \ts. 
Note that we do not implement the \linucb\ and \mlin\ for the bilinear bandit setting.
However, we now implement the \estr\ \citep{jun2019bilinear} which is optimal in the bilinear bandit setting.

\textbf{\estr:} The \estr\ algorithm first estimates the unknown parameter $\bTheta_{m,*}$ for each task $m$ using E-optimal design \citep{pukelsheim2006optimal, fedorov2013theory, jun2019bilinear} for $n_1$ rounds. Let $\wTheta_{m,n_1}$ be the estimate of $\bTheta_{m,*}$ at the end of $n_1$ rounds. Let the SVD of $\wTheta_{m,n_1}$ be given by $\operatorname{SVD}(\wTheta_{m,n_1}) = \wU_{m,n_1}\widehat{\mathbf{S}}_{m, n_1}\wV_{m, n_1}$.
Then \estr\ rotates the actions as follows:
\begin{align*}
    \X^{\prime}_m=\left\{\left[\wU_{m,n_1} \wU^{\perp}_{m,n_1}\right]^{\top} \bx_m: \bx_m \in \X\right\} \textbf{ and } \Z^{\prime}=\left\{\left[\wV_{m,n_1} \wV^{\perp}_{m,n_1}\right]^{\top} \bz_m: \bz_m \in \Z\right\}.
\end{align*}
Then defines a vectorized action set for each task $m$ so that the last $\left(d_1-\kappa\right) \cdot\left(d_2-\kappa\right)$ components are from the complementary subspaces:
\begin{align*}
\widetilde{\A}_m&=\left\{\left[\operatorname{vec}\left(\bx_{m, 1: \kappa} \bz_{m, 1: \kappa}^{\top}\right) ; \operatorname{vec}\left(\bx_{m, \kappa+1: d_1} \bz_{m, 1: \kappa}^{\top}\right) ; \operatorname{vec}\left(\bx_{m, 1: \kappa} \bz_{m, \kappa+1: d_2}^{\top}\right) ; \right.\right.\\
&\quad\quad\left.\left.\operatorname{vec}\left(\bx_{m, \kappa+1: d_1} \bz_{m, \kappa+1: d_2}^{\top}\right)\right] \in \mathbb{R}^{d_1 d_2}: \bx_m \in \X^{\prime}_m, \bz_m \in \Z^{\prime}_m\right\} .
\end{align*}
Finally for $n_2=n-n_1$ rounds, \estr\ invokes the specialized OFUL algorithm \citep{abbasi2011improved}  for the rotated action set $\widetilde{\A}_m$ with the low dimension $k=\left(d_1+d_2\right) \kappa-\kappa^2$.
%
Note that the \estr\ runs the per-task low dimensional OFUL algorithm rather than learning the underlying structure across the tasks \citep{mukherjee2023multi}. 

\textbf{Outcomes:} We first discuss the main outcomes of our experimental results for increasing the horizon:




\begin{tcolorbox}
\customfinding \pred\ (\gt) outperforms \dptg, \ad, and matches the performance of \estr\ in bilinear bandit setting. 
\end{tcolorbox}


\begin{figure}[!hbt]
\centering
\begin{subfigure}[b]{0.48\textwidth}
    \includegraphics[scale=0.15]{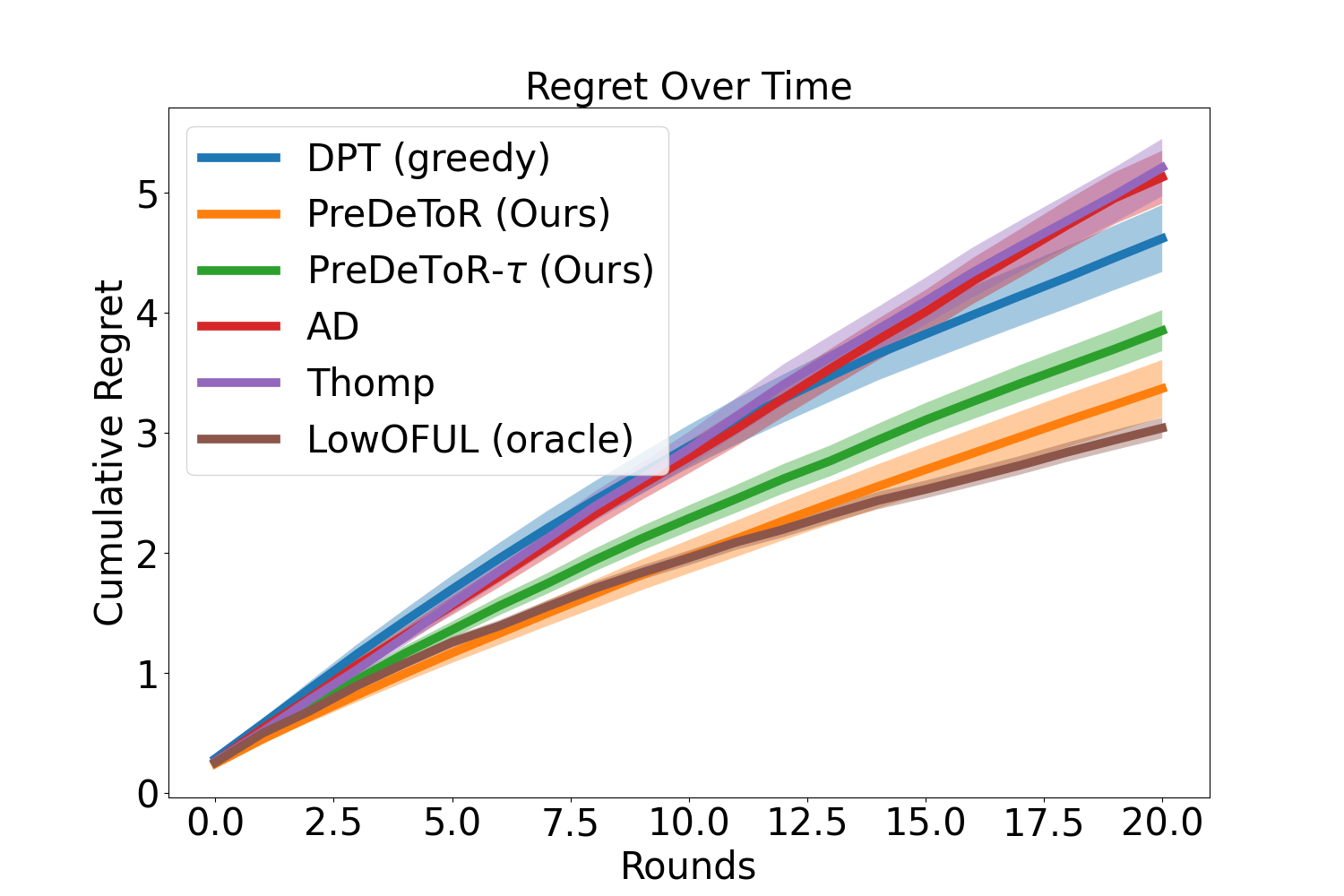}
    \caption{Rank $1$ $\bTheta_{m, *}$}
    \label{fig:bilin-1}
\end{subfigure}%
\begin{subfigure}[b]{0.48\textwidth}
    \includegraphics[scale=0.15]{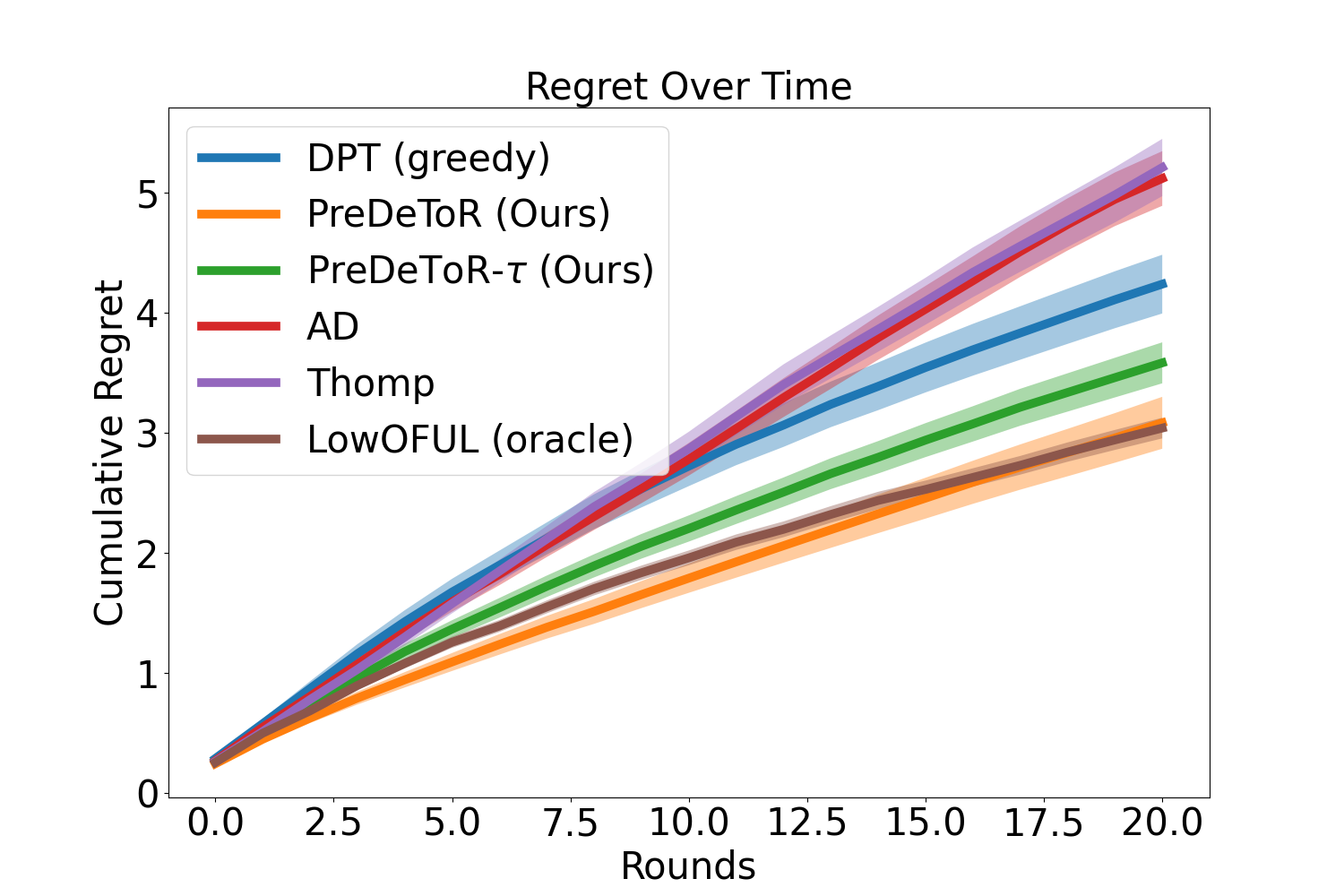}
    \caption{Rank $2$ $\bTheta_{m, *}$}
    \label{fig:bilin-2}
\end{subfigure}
\vspace*{-1em}
\caption{Experiment with bilinear bandits. The y-axis shows the cumulative regret.}
\label{fig:expt-bilin}
\end{figure}

\textbf{Experimental Result:} We observe these outcomes in \Cref{fig:expt-bilin}. In \Cref{fig:bilin-1} we experiment with rank $1$ hidden parameter $\bTheta_{m,*}$ and set horizon $n=20$, $\Mpr =  200000$, $\Mts = 200$, $A=30$, and $d=5$. 
In \Cref{fig:bilin-2} we experiment with rank $2$ hidden parameter $\bTheta_{m,*}$ and set horizon $n=20$, $\Mpr =  250000$, $\Mts = 200$, $A=25$, and $d=5$. 
Again, the demonstrator $\pi^w$ is the \ts\ algorithm. We observe that \pred\ has lower cumulative regret than \dptg, \ad\ and \ts. Note that for any task $m$ for the horizon $20$ the \ts\ will be able to sample all the actions at most once.
%
%
Note that for this small horizon setting the \dptg\ does not have a good estimation of $\widehat{a}_{m,*}$ which results in a poor prediction of optimal action $\widehat{a}_{m,t,*}$. In contrast \pred\ learns the correlation of rewards across tasks and can perform well.
Observe from \Cref{fig:bilin-1}, and \ref{fig:bilin-2} that \pred\ has lower regret than \ts\ and matches \estr. Also, in this low-data regime it is not enough for \estr\ to learn the underlying $\bTheta_{m, *}$ with high precision. Hence, \pred\ also has slightly lower regret than \estr. 
%
Note that the main objective of \ad\ is to match the performance of its demonstrator.
Most importantly it shows that \pred\ can exploit the underlying latent structure and reward correlation better than \dptg, and \ad.

\subsection{Empirical Study: Latent Bandits}
\label{sec:latent}
In this section, we discuss the performance of \pred\ (\gt) against the other baselines in the latent bandit setting and create a generalized bilinear bandit setting. 
Note that the number of tasks $\Mpr \gg A \geq n$.
%
%
Using this experiment, we want to evaluate the ability of \pred\ (\gt) to exploit the underlying reward correlation when the horizon is small, the number of tasks is large, and understand its performance in the latent bandit setting \citep{hong2020latent, maillard2014latent, pal2023optimal, kveton2017stochastic}.
We create a latent bandit setting which generalizes the bilinear bandit setting \citep{jun2019bilinear, lu2021low, kang2022efficient, mukherjee2023multi}.
Again note that this setting also goes beyond the linear feedback model \citep{abbasi2011improved, lattimore2020bandit} and is related to matrix bandits \citep{yang2020reinforcement}.

\textbf{Latent bandit setting:} In this special multi-task latent bandits the learner is again provided with two sets of action sets, $\X\subseteq\R^{d_1}$ and $\Z\subseteq\R^{d_2}$ which are referred to as the left and right action sets. 
%
The reward for the $m$-th task at round $t$ is given by
\begin{align*}
    r_{m,t} = \bx_{m,t}^\top \underbrace{(\bTheta_{m,*} + \bU\bV^\top)}_{\bZ_{m_*}}\bz_{m,t} + \eta_{m,t}.
\end{align*}
Here  $\bTheta_{m,*}\in\R^{d_1\times d_2}$ is the unknown hidden matrix for each task $m$, which is also low-rank. 
Additionally, all the tasks share a \emph{common latent parameter matrix} $\bU\bV^\top \in \R^{d_1\times d_2}$ which is also low rank. Hence the learner needs to learn the latent parameter across the tasks hence the name latent bandits.
Finally, the $\eta_{m,t}$ is a $\sigma^2$ sub-Gaussian noise. Let $\kappa$ be the rank of each of these matrices $\bTheta_{m, *}$ and $\bU\bV^\top$.
Again special case is the rank $1$ structure where the reward for the $m$-th task at round $t$ is given by
\begin{align*}
    r_{m,t} = \bx_{m,t}^\top \underbrace{(\btheta_{m,*}\btheta_{m,*}^\top + \bu\bv^\top)}_{\bZ_{m, *}}\bx_{m,t} + \eta_{m,t}.
\end{align*}
where $\btheta_{m,*}\in\R^d$ for each task $m$ and $\bu, \bv\in\R^d$. Note that the left and right action sets are the same such that $\bx_{m,t} \in \X \subseteq \R^{d}$.

\textbf{Baselines:} We again implement the same baselines discussed in \Cref{sec:short-horizon}. The baselines are \pred, \predt, \dptg, \ad, \ts, and \estr. However, we now implement a special \estr\ (stated in \Cref{sec:bilinear}) which has knowledge of the shared latent parameters $\bU$, and $\bV$. We call this the \estro\ algorithm. Therefore \estro\ has knowledge of the problem parameters in the latent bandit setting and hence the name. 
Again note that we do not implement the \linucb\ and \mlin\ for the latent bandit setting.

\textbf{Outcomes:} We first discuss the main outcomes of our experimental results for increasing the horizon:

\begin{tcolorbox}
\customfinding \pred\ (\gt) outperforms \dptg, \ad, and matches the performance of \estro\ in latent bandit setting. 
\end{tcolorbox}





\begin{figure}[!hbt]
\centering
\begin{subfigure}[b]{0.32\textwidth}
   \includegraphics[scale=0.13]{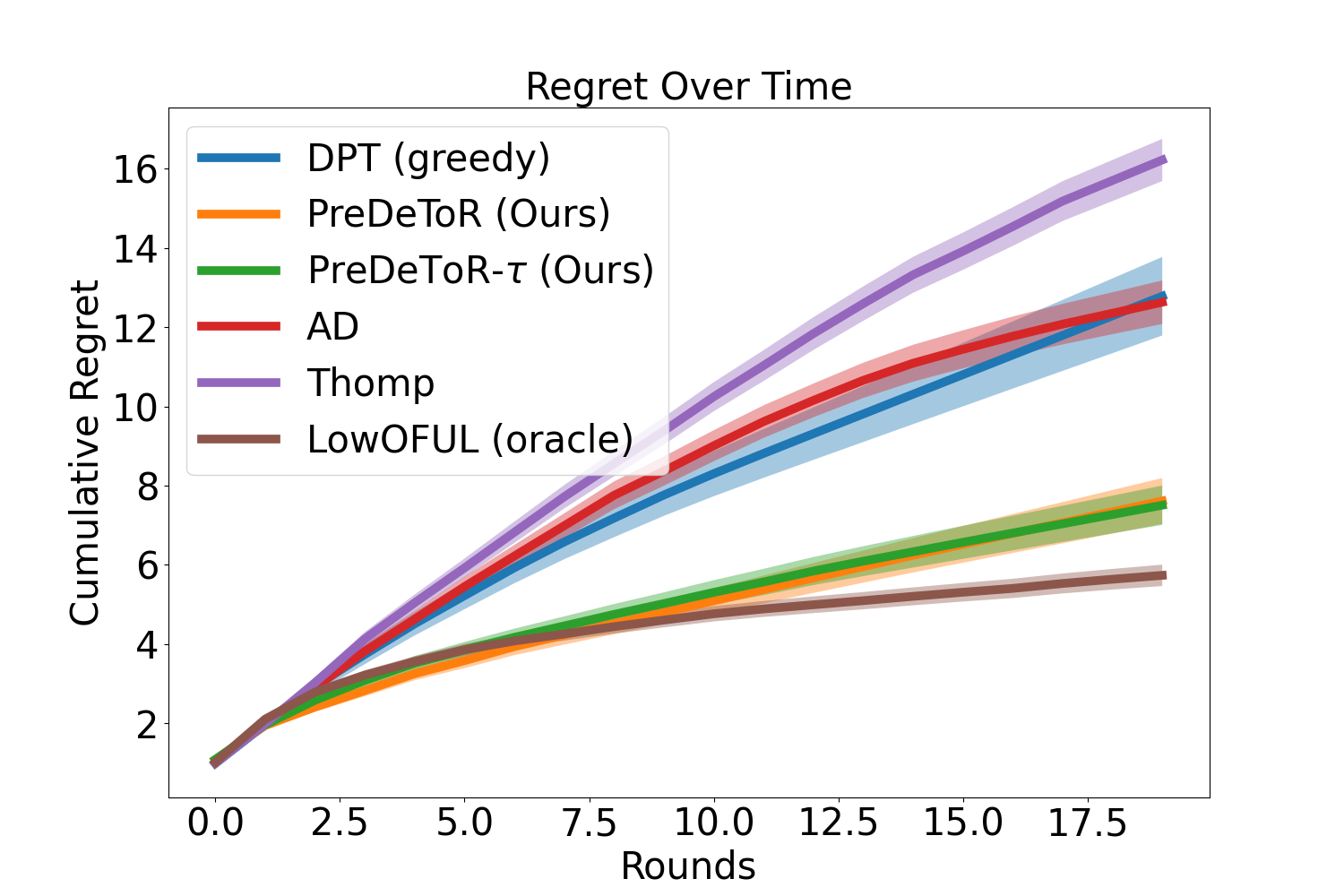}
   \caption{Rank $1$ $\bZ_{m,*}$}
   \label{fig:latent-1}
\end{subfigure}%
\begin{subfigure}[b]{0.32\textwidth}
   \includegraphics[scale=0.13]{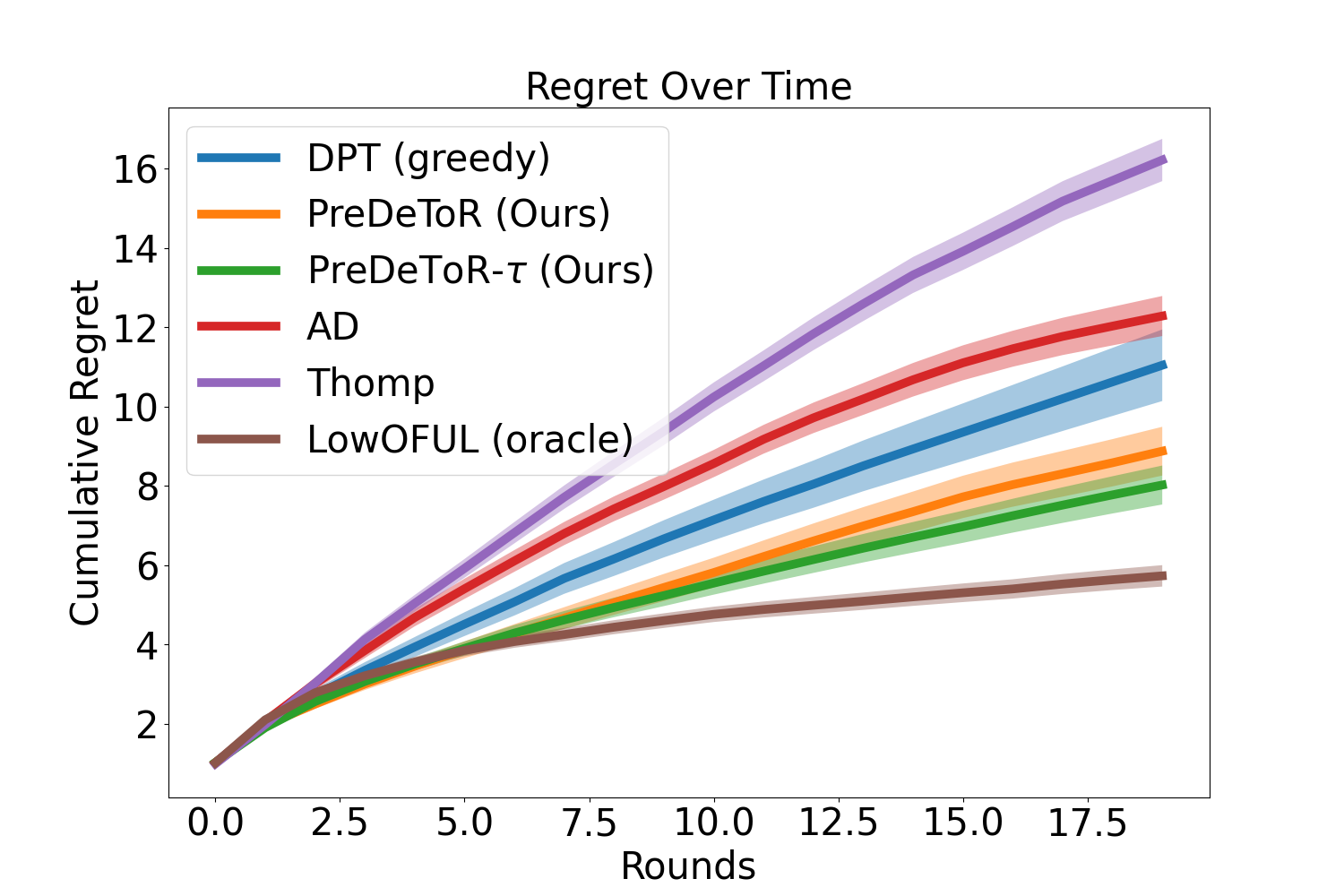}
   \caption{Rank $2$ $\bZ_{m,*}$}
   \label{fig:latent-2}
\end{subfigure}%
\begin{subfigure}[b]{0.32\textwidth}
   \includegraphics[scale=0.13]{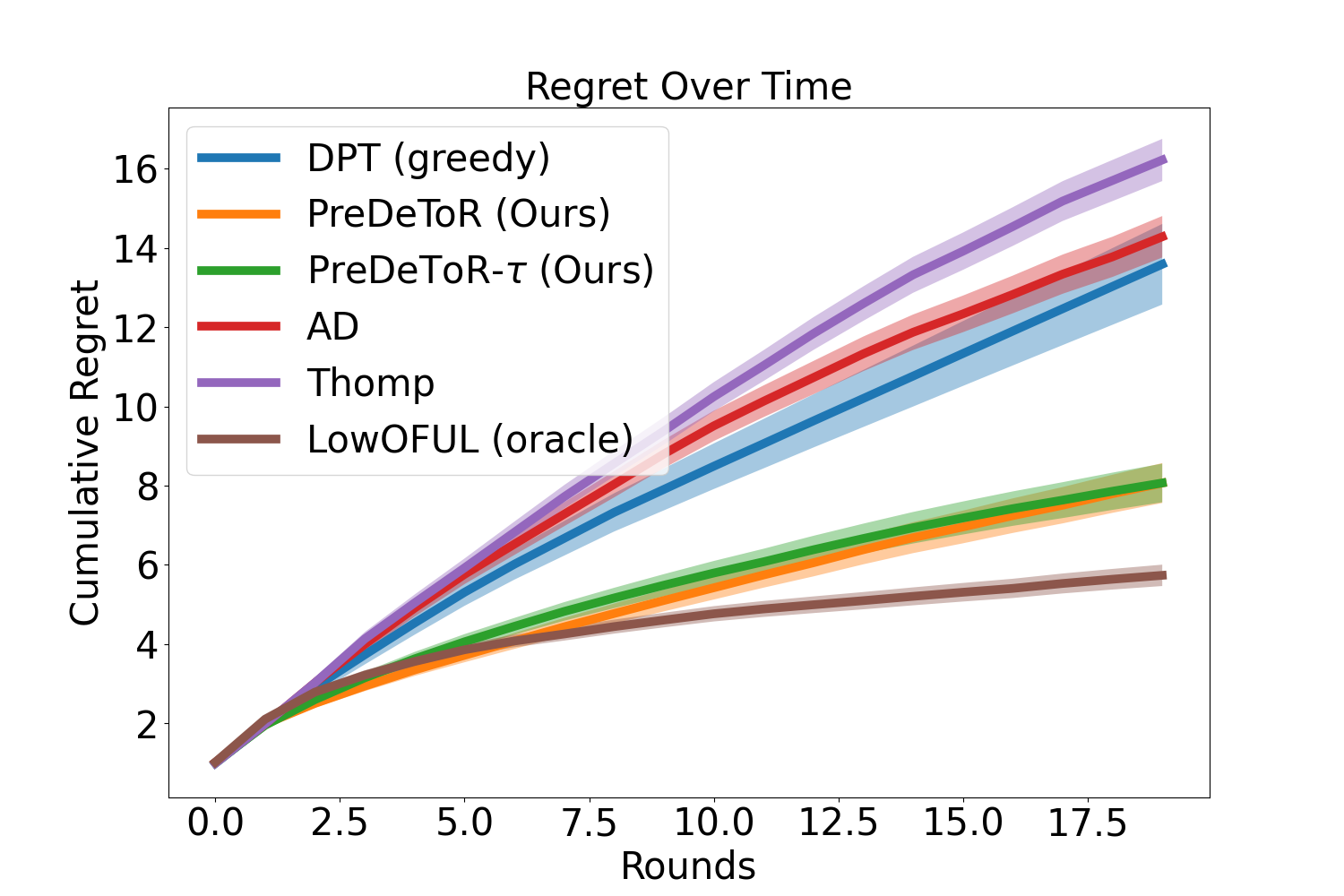}
   \caption{Rank $3$ $\bZ_{m,*}$}
   \label{fig:latent-3}
\end{subfigure}
\vspace*{-1em}
\caption{Experiment with latent bandits. The y-axis shows the cumulative regret.}
\label{fig:expt-latent}
\end{figure}

\textbf{Experimental Result:} We observe these outcomes in \Cref{fig:expt-latent}. In \Cref{fig:latent-1} we experiment with rank $1$ hidden parameter $\btheta_{m,*}\btheta_{m,*}^\top$ and latent parameters $\bu\bv^\top$ shared across the tasks and set horizon $n=20$, $\Mpr =  200000$, $\Mts = 200$, $A=30$, and $d=5$. 
In \Cref{fig:latent-2} we experiment with rank $2$ hidden parameter $\bTheta_{m,*}$, and latent parameters $\bU\bV^\top$ and set horizon $n=20$, $\Mpr =  250000$, $\Mts = 200$, $A=25$, and $d=5$. 
In \Cref{fig:latent-3} we experiment with rank $3$ hidden parameter $\bTheta_{m,*}$, and latent parameters $\bU\bV^\top$ and set horizon $n=20$, $\Mpr =  300000$, $\Mts = 200$, $A=25$, and $d=5$. 
Again, the demonstrator $\pi^w$ is the \ts\ algorithm. We observe that \pred\ (\gt) has lower cumulative regret than \dptg, \ad\ and \ts. Note that for any task $m$ for the horizon $20$ the \ts\ will be able to sample all the actions at most once.
%
%
Note that for this small horizon setting the \dptg\ does not have a good estimation of $\widehat{a}_{m,*}$ which results in a poor prediction of optimal action $\widehat{a}_{m,t,*}$. In contrast \pred\ (\gt) learns the correlation of rewards across tasks and is able to perform well.
Observe from \Cref{fig:latent-1}, \ref{fig:latent-2}, and \ref{fig:latent-3} that \pred\ has lower regret than \ts\ and has regret closer to \estro which has access to the problem-dependent parameters.
Hence. \estro\ outperforms \pred\ (\gt) in this setting.
%
%
This shows that \pred\ is able to exploit the underlying latent structure and reward correlation better than \dptg, and \ad.

\subsection{Connection between \pred\ and Linear Multivariate Gaussian Model}
\label{sec:und}
In this section, we try to understand the behavior of \pred\ and its ability to exploit the reward correlation across tasks under a \emph{linear multivariate Gaussian model}. In this model, the hidden task parameter, $\btheta_*$, is a random variable drawn from a multi-variate Gaussian distribution \citep{bishop2006pattern} and the feedback follows a linear model.
We study this setting since we can estimate the Linear Minimum Mean Square Estimator (LMMSE) in this setting \citep{carlin2008bayesian, box2011bayesian}. 
This yields a posterior prediction for the mean of each action over all tasks on average, by leveraging the linear structure when $\btheta_*$ is drawn from a multi-variate Gaussian distribution. 
So we can compare the performance of \pred\ against such an LMMSE and evaluate whether it is exploiting the underlying linear structure and the reward correlation across tasks. We summarize this as follows:

\begin{tcolorbox}
\customfinding \pred\ learns the reward correlation covariance matrix from the in-context training data $\Htr$ and acts greedily on it.
\end{tcolorbox}

\begin{wrapfigure}{L}{0.4\textwidth}
\centering
   \hspace*{-1.2em}\includegraphics[scale = 0.15]{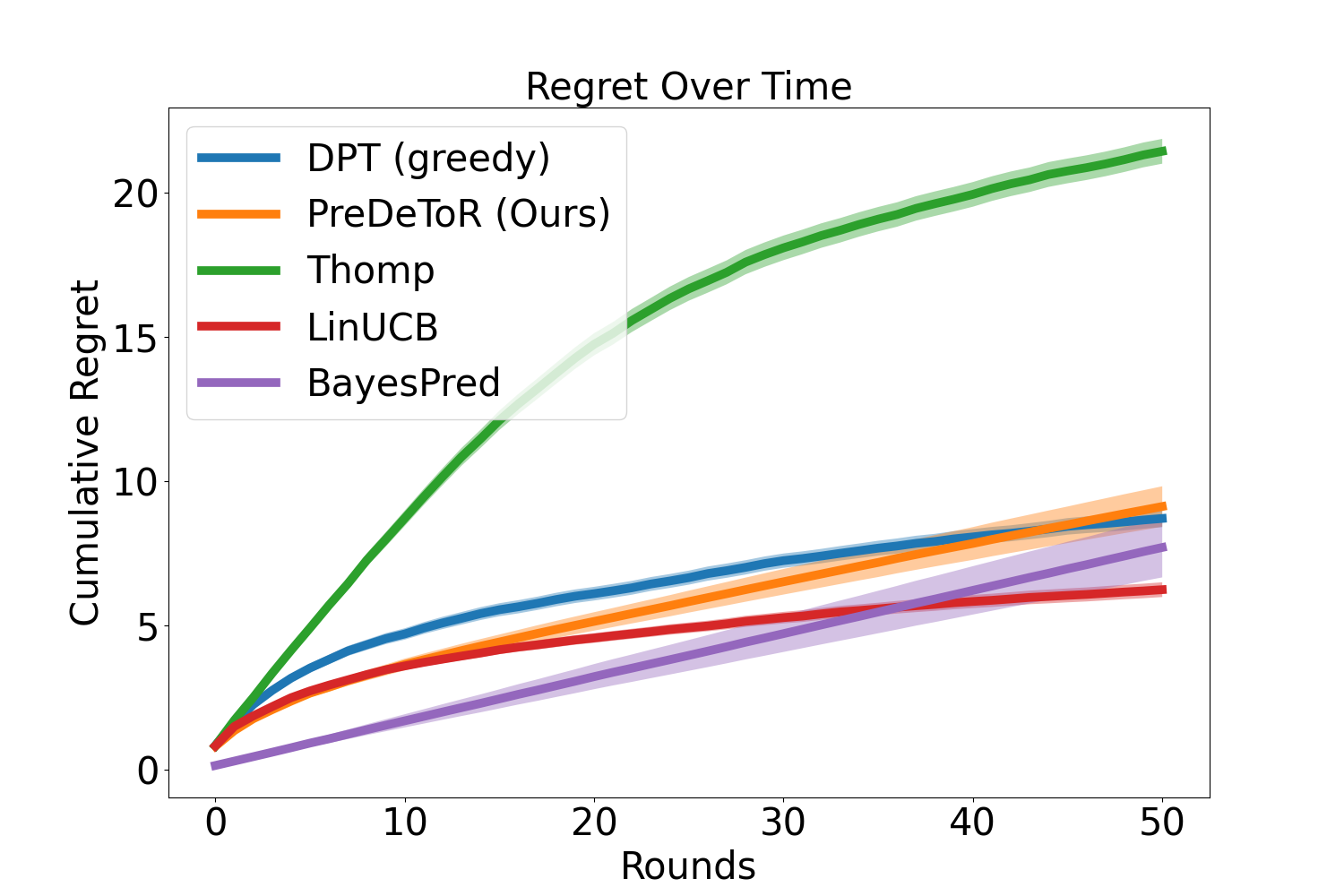}
   \vspace*{-1.0em}
  \caption{\label{fig:bp-tr}\bayes\ Performance}
\end{wrapfigure}
Consider the linear feedback setting consisting of $A$ actions and the hidden task parameter $\btheta_{*}\sim\N(0,\sigma^2_{\btheta}\bI_d)$. The reward of the action $\bx_{t}$ at round $t$ is given by $r_{t} = \bx_{t}^\top\btheta_{*} + \eta_{t}$, where $\eta_t$ is $\sigma^2$ sub-Gaussian.
Let $\pi^w$ collect $n$ rounds of pretraining in-context data and observe $\{I_t, r_t\}_{t=1}^n$. Let $N_n(a)$ denote the total number of times the action $a$ is sampled for $n$ rounds. Note that we drop the task index $m$ in these notations as the random variable $\btheta_{*}$ corresponds to the task.
Define the matrix $\bH_n \in \R^{n \times A}$ where the $t$-th row represents the action $I_t$ for $t\in [n]$. 
The $t$-th row of $\bH_n$ is a one-hot vector with the  $I_t$-th component being 1. We represent each action by one hot vector because we assume that this LMMSE does not have access to the feature vectors of the actions similar to the \pred~for fair comparison.
Then define the reward vector $\bY_n \in \R^n$ where the $t$-th component is the reward $r_t$ observed for the action $I_t$ for $t\in[n]$ in the pretraining data.
%
Define the diagonal matrix $\bD_A\in\R^{A\times A}$ estimated from pretraining data as follows
\begin{align}
    \bD_A(i,i) &= \begin{cases}
        \frac{\sigma^2}{N_n(a)}, \text{ if } N_n(a) > 0\\ \label{eq:D-matrix}
        = 0, \text{ if } N_n(a) = 0
    \end{cases}
    \vspace*{-1em}
\end{align}
where the reward noise being $\sigma^2$ sub-Gaussian is known. 
Finally define the estimated reward covariance matrix $\bS_A\in\R^{A\times A}$ as $\bS_A(a,a') = \wmu_n(a)\wmu_n(a')$, where $\wmu_n(a)$ is the empirical mean of action $a$ estimated from the pretraining data. This matrix captures the reward correlation between the pairs of actions $a,a'\in [A]$.
%
%
%
%
%
Then the posterior average mean estimator $\wmu\in\R^A$ over all tasks is given by the following lemma. The proof is given in \Cref{app:proof-lemma-2}.
\begin{customlemma}{1}
\label{lemma:bayes-reg-1}
    Let $\bH_n$ be the action matrix, $\bY_n$ be the reward vector and $\bS_A$ be the estimated reward covariance matrix. Then the posterior prediction of the average mean reward vector $\wmu$ over all tasks is given by
    \begin{align}
        \wmu = \sigma^2_{\btheta}\bS_A\bH_n^\top\left(\sigma^2_{\btheta}\bH_n(\bS_A + \bD_A)\bH_n^\top\right)^{-1}\bY_{n}. \label{eq:mu}
    \end{align}
\end{customlemma}
The $\wmu$ in \eqref{eq:mu} represents the posterior mean vector averaged on all tasks. So if some action $a\in [A]$ consistently yields high rewards in the pretraining data then $\wmu(a)$ has high value. Since the test distribution is the same as pretraining, this action on average will yield a high reward during test time.  
%
%

We hypothesize that the \pred\ is learning the reward correlation covariance matrix from the training data $\Htr$ and acting greedily on it. To test this hypothesis, we consider the greedy \bayes\ algorithm that first estimates $\bS_A$ from the pretraining data. It then uses the LMMSE estimator in \Cref{lemma:bayes-reg-1} to calculate the posterior mean vector $\wmu$, and then selects $I_{t} =\argmax_a \wmu(a)$ at each round $t$. Note that \bayes\ is a greedy algorithm that always selects the most rewarding action (exploitation) without any exploration of sub-optimal actions. 
Also the \bayes\ is an LMMSE estimator that leverages the linear reward structure and estimates the reward covariance matrix, and therefore can be interpreted as a lower bound to the regret of \pred.
The hypothesis that \bayes\ is a lower bound to \pred\ is supported by \Cref{fig:bp-tr}.  In \Cref{fig:bp-tr} the reward covariance matrix for \bayes\ is estimated from the $\Htr$ by first running the \ts\ ($\pi^w$). Observe that the \bayes\ has a lower cumulative regret than \pred\ and almost matches the regret of \pred\ towards the end of the horizon. 
Also note that \linucb\ has lower cumulative regret towards the end of horizon as it leverages the linear structure and the feature of the actions in selecting the next action.
%
%

%
%
%

    

\subsection{Empirical Study: Increasing number of Actions}
\label{sec:lim}
In this section, we discuss the performance of \pred\ when the number of actions is very high so that the weak demonstrator $\pi^w$ does not have sufficient samples for each action. However, the number of tasks $\Mpr \gg A > n$.

\textbf{Baselines:} We again implement the same baselines discussed in \Cref{sec:short-horizon}. The baselines are \pred, \predt, \dptg, \ad, \ts, and \linucb.

\textbf{Outcomes:} We first discuss the main outcomes from our experimental results of introducing more actions than the horizon (or more dimensions than actions) during data collection and evaluation:

\begin{tcolorbox}
\customfinding \pred\ (\gt) outperforms \dptg, and \ad, even when $A > n$ but $\Mpr \gg A$.
\end{tcolorbox}

\begin{figure}[!hbt]
\centering
\begin{subfigure}[b]{0.48\textwidth}
   \includegraphics[scale=0.15]{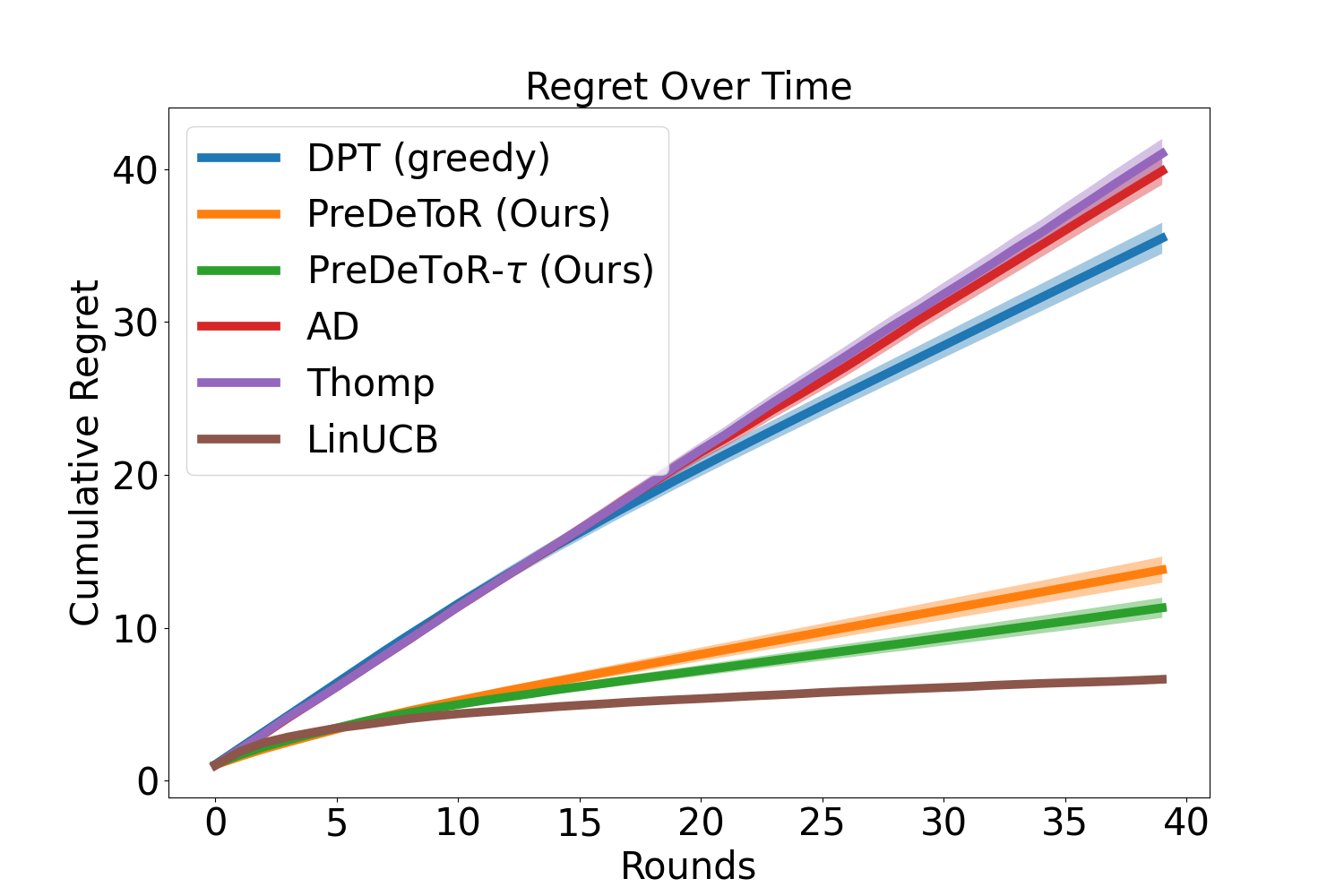}
   \caption{Linear Bandit}
   \label{fig:lim-lin}
\end{subfigure}%
\begin{subfigure}[b]{0.48\textwidth}
   \includegraphics[scale=0.15]{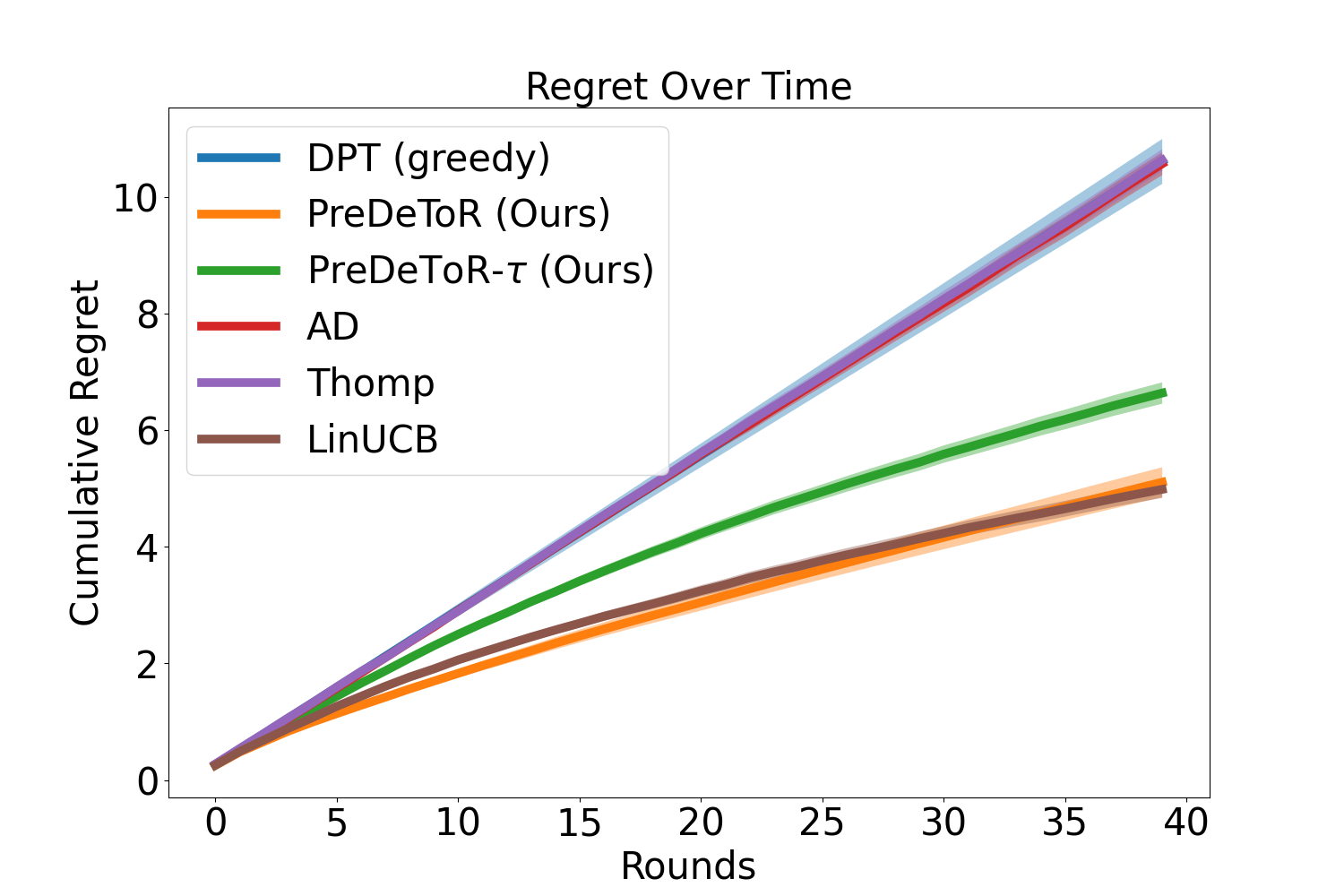}
   \caption{Non-linear Bandit}
   \label{fig:lim-nlm}
\end{subfigure}
\vspace*{-1em}
\caption{Testing the limit experiments. The horizontal axis is the number of rounds. Confidence bars show one standard error.}
\label{fig:expt-lim}
\vspace{-0.7em}
\end{figure}

\textbf{Experimental Result:} We observe these outcomes in \Cref{fig:expt-lim}. In \Cref{fig:lim-lin} we show the linear bandit setting for $\Mpr = 250000$, $\Mts = 200$, $A=100$, $n=50$ and $d=5$. 
Again, the demonstrator $\pi^w$ is the \ts\ algorithm. We observe that \pred\ (\gt) has lower cumulative regret than \dptg\ and \ad. Note that for any task $m$ the \ts\ will not be able to sample all the actions even once.
%
The weak performance of \dptg\ can be attributed to both short horizons and the inability to estimate the optimal action for such a short horizon $n < A$. The \ad\ performs similar to the demonstrator \ts\ because of its training.
%
Observe that \pred\ (\gt) has similar regret to \linucb\ and lower regret than \ts\ which also shows that \pred\ is exploiting the latent linear structure of the underlying tasks.
In \Cref{fig:lim-nlm} we show the non-linear bandit setting for horizon $n=40$, $\Mpr = 200000$, $A=60$, $d=2$, and $|\Anc|=5$. The demonstrator $\pi^w$ is the \ts\ algorithm.
%
%
Again we observe that \pred\ (\gt) has lower cumulative regret than \dptg, \ad\ and \linucb\ which fails to perform well in this non-linear setting due to its algorithmic design. 


\subsection{Empirical Study: Increasing Horizon}
\label{sec:horizon}
In this section, we discuss the performance of \pred\ with respect to an increasing horizon for each task $m\in [M]$. However, note that the number of tasks $\Mpr \geq n$.
%
%
Note that \citet{lee2023supervised} studied linear bandit setting for $n=200$. We study the setting up to a similar horizon scale.

\textbf{Baselines:} We again implement the same baselines discussed in \Cref{sec:short-horizon}. The baselines are \pred, \predt, \dptg, \ad, \ts, and \linucb.

\textbf{Outcomes:} We first discuss the main outcomes of our experimental results for increasing the horizon:

\begin{tcolorbox}
\customfinding \pred\ (\gt) outperforms \dptg, and \ad\ with increasing horizon.
\end{tcolorbox}

\begin{figure}[!hbt]
\centering
\begin{subfigure}[b]{0.32\textwidth}
   \includegraphics[scale=0.13]{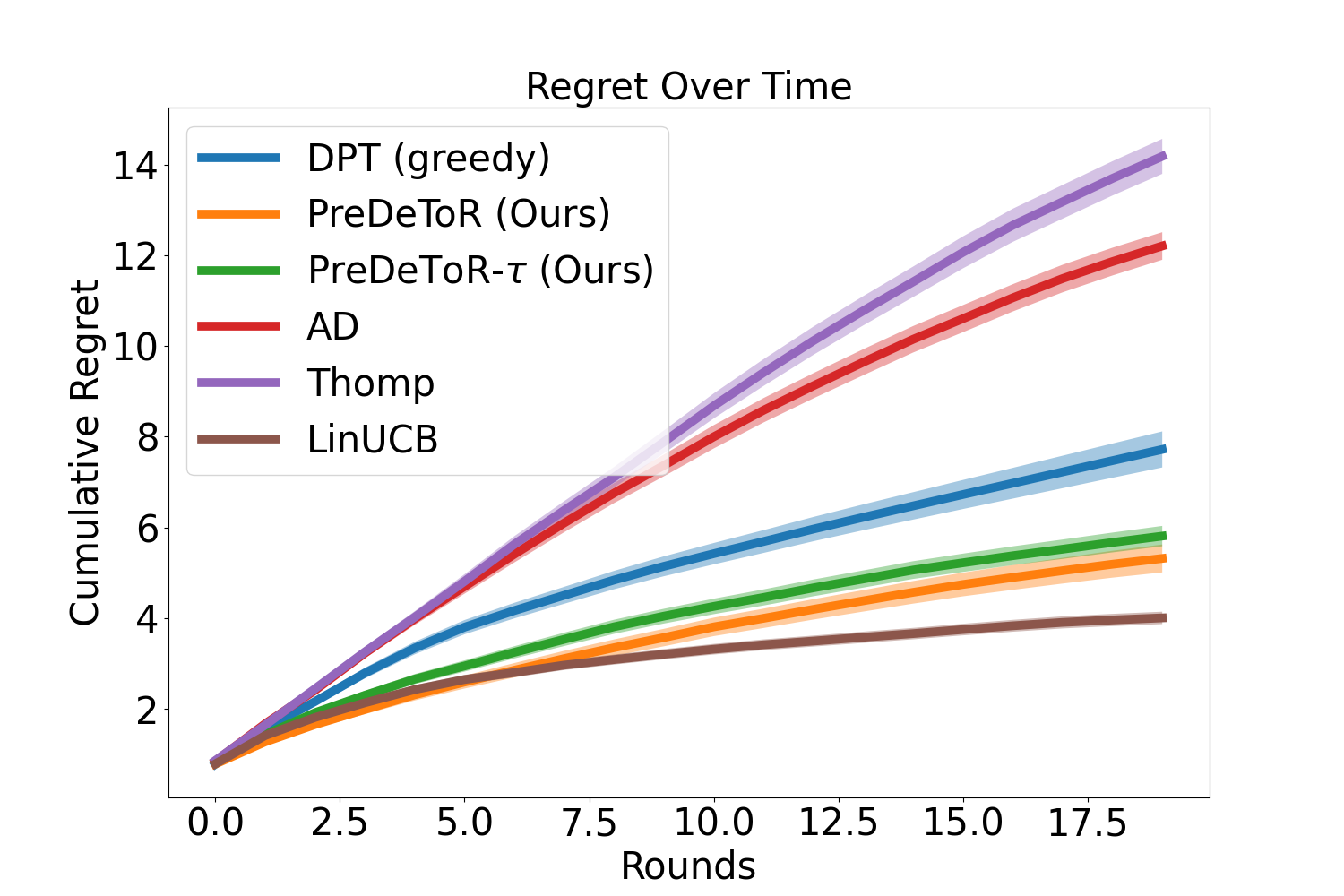}
   \caption{Horizon $20$}
   \label{fig:hor-20}
\end{subfigure}%
\begin{subfigure}[b]{0.32\textwidth}
   \includegraphics[scale=0.13]{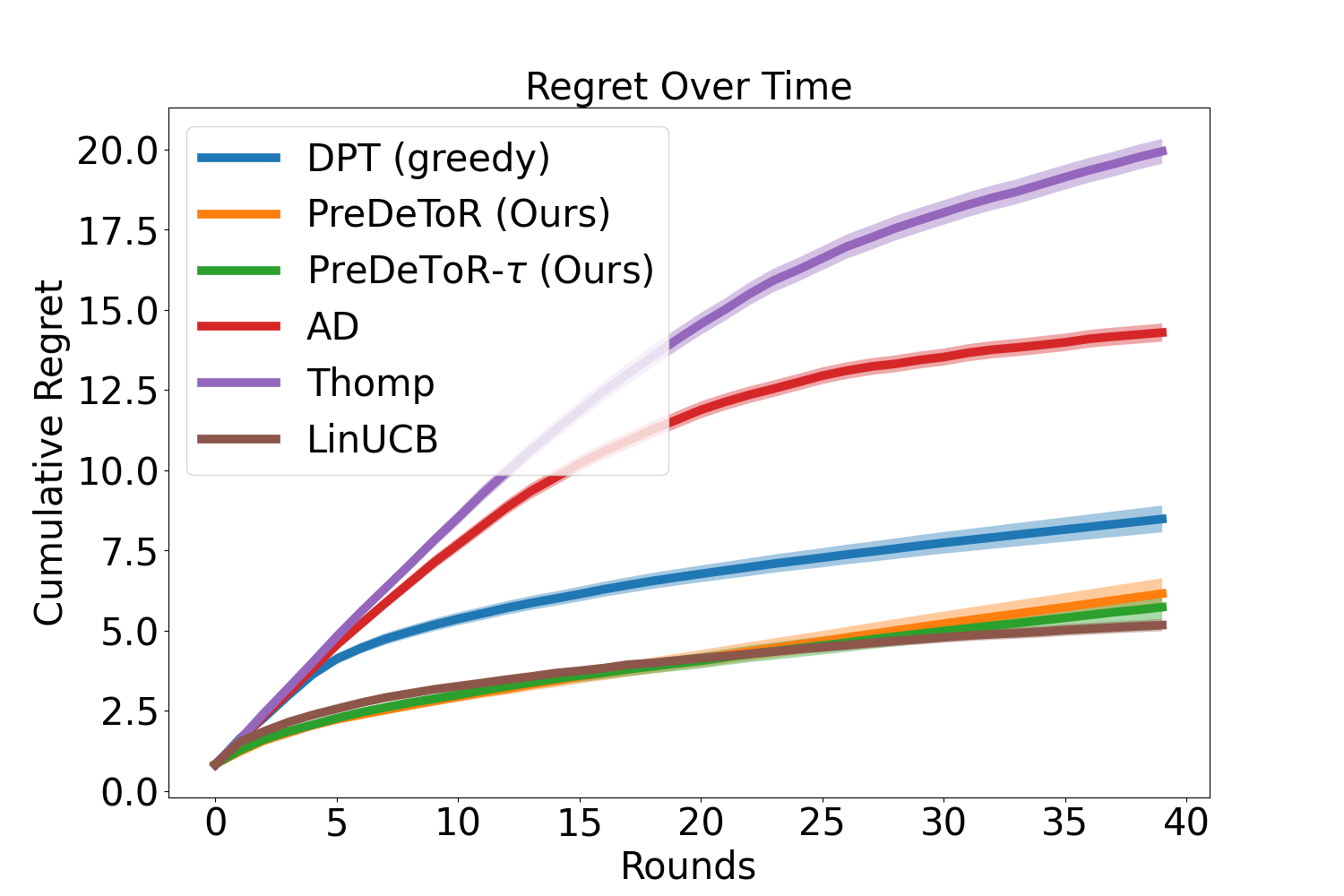}
   \caption{Horizon $40$}
   \label{fig:hor-40}
\end{subfigure}%
\begin{subfigure}[b]{0.32\textwidth}
   \includegraphics[scale=0.13]{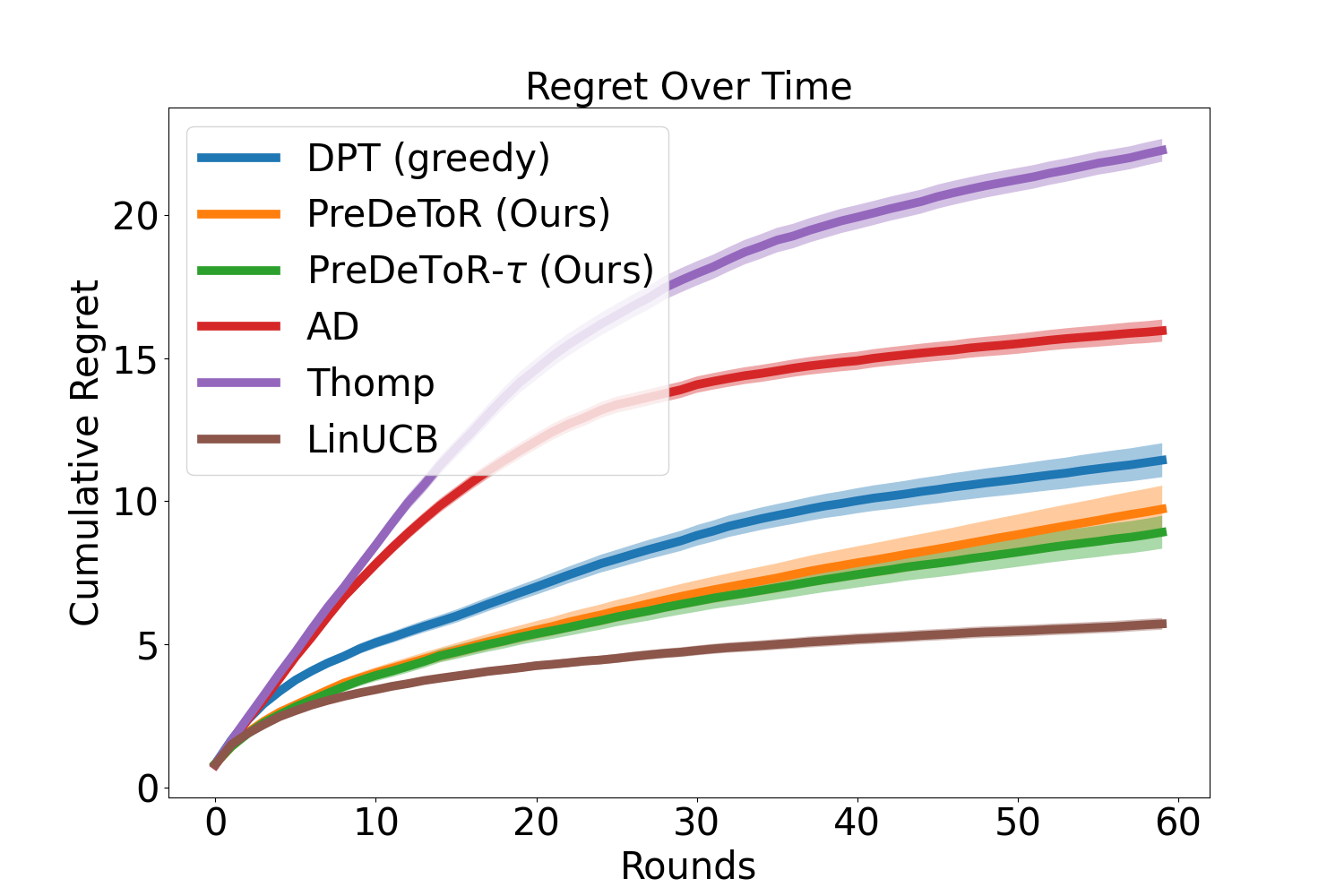}
   \caption{Horizon $60$}
   \label{fig:hor-60}
\end{subfigure}

\begin{subfigure}[b]{0.32\textwidth}
   \includegraphics[scale=0.13]{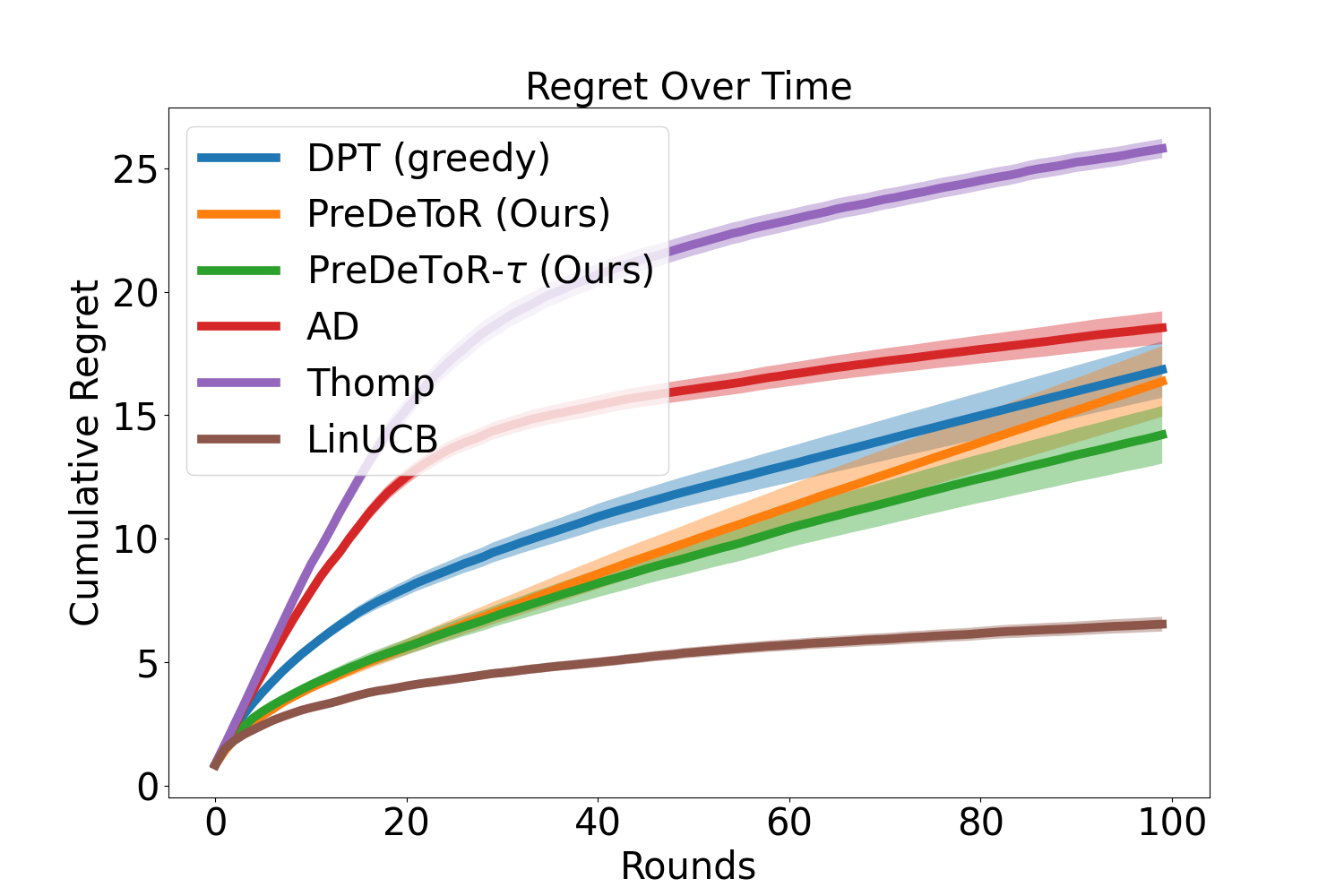}
   \caption{Horizon $100$}
   \label{fig:hor-100}
\end{subfigure}%
\begin{subfigure}[b]{0.32\textwidth}
   \includegraphics[scale=0.13]{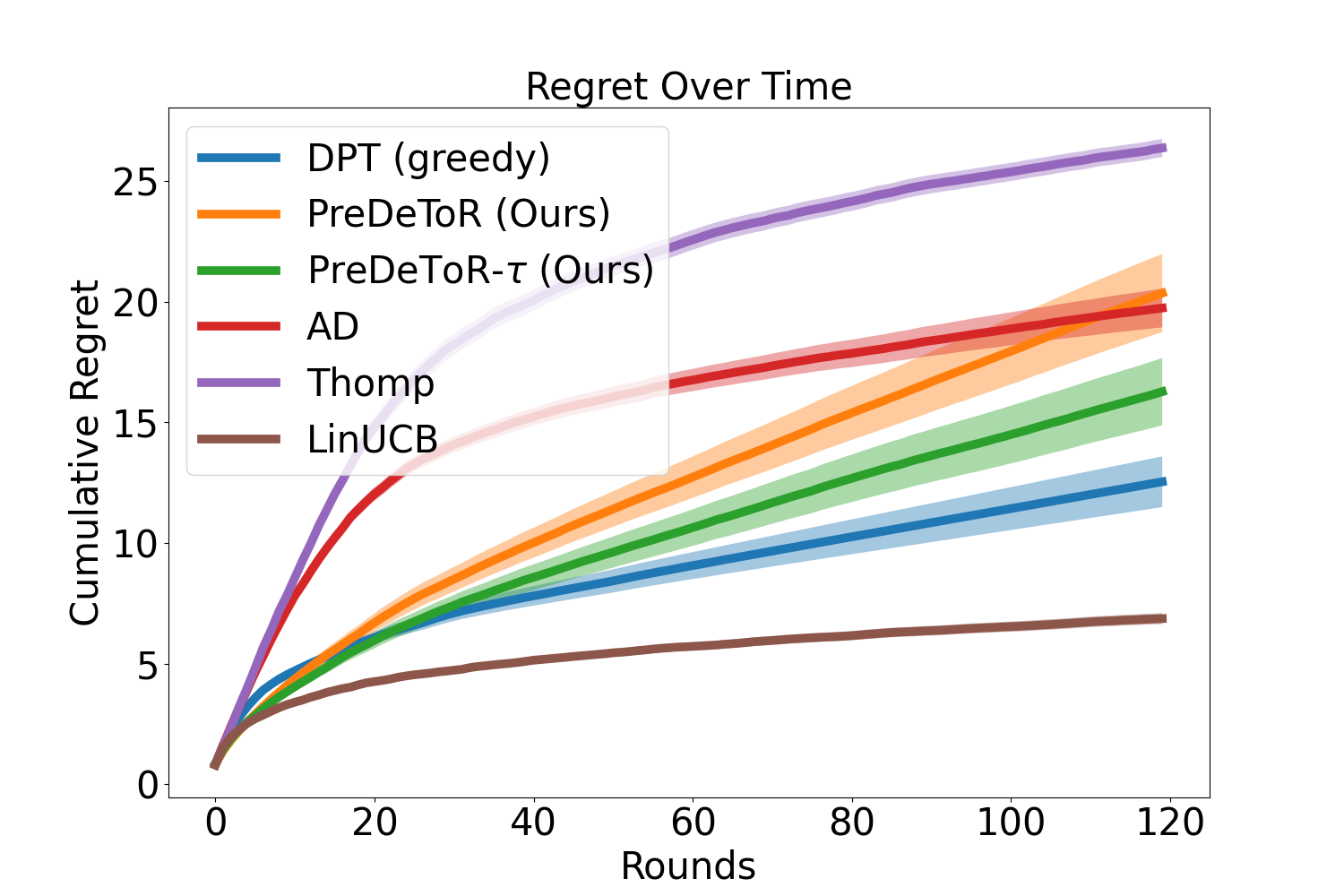}
   \caption{Horizon $120$}
   \label{fig:hor-120}
\end{subfigure}%
\begin{subfigure}[b]{0.32\textwidth}
   \includegraphics[scale=0.13]{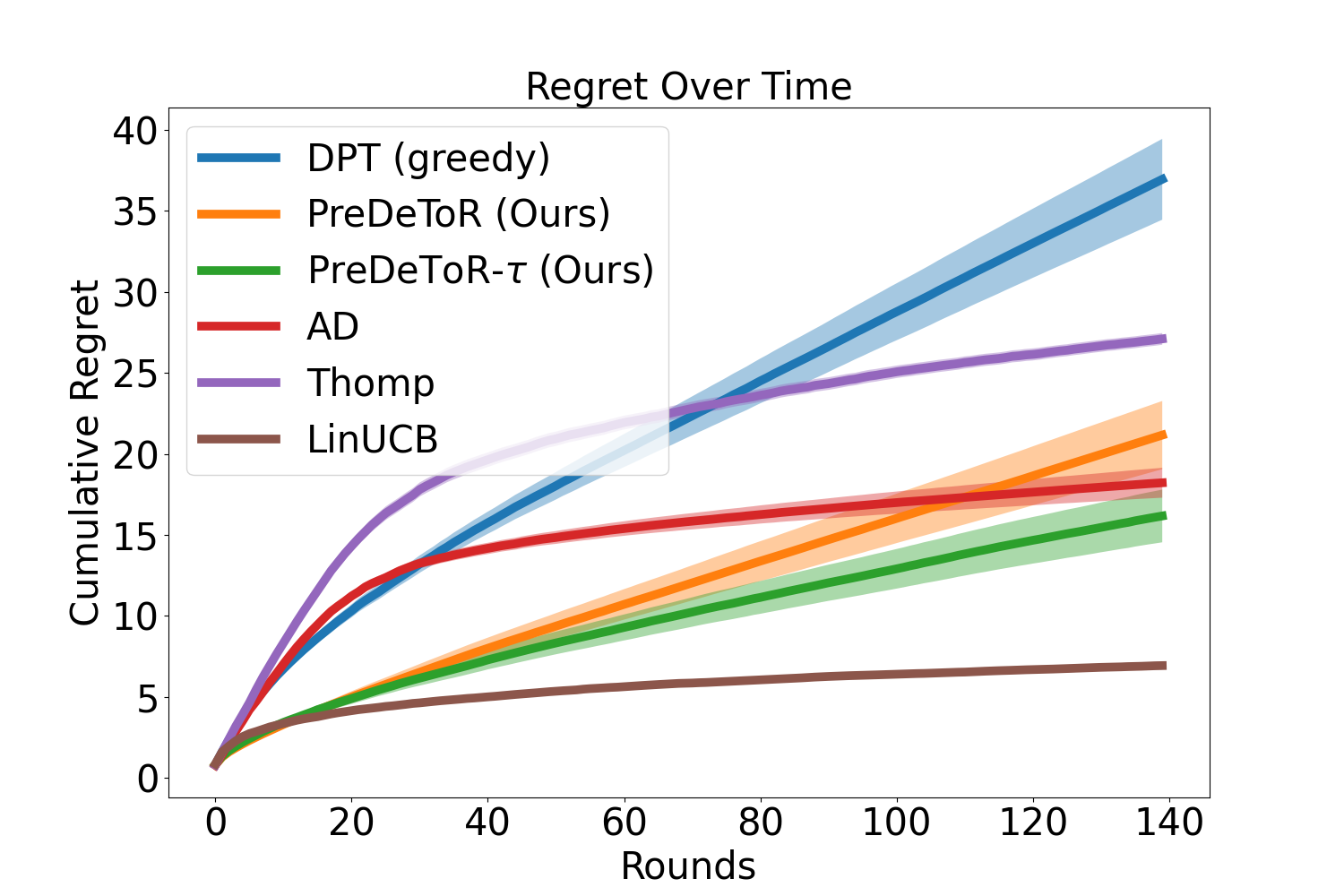}
   \caption{Horizon $140$}
   \label{fig:hor-140}
\end{subfigure}

\begin{subfigure}[b]{0.48\textwidth}
   \includegraphics[scale=0.13]{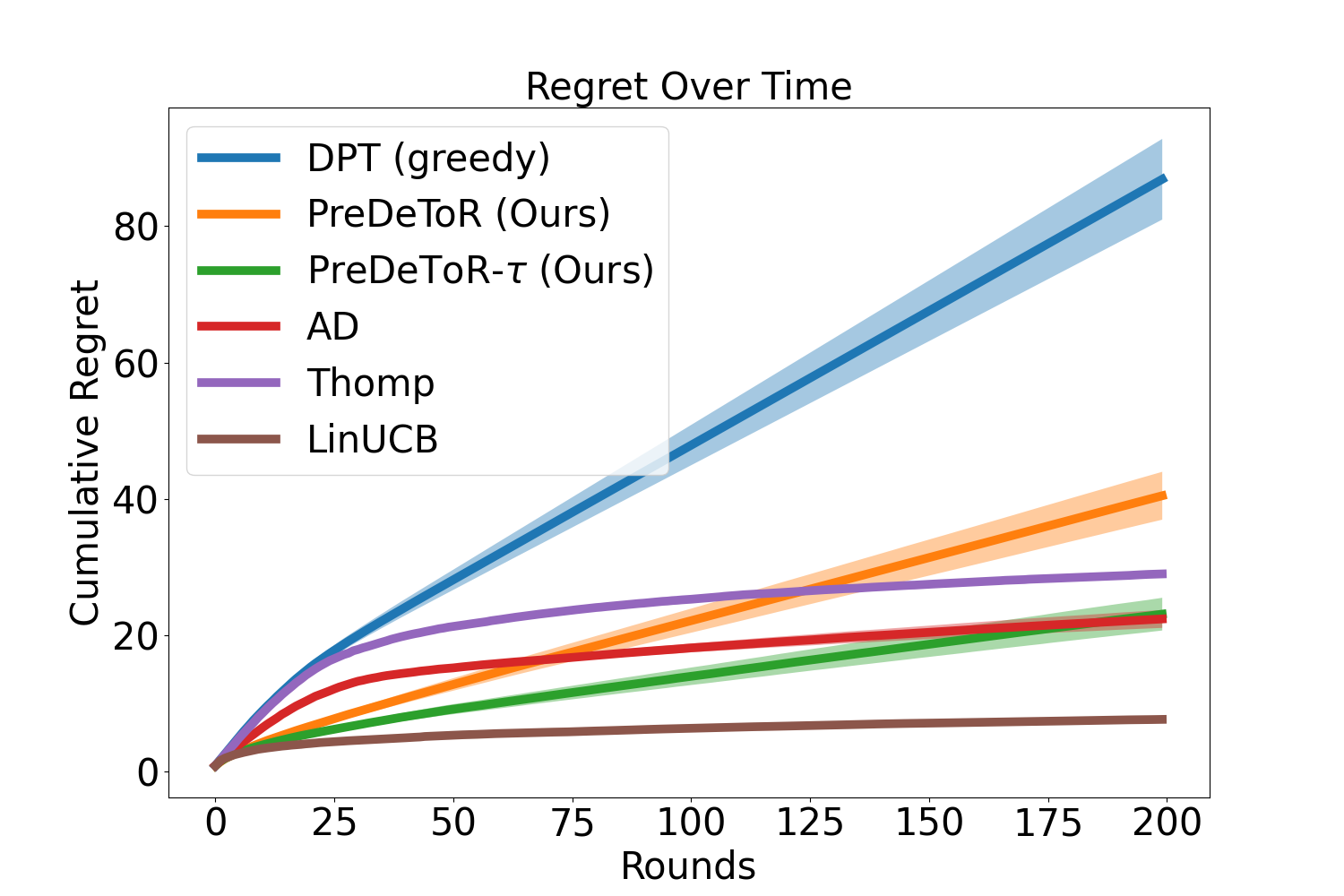}
   \caption{Horizon $200$}
   \label{fig:hor-200}
\end{subfigure}%
\begin{subfigure}[b]{0.48\textwidth}
   \includegraphics[scale=0.13]{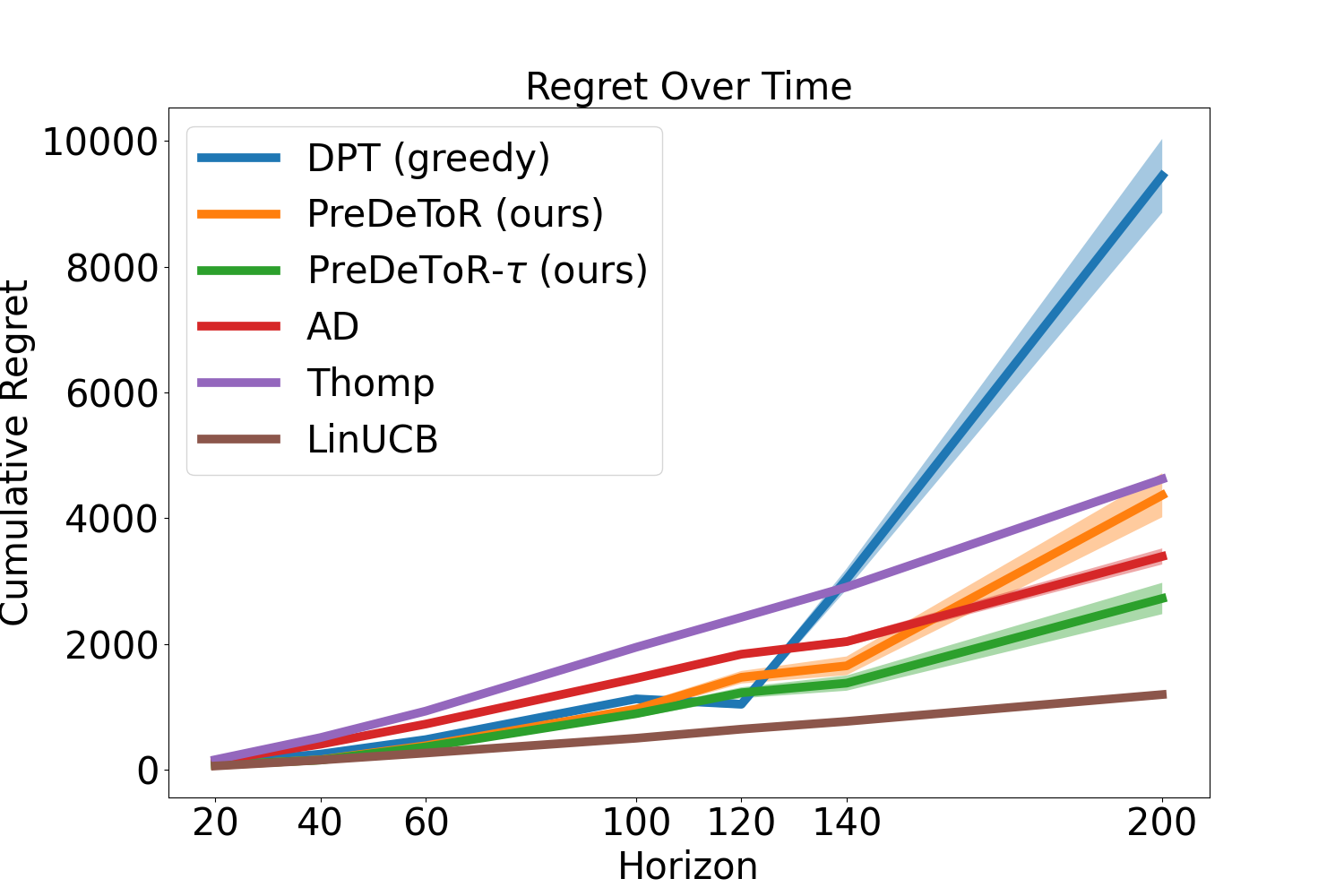}
   \caption{Increasing Horizon}
   \label{fig:hor-inc}
\end{subfigure}
\vspace*{-1em}
\caption{Experiment with increasing horizon. The y-axis shows the cumulative regret.}
\label{fig:expt-horizon}
\vspace{-0.7em}
\end{figure}

\textbf{Experimental Result:} We observe these outcomes in \Cref{fig:expt-horizon}. In \Cref{fig:expt-horizon} we show the linear bandit setting for $\Mpr =  150000$, $\Mts = 200$, $A=20$, $n=\{20, 40, 60, 100, 120, 140, 200\}$ and $d=5$. 
Again, the demonstrator $\pi^w$ is the \ts\ algorithm. We observe that \pred\ (\gt) has lower cumulative regret than \dptg, and \ad. Note that for any task $m$ for the horizon $20$ the \ts\ will be able to sample all the actions at most once.
%
%
Observe from \Cref{fig:hor-20}, \ref{fig:hor-40}, \ref{fig:hor-60}, \Cref{fig:hor-100}, \ref{fig:hor-120}, \ref{fig:hor-140} and \ref{fig:hor-200} that \pred\ (\gt) is closer to \linucb\ and outperforms \ts\ which also shows that \pred\ (\gt) is learning the latent linear structure of the underlying tasks.
In \Cref{fig:hor-inc} we plot the regret of all the baselines with respect to the increasing horizon. Again we see that \pred\ (\gt) is closer to \linucb\ and outperforms \dptg, \ad\ and \ts.
%
%
This shows that \pred\ (\gt) is able to exploit the latent structure and reward correlation across the tasks for varying horizon length.



\subsection{Empirical Study: Increasing Dimension}
\label{sec:dimension}
In this section, we discuss the performance of \pred\ with respect to an increasing dimension for each task $m\in [M]$. Again note that the number of tasks $\Mpr \gg A \geq n$.
%
%
Through this experiment, we want to evaluate the performance of \pred\ and see how it exploits the underlying reward correlation when the horizon is small as well as for increasing dimensions.

\textbf{Baselines:} We again implement the same baselines discussed in \Cref{sec:short-horizon}. The baselines are \pred, \predt\ \dptg, \ad, \ts, and \linucb.

\textbf{Outcomes:} We first discuss the main outcomes of our experimental results for increasing the horizon:

\begin{tcolorbox}
\customfinding \pred\ (\gt) outperforms \dptg, \ad\ with increasing dimension and has lower regret than \linucb\ for larger dimension.
\end{tcolorbox}

\begin{figure}[!hbt]
\centering
\begin{subfigure}[b]{0.32\textwidth}
   \includegraphics[scale=0.13]{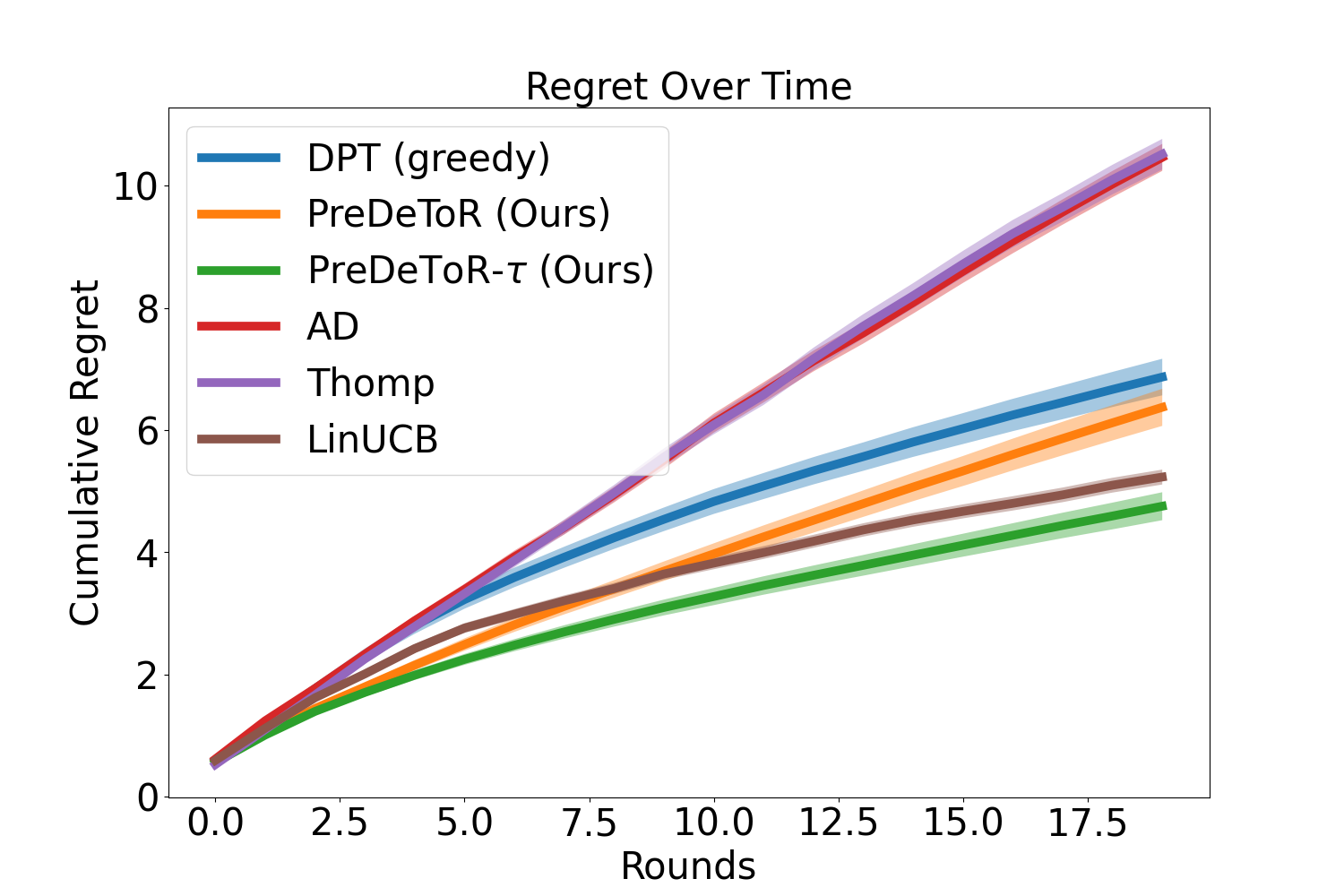}
   \caption{Dimension $10$}
   \label{fig:dim-10}
\end{subfigure}%
\begin{subfigure}[b]{0.32\textwidth}
   \includegraphics[scale=0.13]{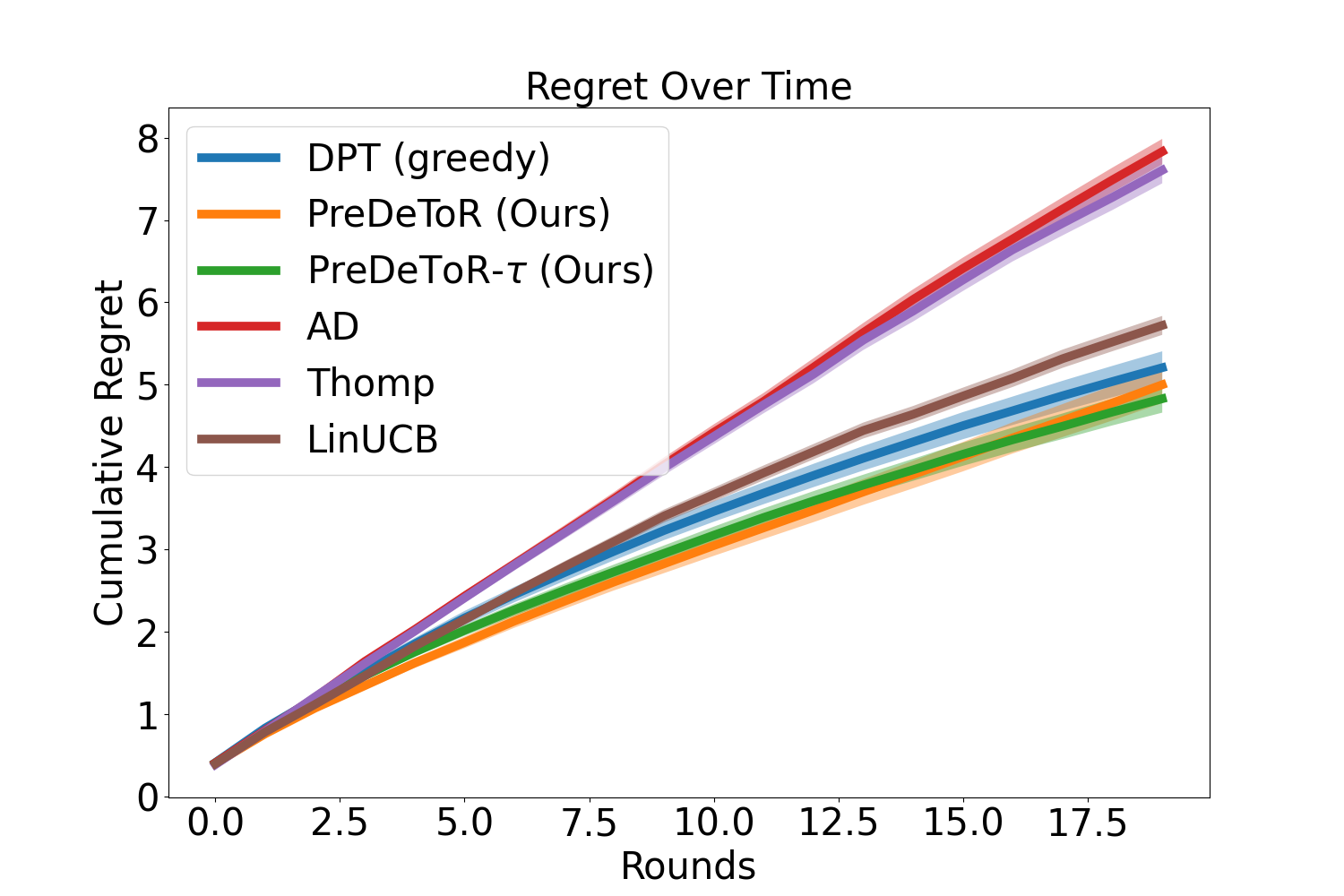}
   \caption{Dimension $20$}
   \label{fig:dim-20}
\end{subfigure}%
\begin{subfigure}[b]{0.32\textwidth}
   \includegraphics[scale=0.13]{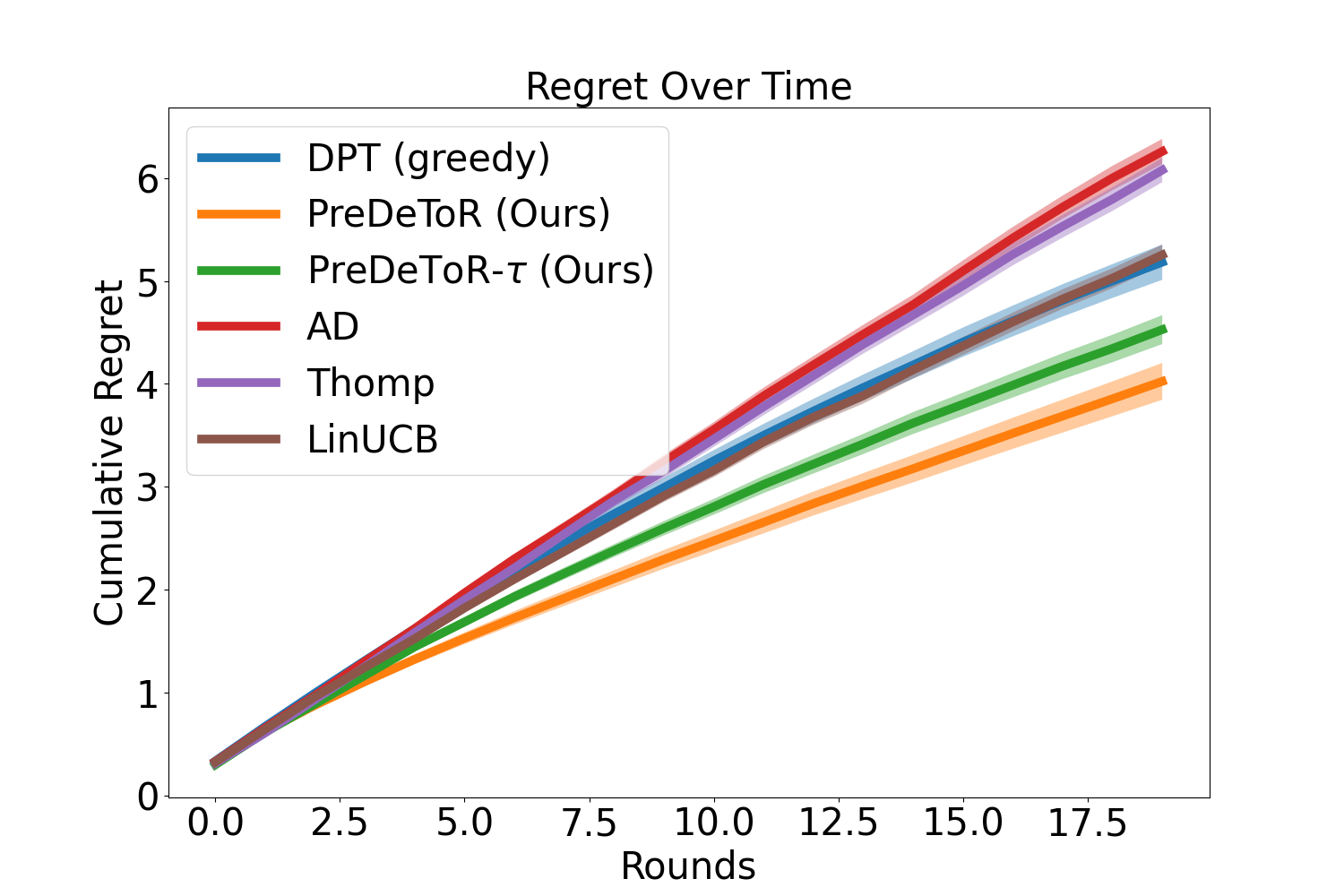}
   \caption{Dimension $30$}
   \label{fig:dim-30}
\end{subfigure}

\begin{subfigure}[b]{0.48\textwidth}
   \includegraphics[scale=0.13]{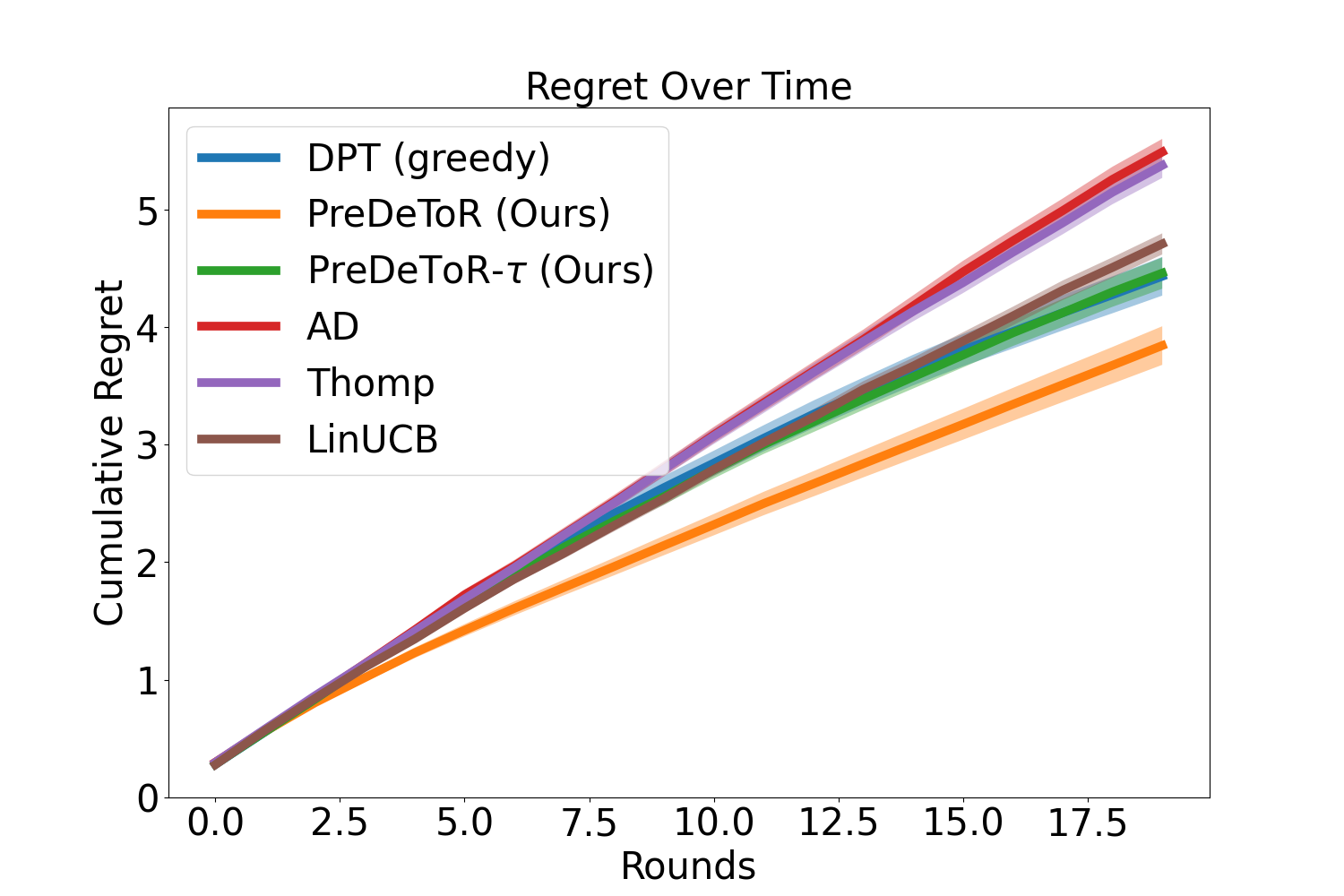}
   \caption{Dimension $40$}
   \label{fig:dim-40}
\end{subfigure}%
\begin{subfigure}[b]{0.48\textwidth}
   \includegraphics[scale=0.13]{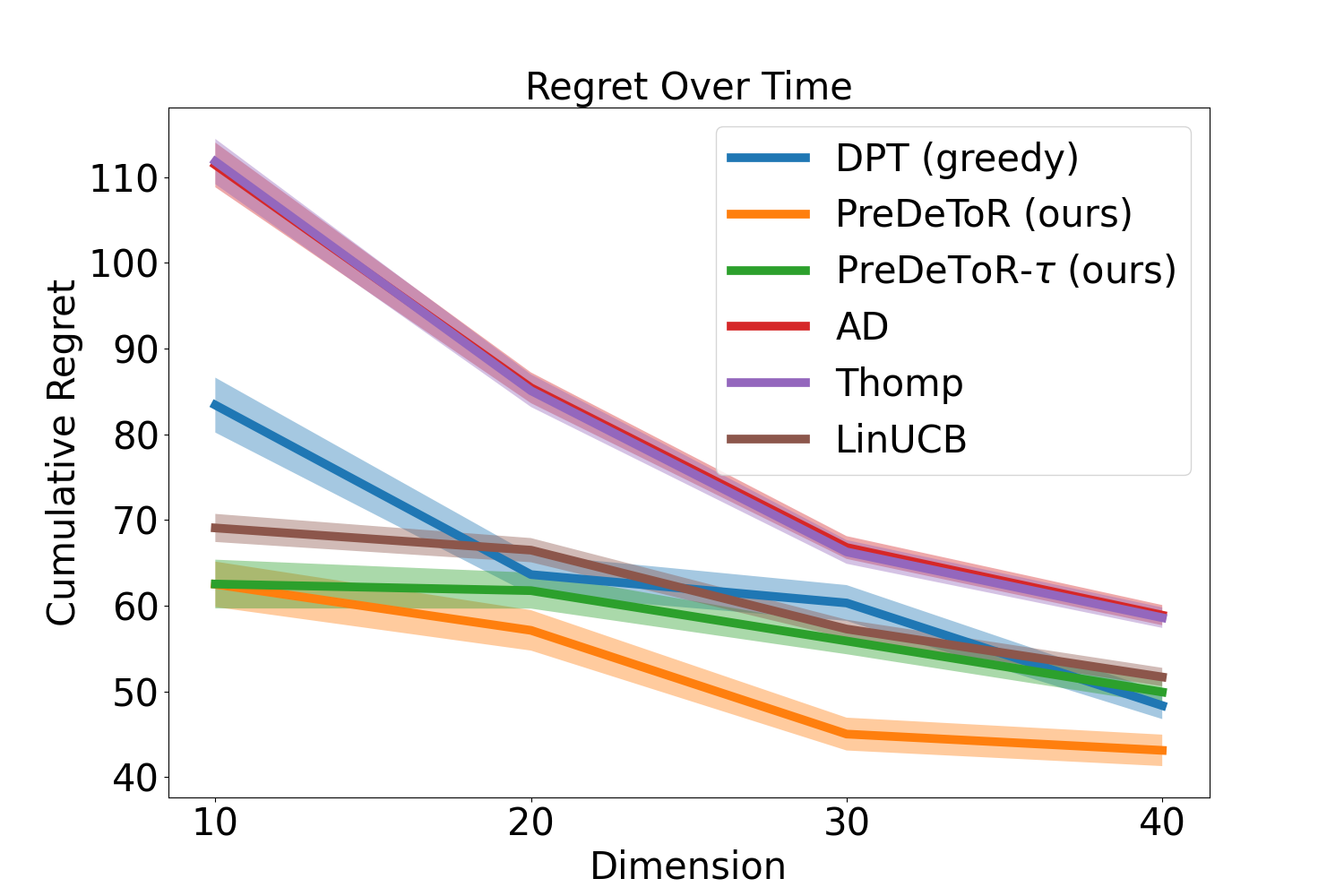}
   \caption{Increasing Dimension}
   \label{fig:dim-inc}
\end{subfigure}
\vspace*{-1em}
\caption{Experiment with increasing dimension. The y-axis shows the cumulative regret.}
\label{fig:expt-dimension}
\vspace{-0.7em}
\end{figure}

\textbf{Experimental Result:} We observe these outcomes in \Cref{fig:expt-horizon}. In \Cref{fig:expt-horizon} we show the linear bandit setting for horizon $n=20$, $\Mpr =  160000$, $\Mts = 200$, $A=20$, and $d=\{10, 20, 30, 40\}$. 
Again, the demonstrator $\pi^w$ is the \ts\ algorithm. We observe that \pred\ (\gt) has lower cumulative regret than \dptg, \ad. Note that for any task $m$ for the horizon $20$ the \ts\ will be able to sample all the actions at most once.
%
Observe from \Cref{fig:dim-10}, \ref{fig:dim-20}, \ref{fig:dim-30}, and \ref{fig:dim-40} that \pred\ (\gt) is closer to \linucb\ and has lower regret than \ts\ which also shows that \pred\ (\gt) is exploiting the latent linear structure of the underlying tasks.
In \Cref{fig:dim-inc} we plot the regret of all the baselines with respect to the increasing dimension. Again we see that \pred\ (\gt) has lower regret than \dptg, \ad\ and \ts. Observe that with  increasing dimension \pred\ is able to outperform \linucb. This shows that the \pred\ (\gt) is able to exploit reward correlation across tasks for varying dimensions.

\subsection{Empirical Study: Increasing Attention Heads}
\label{sec:heads}
In this section, we discuss the performance of \pred\ with respect to an increasing attention heads for the transformer model for the non-linear feedback model. Again note that the number of tasks $\Mpr \gg A \geq n$.
%
%
Through this experiment, we want to evaluate the performance of \pred\ to exploit the underlying reward correlation when the horizon is small and understand the representative power of the transformer by increasing the attention heads.
Note that we choose the non-linear feedback model and low data regime to leverage the representative power of the transformer.

\textbf{Baselines:} We again implement the same baselines discussed in \Cref{sec:short-horizon}. The baselines are \pred, \predt, \dptg, \ad, \ts, and \linucb.

\textbf{Outcomes:} We first discuss the main outcomes of our experimental results for increasing the horizon:




\begin{tcolorbox}
\customfinding \pred\ (\gt) outperforms \dptg, and \ad\ with increasing attention heads.
\end{tcolorbox}

\begin{figure}[!hbt]
\centering
\begin{subfigure}[b]{0.32\textwidth}
   \includegraphics[scale=0.13]{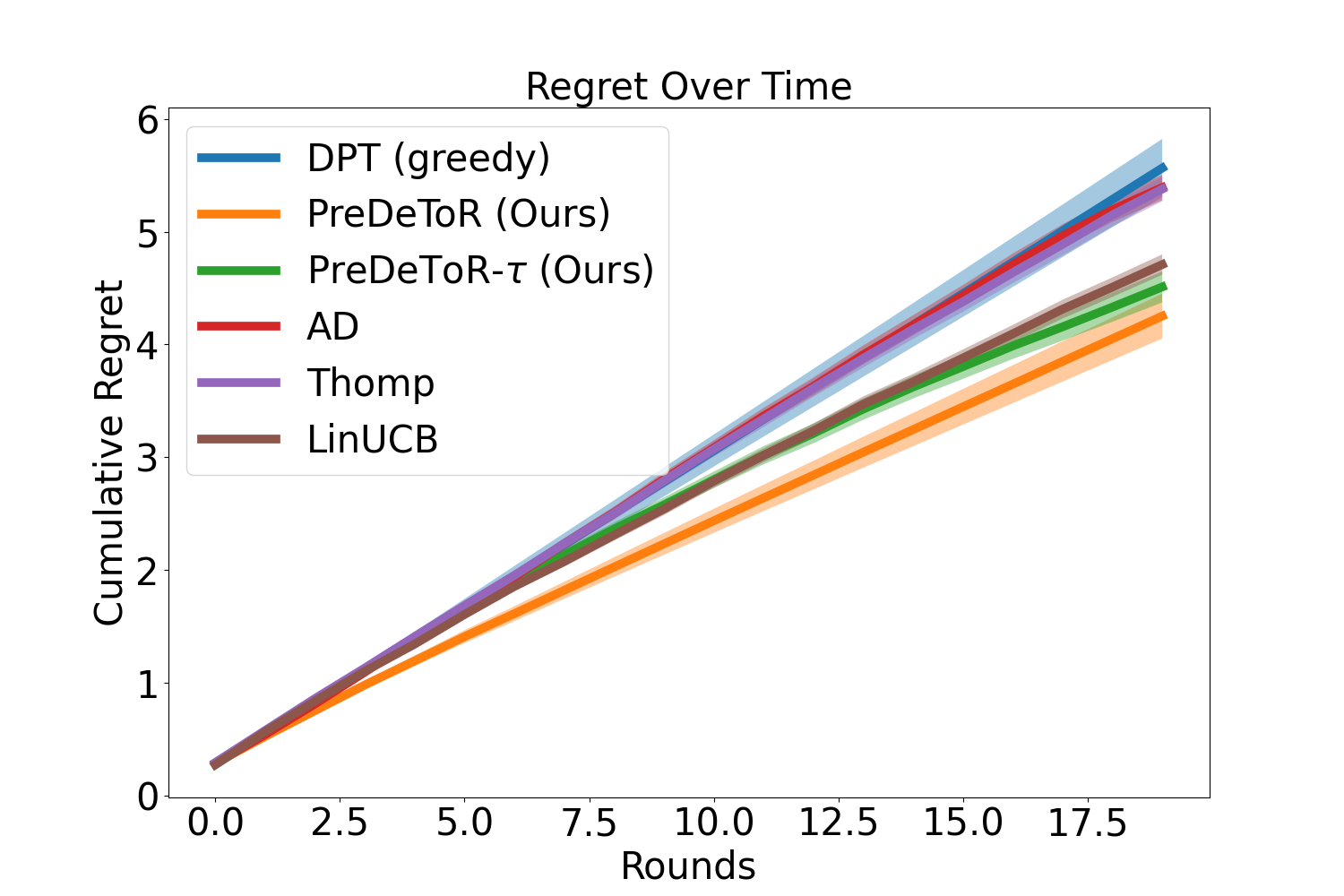}
   \caption{Attention Heads $2$}
   \label{fig:head-2}
\end{subfigure}%
\begin{subfigure}[b]{0.32\textwidth}
   \includegraphics[scale=0.13]{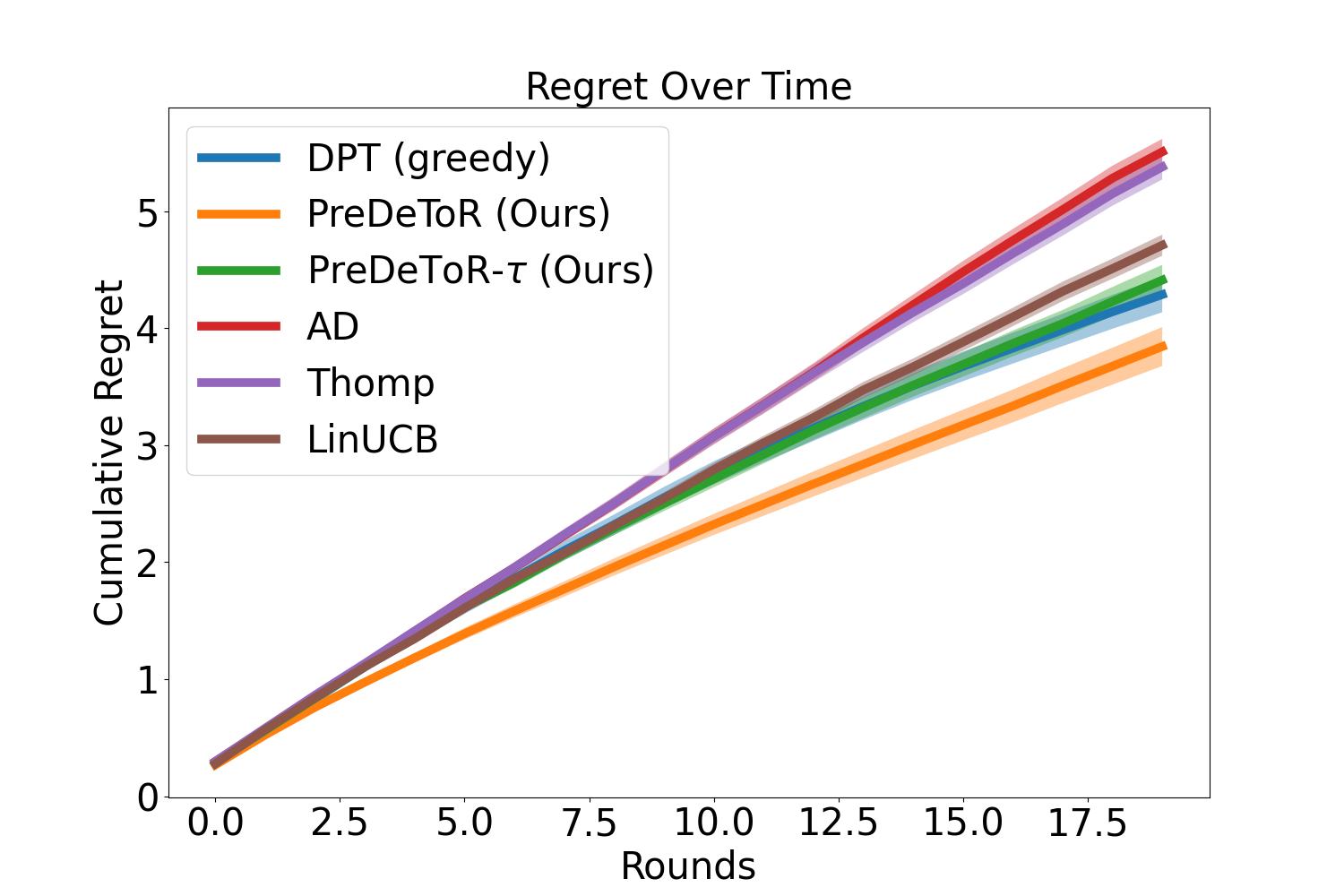}
   \caption{Attention Heads $4$}
   \label{fig:head-4}
\end{subfigure}%
\begin{subfigure}[b]{0.32\textwidth}
   \includegraphics[scale=0.13]{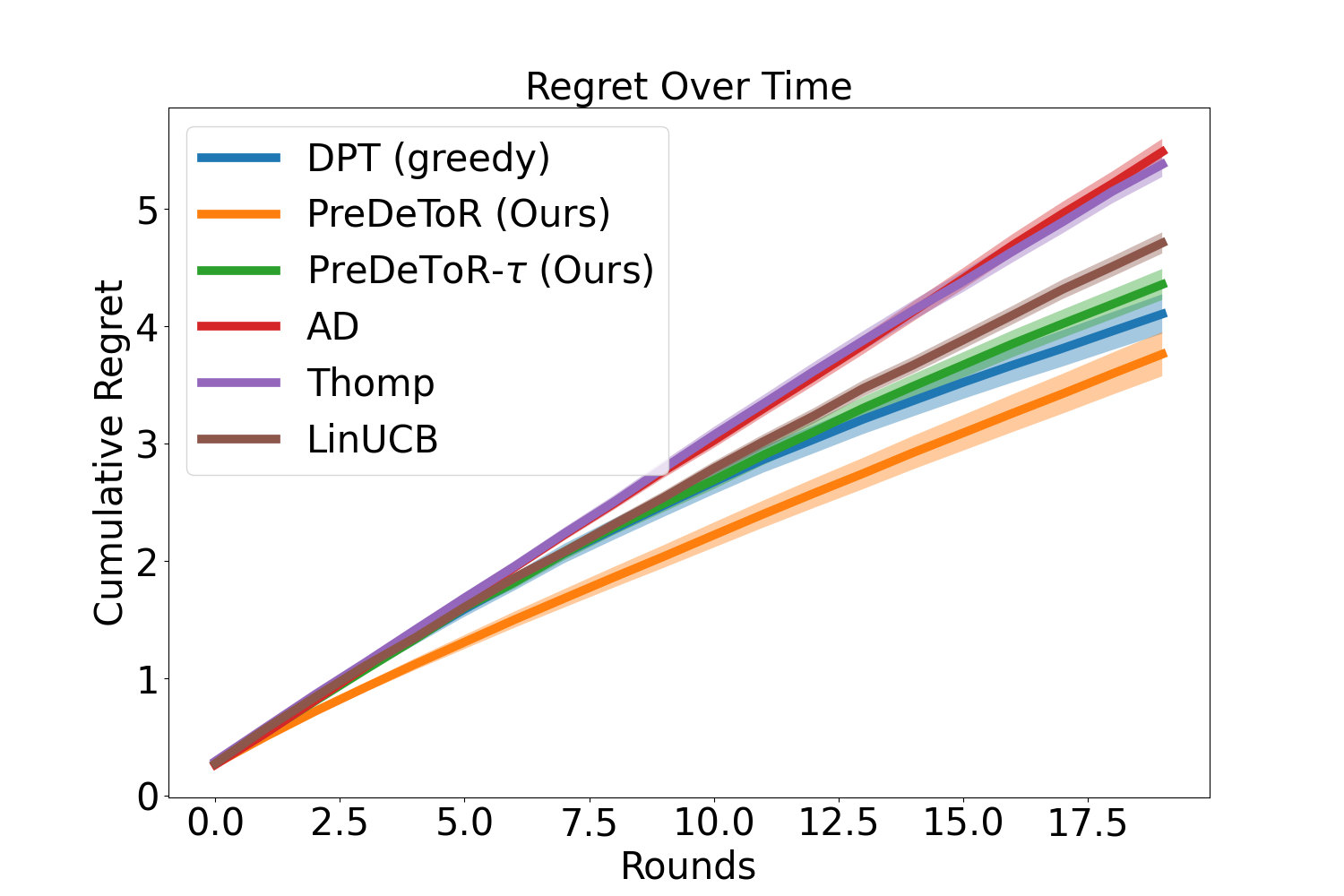}
   \caption{Attention Heads $6$}
   \label{fig:head-6}
\end{subfigure}

\begin{subfigure}[b]{0.32\textwidth}
   \includegraphics[scale=0.13]{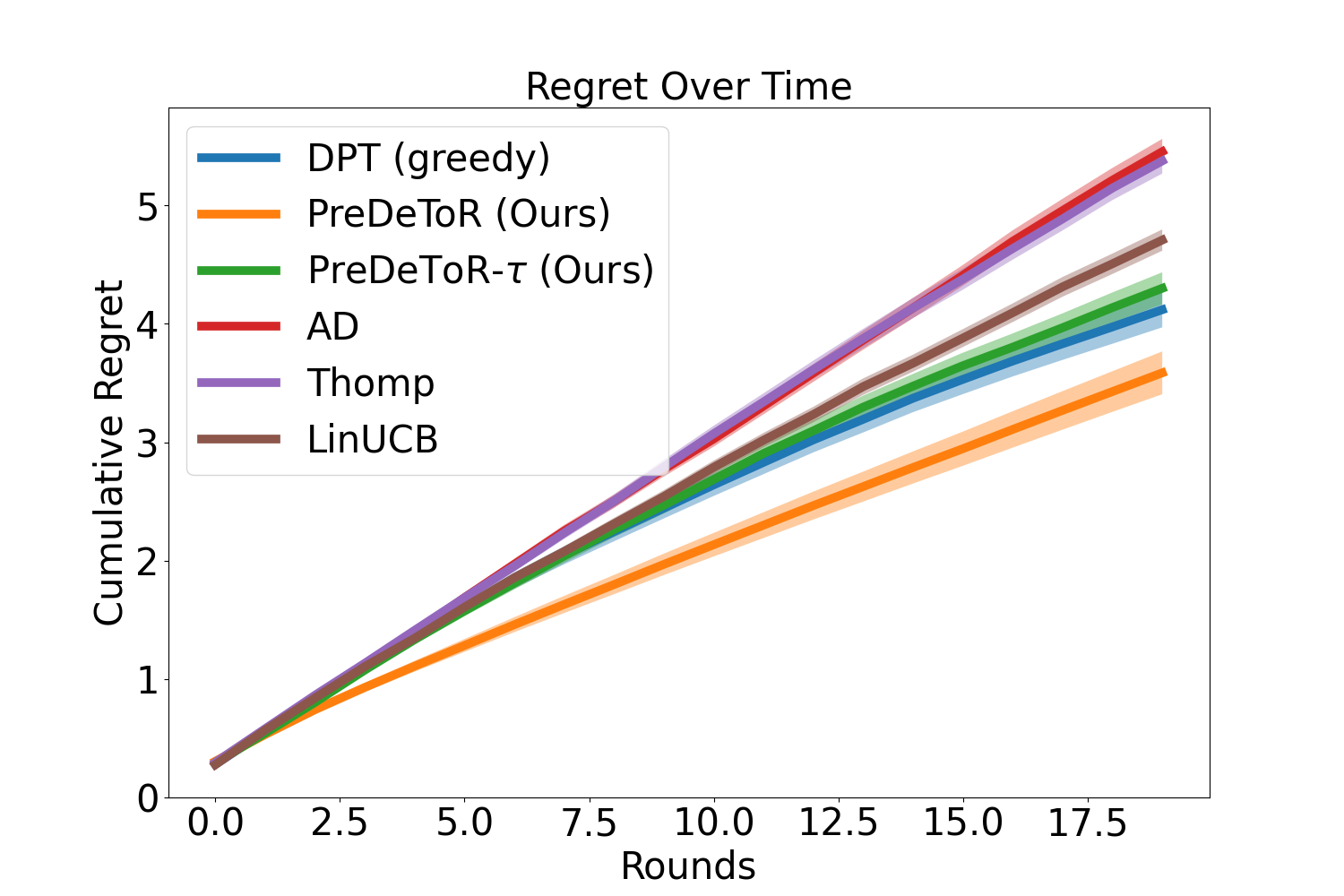}
   \caption{Attention Heads $8$}
   \label{fig:head-8}
\end{subfigure}%
\begin{subfigure}[b]{0.32\textwidth}
   \includegraphics[scale=0.13]{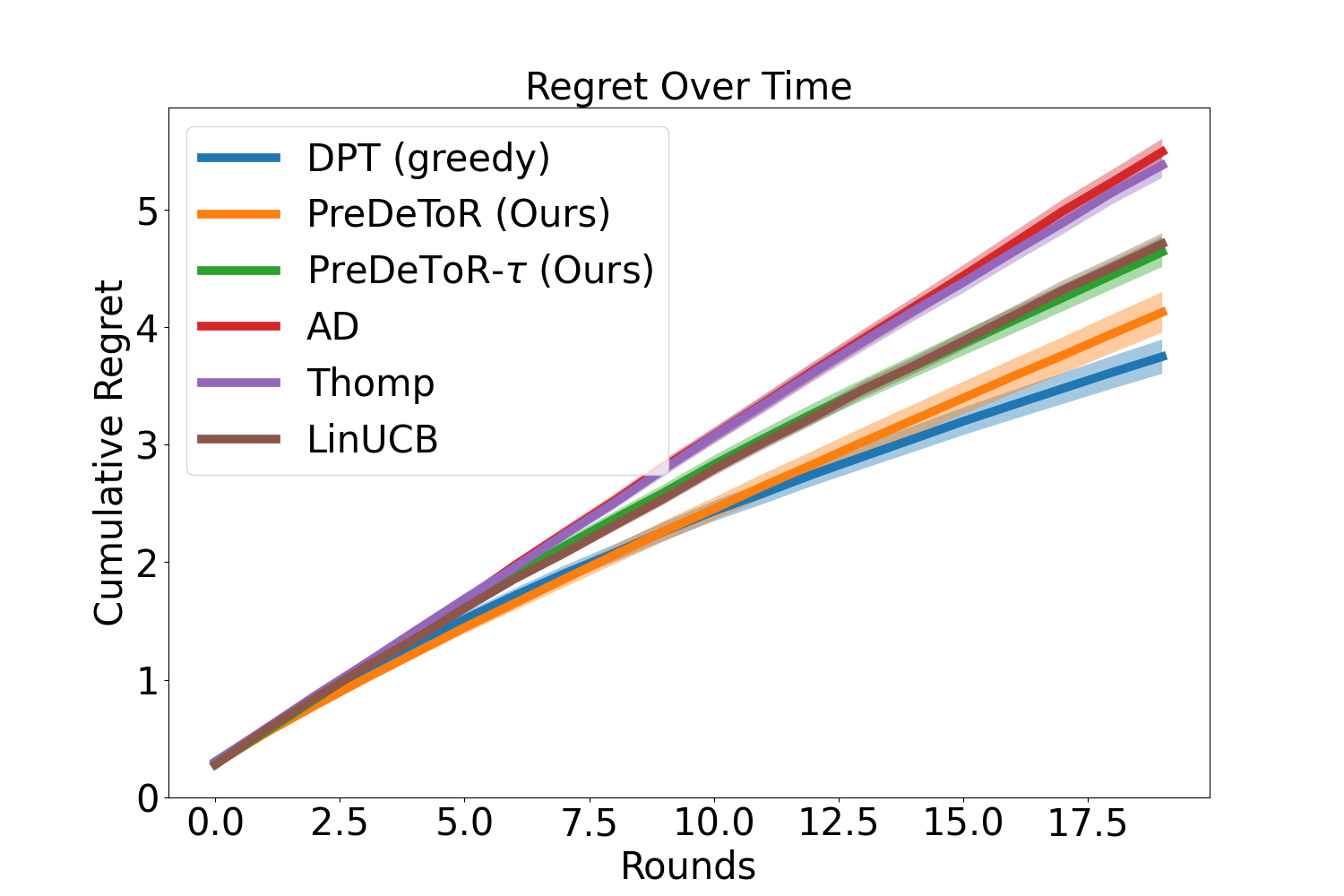}
   \caption{Attention Heads $12$}
   \label{fig:head-12}
\end{subfigure}%
\begin{subfigure}[b]{0.32\textwidth}
   \includegraphics[scale=0.13]{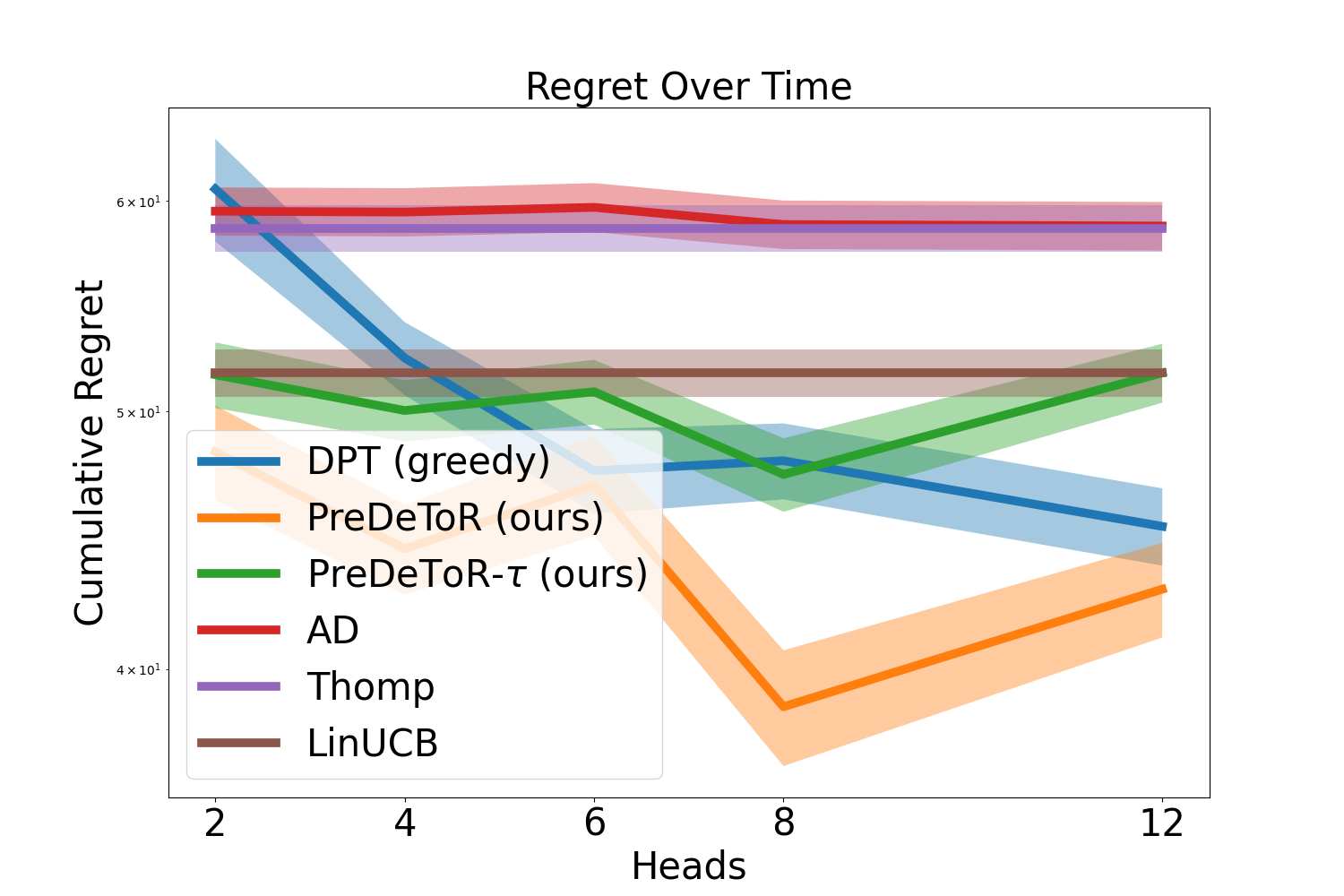}
   \caption{Increasing Attention Heads}
   \label{fig:head-inc}
\end{subfigure}
\vspace*{-1em}
\caption{Experiment with increasing attention heads. The y-axis shows the cumulative regret.}
\label{fig:expt-head}
\vspace{-0.7em}
\end{figure}

\textbf{Experimental Result:} We observe these outcomes in \Cref{fig:expt-head}. In \Cref{fig:expt-head} we show the non-linear bandit setting for horizon $n=20$, $\Mpr =  160000$, $\Mts = 200$, $A=20$, $\mathrm{heads}=\{2, 4, 6, 8\}$ and $d=5$. 
Again, the demonstrator $\pi^w$ is the \ts\ algorithm. We observe that \pred\ (\gt) has lower cumulative regret than \dptg, \ad. Note that for any task $m$ for the horizon $20$ the \ts\ will be able to sample all the actions atmost once.
%
%
Observe from \Cref{fig:head-2}, \ref{fig:head-4}, \ref{fig:head-6}, and \ref{fig:head-8} that \pred\ (\gt) has lower regret than \ad, \ts\ and \linucb\ which also shows that \pred\ (\gt) is exploiting the latent linear structure of the underlying tasks for the non-linear setting.
%
%
In \Cref{fig:head-inc} we plot the regret of all the baselines with respect to the increasing attention heads. Again we see that \pred\ (\gt) regret decreases as we increase the attention heads. 


\subsection{Empirical Study: Increasing Number of Tasks}
\label{sec:envs}
In this section, we discuss the performance of \pred\ with respect to the increasing number of tasks for the linear bandit setting. Again note that the number of tasks $\Mpr \gg A \geq n$.
Through this experiment, we want to evaluate the performance of \pred\ to exploit the underlying reward correlation when the horizon is small and the number of tasks is changing.
Finally, recall that when the horizon is small the weak demonstrator $\pi^w$ does not have sufficient samples for each action. This leads to a poor approximation of the greedy action.

\textbf{Baselines:} We again implement the same baselines discussed in \Cref{sec:short-horizon}. The baselines are \pred, \predt, \dptg, \ad, \ts, and \linucb.

\textbf{Outcomes:} We first discuss the main outcomes of our experimental results for increasing the horizon:

\begin{tcolorbox}
\customfinding \pred\ (\gt) fails to exploit the underlying latent structure and reward correlation from in-context data when the number of tasks is small. 
\end{tcolorbox}

\begin{figure}[!hbt]
\centering
\begin{subfigure}[b]{0.32\textwidth}
   \includegraphics[scale=0.13]{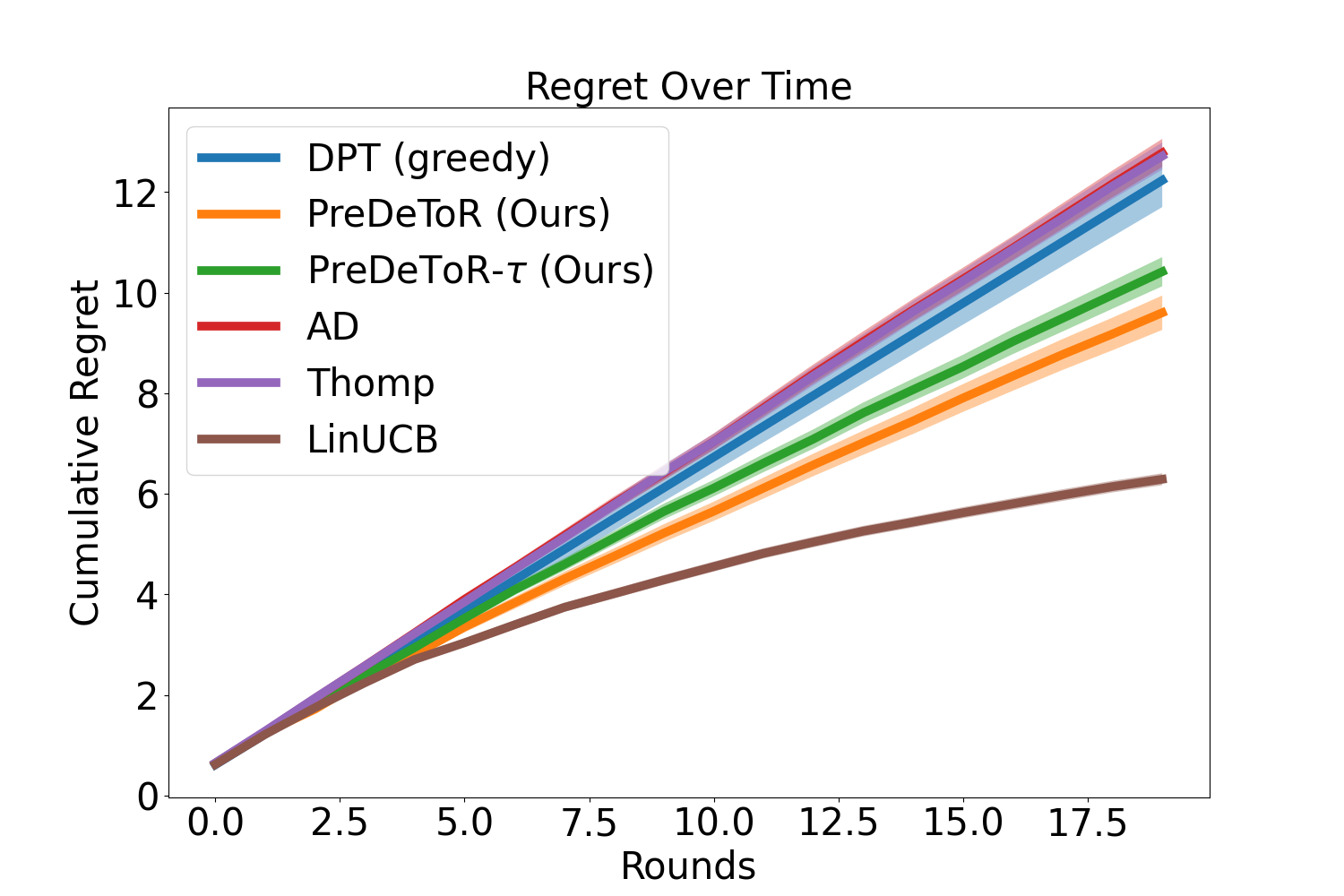}
   \caption{Tasks $\Mtr=5000$}
   \label{fig:env-5}
\end{subfigure}%
\begin{subfigure}[b]{0.32\textwidth}
   \includegraphics[scale=0.13]{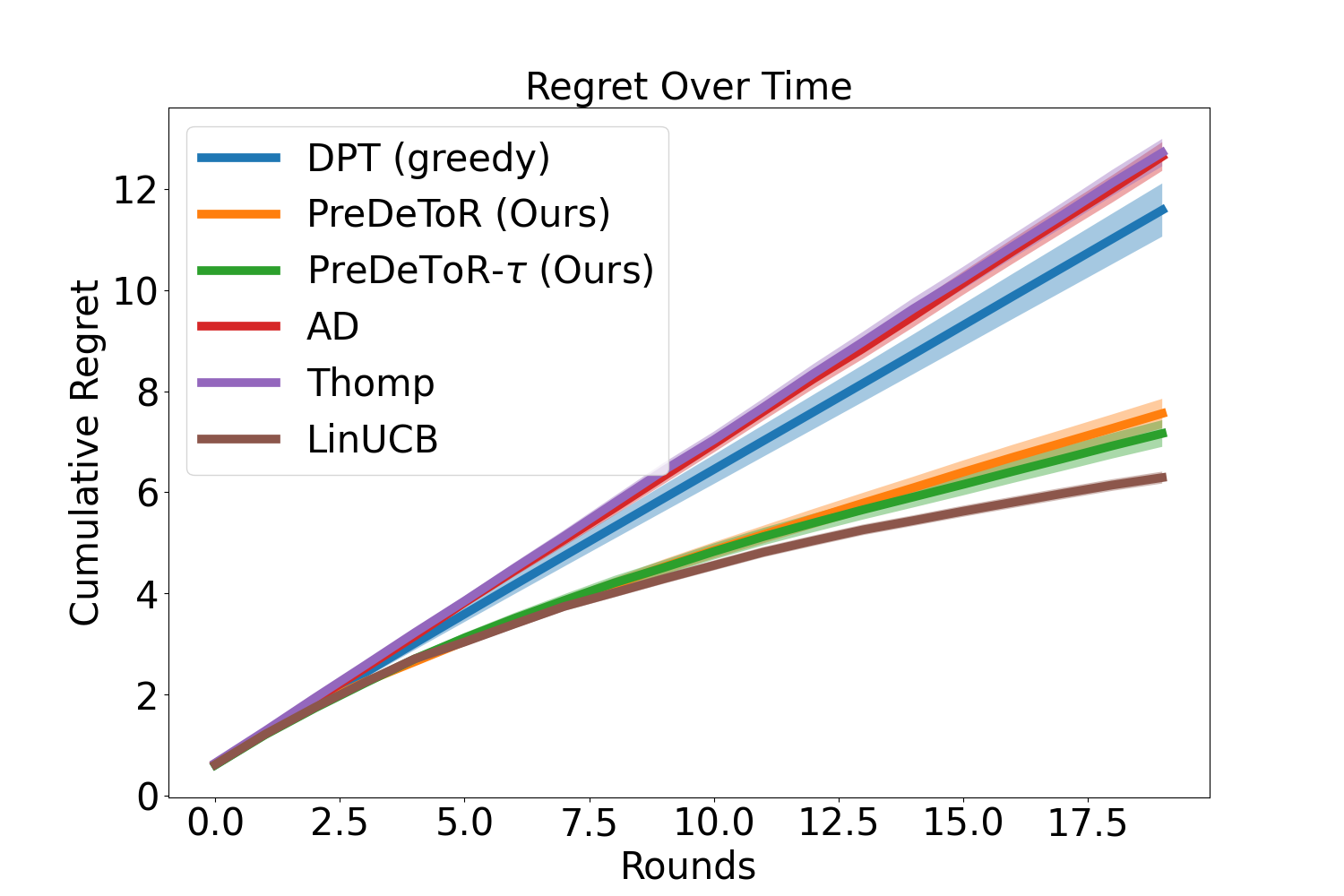}
   \caption{Tasks $\Mtr=10000$}
   \label{fig:env-10}
\end{subfigure}%
\begin{subfigure}[b]{0.32\textwidth}
   \includegraphics[scale=0.13]{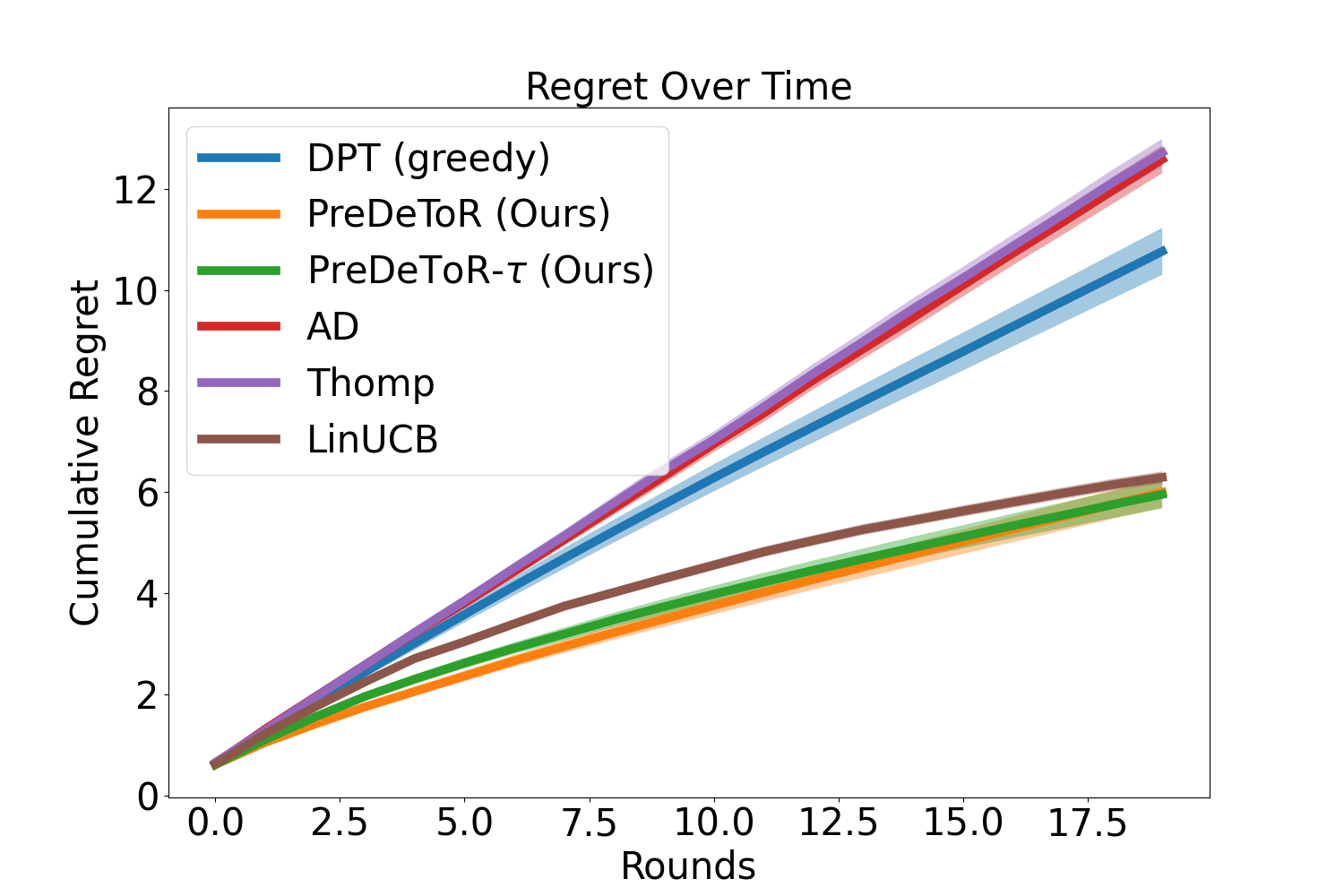}
   \caption{Tasks $\Mtr=50000$}
   \label{fig:env-50}
\end{subfigure}

\begin{subfigure}[b]{0.32\textwidth}
   \includegraphics[scale=0.13]{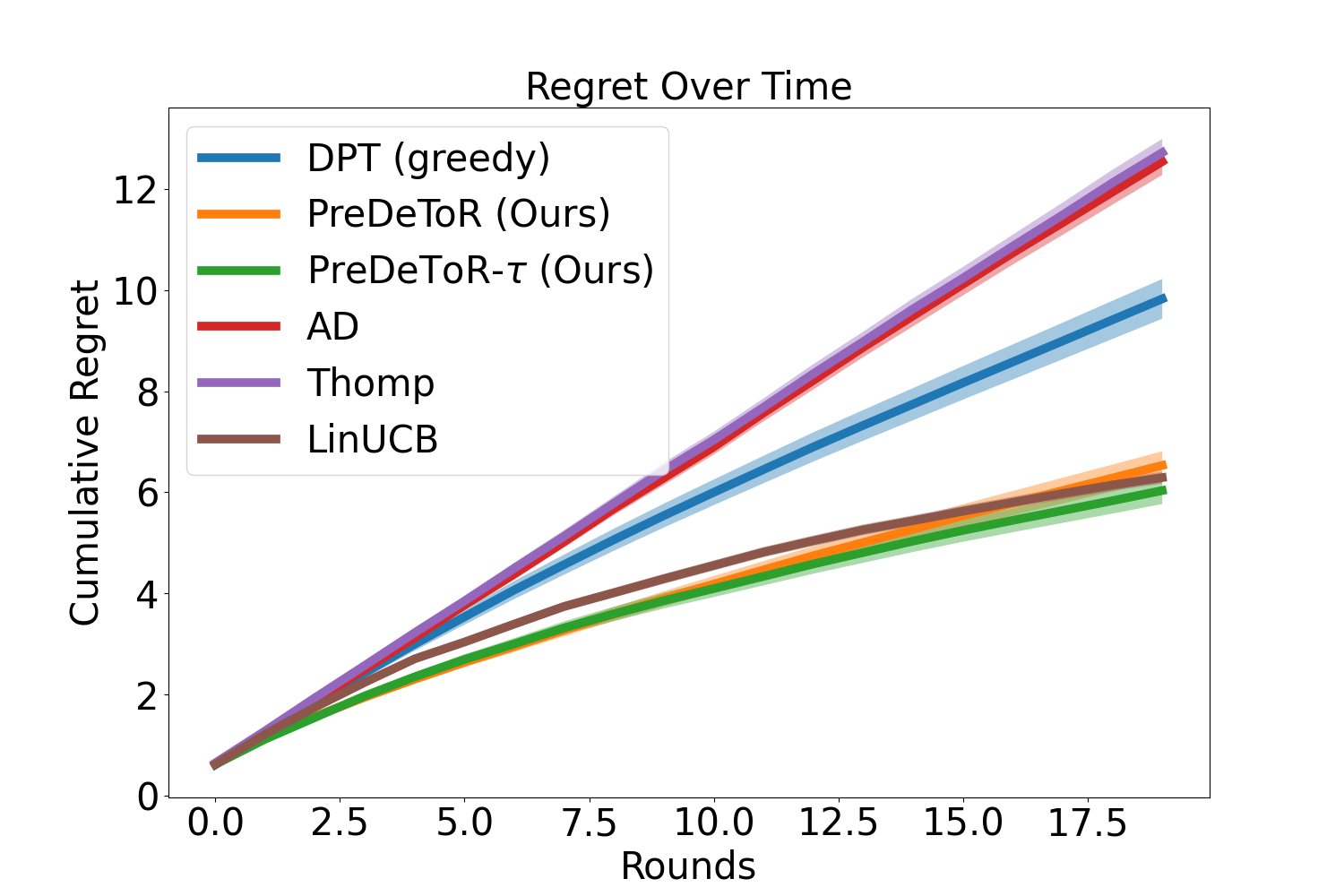}
   \caption{Tasks $\Mtr=100000$}
   \label{fig:env-100}
\end{subfigure}%
\begin{subfigure}[b]{0.32\textwidth}
   \includegraphics[scale=0.13]{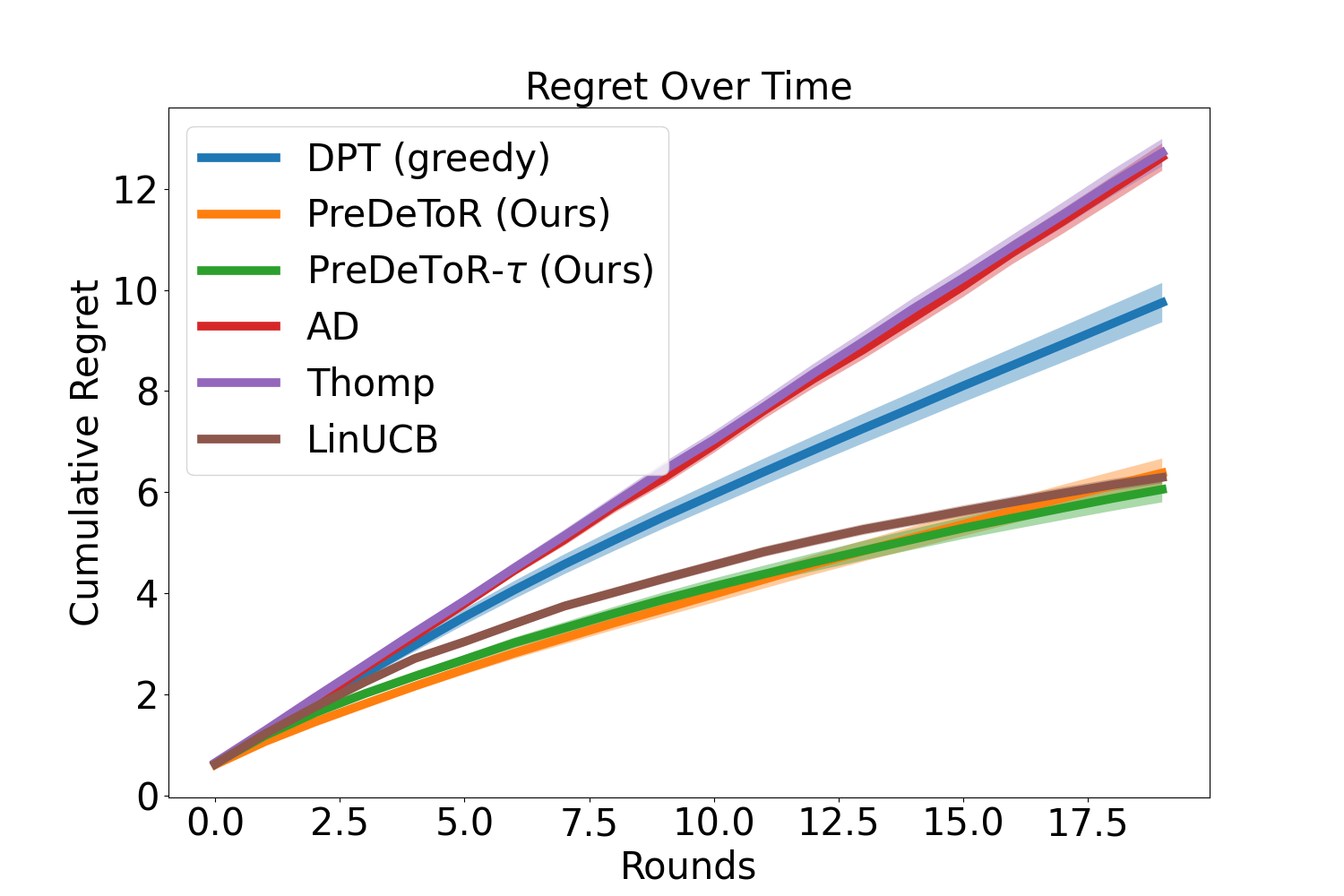}
   \caption{Tasks $\Mtr=150000$}
   \label{fig:env-150}
\end{subfigure}%
\begin{subfigure}[b]{0.32\textwidth}
   \includegraphics[scale=0.13]{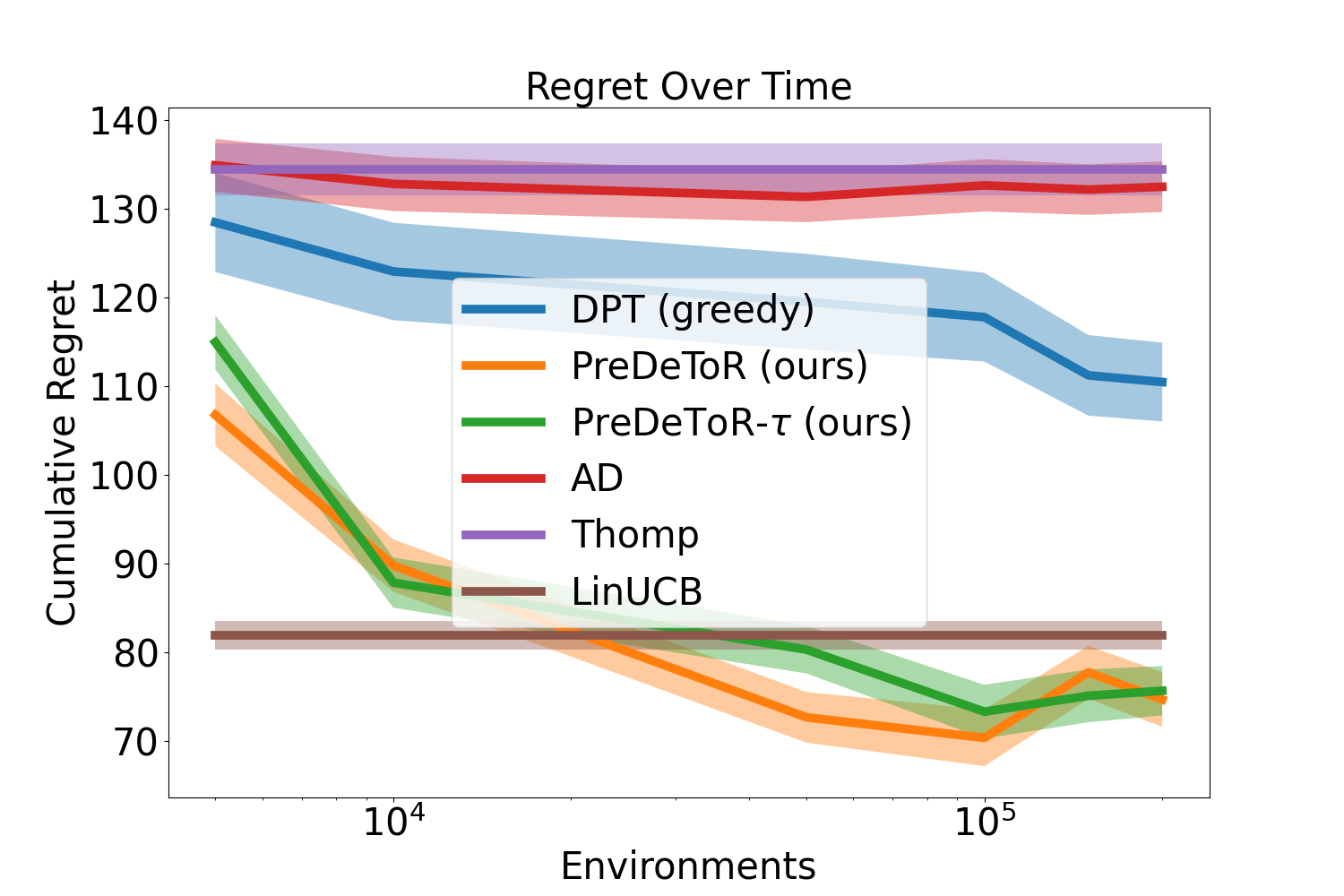}
   \caption{Increasing tasks}
   \label{fig:env-200}
\end{subfigure}
\vspace*{-1em}
\caption{Experiment with an increasing number of tasks. The y-axis shows the cumulative regret.}
\label{fig:expt-inc-env}
\vspace{-0.7em}
\end{figure}

\textbf{Experimental Result:} We observe these outcomes in \Cref{fig:expt-inc-env}. In \Cref{fig:expt-inc-env} we show the linear bandit setting for horizon $n=20$, $\Mpr  \in \{5000, 10000, 50000, 100000, 150000\}$, $\Mts = 200$, $A=20$, and $d=40$. 
Again, the demonstrator $\pi^w$ is the \ts\ algorithm. We observe that \pred\ (\gt), \ad\ and \dptg\ suffer more regret than the \linucb\ when the number of tasks is small ($\Mtr\in\{5000, 10000\}$ in \Cref{fig:env-5}, and \ref{fig:env-10}. However in \Cref{fig:env-50}, \ref{fig:env-100}, \ref{fig:env-150}, and \ref{fig:env-200} we show that \pred\ has lower regret than \ts\ and matches \linucb. This shows that \pred\ (\gt) is exploiting the latent linear structure of the underlying tasks for the non-linear setting.
Moreover, observe that as $\Mtr$ increases the \pred\ has lower cumulative regret than \dptg, \ad. Note that for any task $m$ for the horizon $20$ the \ts\ will be able to sample all the actions at most once. Therefore \dptg\ does not perform as well as \pred.
%
%
Finally, note that the result shows that \pred\ (\gt) is able to exploit the reward correlation across the tasks better as the number of tasks increases.

%
%


\subsection{Exploration of \pred (\gt) in New Arms Setting}
\label{sec:new-arms-app}
In this section, we discuss the exploration of \pred\ (\gt) in the linear and non-linear new arms bandit setting discussed in \Cref{sec:new-actions}.  Recall that we consider the linear bandit setting of horizon $n=50$, $\Mpr = 200000$, $\Mts = 200$, $A=20$, and $d=5$. 
Here during data collection and during collecting the test data, we randomly select one new action from $\R^d$ for each task $m$. 
So the number of invariant actions is $|\Anc| = 19$.

\textbf{Outcomes:} We first discuss the main outcomes of our analysis of exploration in the low-data regime:

\begin{tcolorbox}
    \customfinding The \pred\ (\gt) is robust to changes when the number of in-variant actions is large. \pred\ (\gt) performance drops as shared structure breaks down.
\end{tcolorbox}



We first show in \Cref{fig:train-dist-new-1} the training distribution of the optimal actions. For each bar, the frequency indicates the number of tasks where the action (shown in the x-axis) is the optimal action.

Then in \Cref{fig:analysis-epxploration-new-1} we show how the sampling distribution of \dptg, \pred\, and \predt\ change in the first $10$ and last $10$ rounds for all the tasks where action $17$ is optimal. We plot this graph the same way as discussed in \Cref{sec:linear}.
%
%
From the figure \Cref{fig:analysis-epxploration-new-1} we see that \pred (\gt) consistently pulls the action $17$ more than \dptg. It also explores other optimal actions like $\{1,2,3,8,9,15\}$ but discards them quickly in favor of the optimal action $17$ in these tasks. 

Finally, we plot the feasible action set considered by \dptg, \pred, and \predt\ in \Cref{fig:analysis-exploration-time-new-1}. To plot this graph again we consider the test tasks where the optimal action is $17$. Then we count the number of distinct actions that are taken from round $t$ up until horizon $n$. Finally we average this over all the considered tasks where the optimal action is $17$. We call this the candidate action set considered by the algorithm. From the \Cref{fig:analysis-exploration-time-new-1} we see that \predt\ explores more than \pred\ in this setting.


We also show how the prediction error of the optimal action by \pred\ compared to \linucb\ in this $1$ new arm linear bandit setting. In \Cref{fig:short-horizon-action-dist-new-1} we first show how the $20$ actions are distributed in the $\Mts=200$ test tasks. In \Cref{fig:short-horizon-action-dist-new-1} for each bar, the frequency indicates the number of tasks where the action (shown in the x-axis) is the optimal action. Then in \Cref{fig:short-horizon-action-error-new-1} we show the prediction error of \pred\ (\gt) for each task $m\in[\Mts]$. The prediction error is calculated the same way as stated in \Cref{sec:new-actions} 
From the \Cref{fig:short-horizon-action-error-new-1} we see that for most actions the prediction error of \pred\ (\gt) is closer to \linucb\ showing that the introduction of $1$ new action does not alter the prediction error much.
%
Note that \linucb\ estimates the empirical mean directly from the test task, whereas \pred\ has a strong prior based on the training data. Therefore we see that \pred\ is able to estimate the reward of the optimal action quite well from the training dataset $\Dpr$.
%


\begin{figure}[!hbt]
\centering
\begin{subfigure}[b]{0.32\textwidth}
    \includegraphics[scale=0.27]{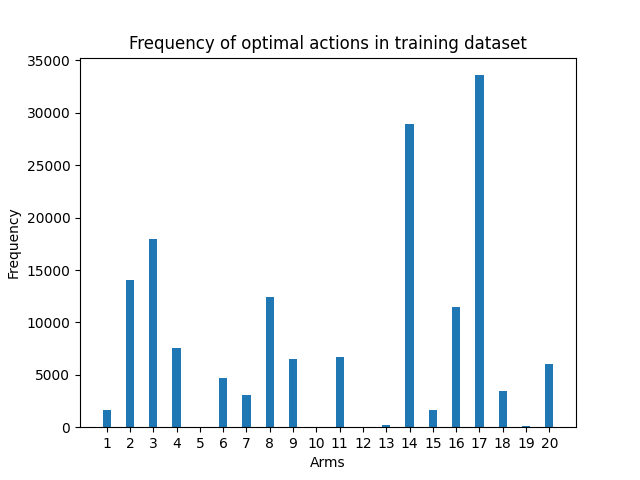}
    \caption{Train Optimal Action Distribution}
    \label{fig:train-dist-new-1}
\end{subfigure}%
\begin{subfigure}[b]{0.32\textwidth}
    \includegraphics[scale=0.27]{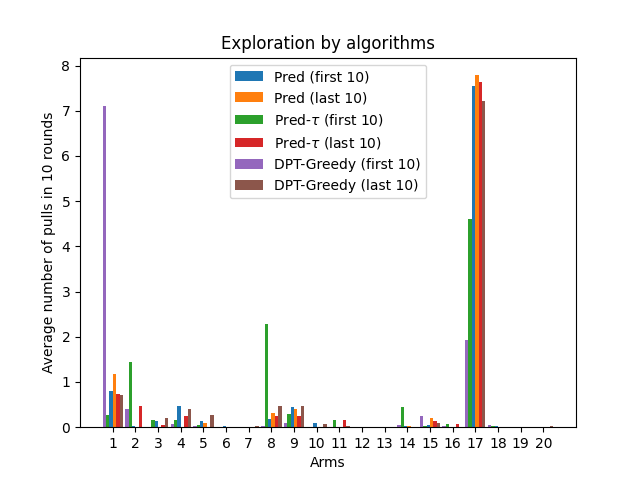}
    \caption{Distribution of action sampling in all tasks where action $17$ is optimal}
    \label{fig:analysis-epxploration-new-1}
\end{subfigure}%
\begin{subfigure}[b]{0.32\textwidth}
    \includegraphics[scale=0.27]{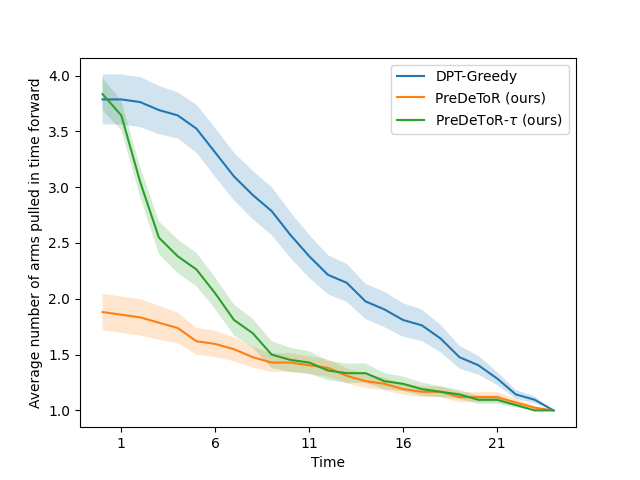}
    \caption{Candidate Action Set in Time averaged over all tasks where action $17$ is optimal}
    \label{fig:analysis-exploration-time-new-1}
\end{subfigure}
\caption{Exploration Analysis of \pred (\gt) in linear 1 new arm setting}
\label{fig:exploration-analysis-time-new-1}
\end{figure}

\begin{figure}[!hbt]
\centering
\vspace*{-1em}
\begin{subfigure}[b]{0.48\textwidth}
    \includegraphics[scale=0.27]{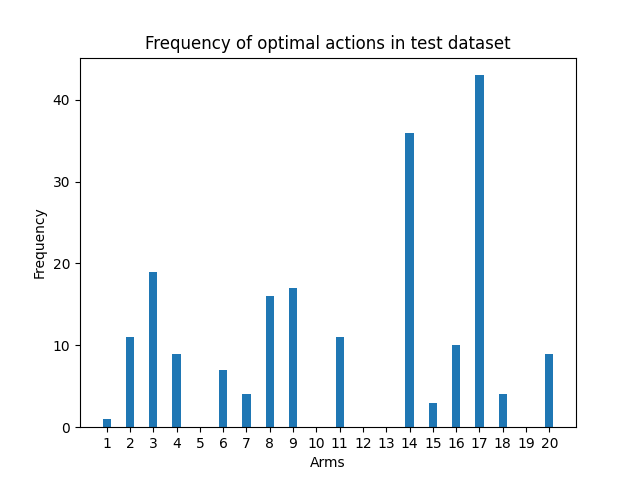}
    \caption{Test action distribution}
    \label{fig:short-horizon-action-dist-new-1}
\end{subfigure}%
\begin{subfigure}[b]{0.48\textwidth}
    \includegraphics[scale=0.27]{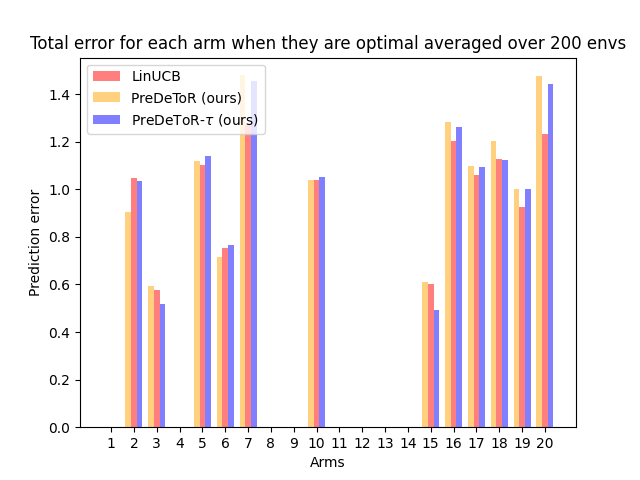}
    \caption{Test Prediction Error}
    \label{fig:short-horizon-action-error-new-1}
\end{subfigure}
\vspace*{-1em}
\caption{Prediction error of \pred (\gt) in linear 1 new arm setting}
\label{fig:expt-short-horizon-new-1}
\vspace{-0.7em}
\end{figure}

We now consider the setting where the number of invariant actions is $|\Anc| = 15$.
We again show in \Cref{fig:train-dist-new-5} the training distribution of the optimal actions. For each bar, the frequency indicates the number of tasks where the action (shown in the x-axis) is the optimal action. Then in \Cref{fig:analysis-epxploration-new-5} we show how the sampling distribution of \dptg, \pred\, and \predt\ change in the first $10$ and last $10$ rounds for all the tasks where action $17$ is optimal. We plot this graph the same way as discussed in \Cref{sec:linear}.
%
%
From the figure \Cref{fig:analysis-epxploration-new-5} we see that none of the algorithms \pred, \predt, \dptg\ consistently pulls the action $17$ more than other actions. This shows that the common underlying actions across the tasks matter for learning the epxloration.

Finally, we plot the feasible action set considered by \dptg, \pred, and \predt\ in \Cref{fig:analysis-exploration-time-new-5}. To plot this graph again we consider the test tasks where the optimal action is $17$. We build the candidate set the same way as before. From the \Cref{fig:analysis-exploration-time-new-5} we see that none of the three algorithms \dptg, \pred, \predt, is able to sample the optimal action $17$ sufficiently high number of times.


We also show how the prediction error of the optimal action by \pred\ compared to \linucb\ in this $1$ new arm linear bandit setting. In \Cref{fig:short-horizon-action-dist-new-5} we first show how the $20$ actions are distributed in the $\Mts=200$ test tasks. In \Cref{fig:short-horizon-action-dist-new-5} for each bar, the frequency indicates the number of tasks where the action (shown in the x-axis) is the optimal action. Then in \Cref{fig:short-horizon-action-error-new-5} we show the prediction error of \pred\ (\gt) for each task $m\in[\Mts]$. The prediction error is calculated the same way as stated in \Cref{sec:new-actions}. 
From the \Cref{fig:short-horizon-action-error-new-5} we see that for most actions the prediction error is higher than \linucb\ showing that the introduction of $5$ new actions (and thereby decreasing the invariant action set) significantly alters the prediction error.

\begin{figure}[!hbt]
\centering
\begin{subfigure}[b]{0.32\textwidth}
   \includegraphics[scale=0.27]{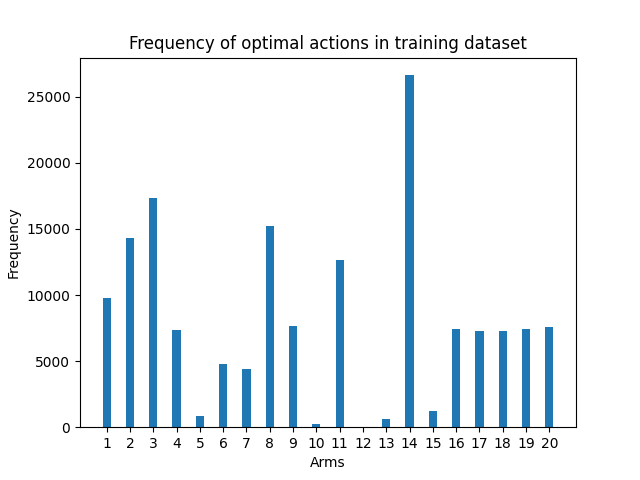}
   \caption{Train Optimal Action Distribution}
   \label{fig:train-dist-new-5}
\end{subfigure}%
\begin{subfigure}[b]{0.32\textwidth}
   \includegraphics[scale=0.27]{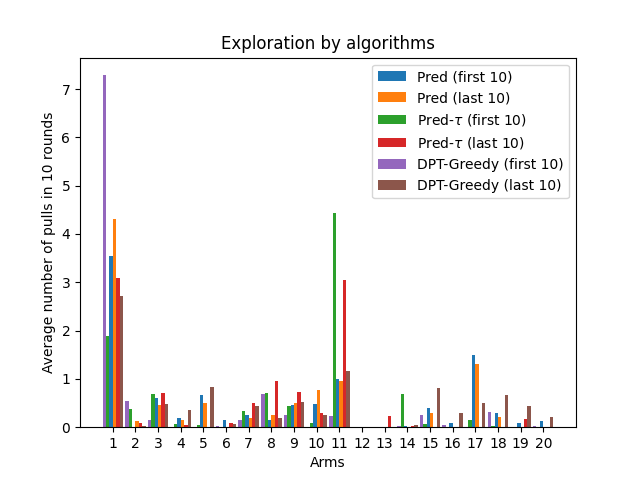}
   \caption{Distribution of action sampling in all tasks where action $17$ is optimal}
   \label{fig:analysis-epxploration-new-5}
\end{subfigure}%
\begin{subfigure}[b]{0.32\textwidth}
   \includegraphics[scale=0.27]{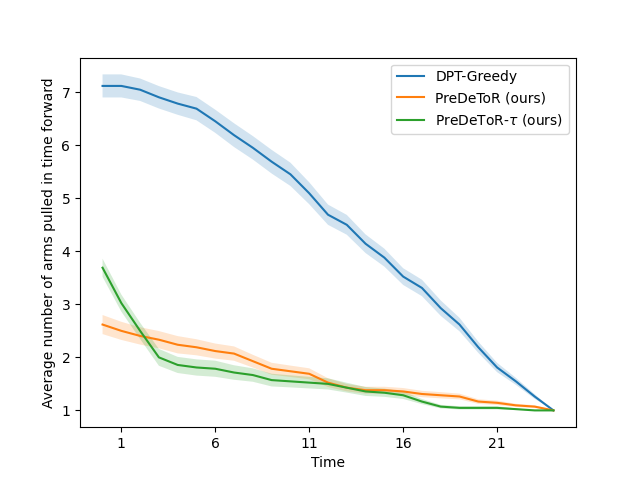}
   \caption{Candidate Action Set in Time averaged all tasks where action $17$ is optimal}
   \label{fig:analysis-exploration-time-new-5}
\end{subfigure}
\caption{Exploration Analysis of \pred (\gt) in linear 5 new arm setting}
\label{fig:exploration-analysis-time-new-5}
\end{figure}

\begin{figure}[!hbt]
\centering
\vspace*{-1em}
\begin{subfigure}[b]{0.48\textwidth}
   \includegraphics[scale=0.27]{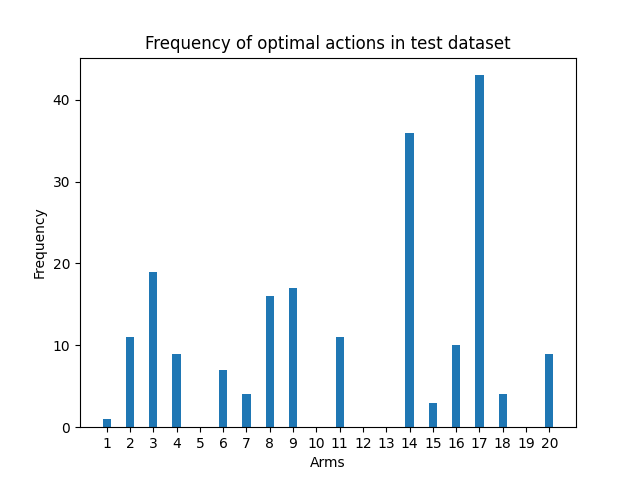}
   \caption{Test action distribution}
   \label{fig:short-horizon-action-dist-new-5}
\end{subfigure}%
\begin{subfigure}[b]{0.48\textwidth}
   \includegraphics[scale=0.27]{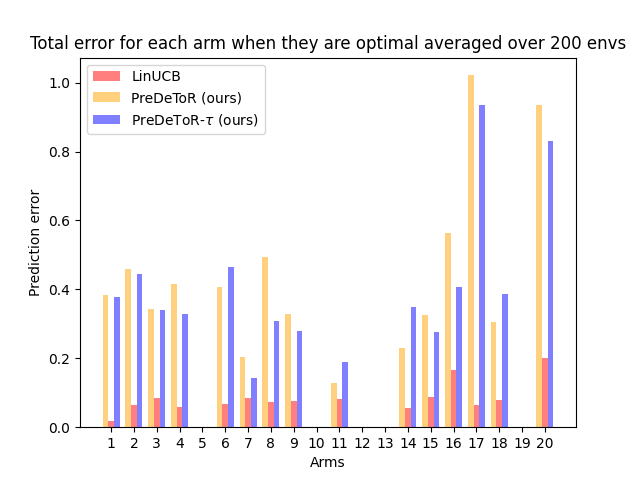}
   \caption{Test Prediction Error}
   \label{fig:short-horizon-action-error-new-5}
\end{subfigure}
\vspace*{-1em}
\caption{Prediction error of \pred (\gt) in linear 1 new arm setting}
\label{fig:expt-short-horizon-new-5}
\vspace{-0.7em}
\end{figure}


\subsection{Empirical Validation of Theoretical Result}
\label{sec:validation}
In this section, we empirically validate the theoretical result proved in \Cref{sec:theory}. We again consider the linear bandit setting discussed in \Cref{sec:short-horizon}. Recall that the linear bandit setting consist of horizon $n=25$, $\Mpr = \{100000, 200000\}$, $\Mts = 200$, $A=10$, and $d=2$. 
The demonstrator $\pi^w$ is the \ts\ algorithm and we observe that \pred\ (\gt) has lower cumulative regret than \dptg, \ad\ and matches the performance of \linucb. 

\textbf{Baseline (\linucbt):} We define soft LinUCB (\linucbt) as follows: At every round $t$ for task $m$, it calculates the ucb value $B_{m,a,t}$ for each action $\bx_{m,a} \in \X$ such that $B_{m,a,t} = \bx_{m,a}^\top \wtheta_{m,t-1} + \alpha\|\bx_{m,a}\|_{\bSigma_{m,t-1}^{-1}}$ where $\alpha > 0$ is a constant and $\wtheta_{m,t}$ is the estimate of the model parameter $\btheta_{m, *}$ at round $t$. 
Here, $\bSigma_{m,t-1} = \sum_{s=1}^{t-1}\bx_{m,s}\bx_{m,s}^\top +\lambda\bI_d$ is the data covariance matrix or the arms already tried.
Then it chooses $I_t \sim \mathrm{softmax}_a^\tau(B_{m,a,t})$, where $\mathrm{softmax}_a^\tau(\cdot)\in\triangle^A$ denotes a softmax distribution over the actions and $\tau$ is a temperature parameter (See \Cref{sec:short-horizon} for definition of $\mathrm{softmax}_a^\tau(\cdot)$).

\textbf{Outcomes:} We first discuss the main outcomes of our experimental results:

\begin{tcolorbox}
\customfinding \pred\ (\gt) excels in predicting the rewards for test tasks when the number of training (source) tasks is large. 
\end{tcolorbox}




\begin{figure}[!hbt]
\centering
\begin{subfigure}[b]{0.32\textwidth}
   \includegraphics[scale=0.27]{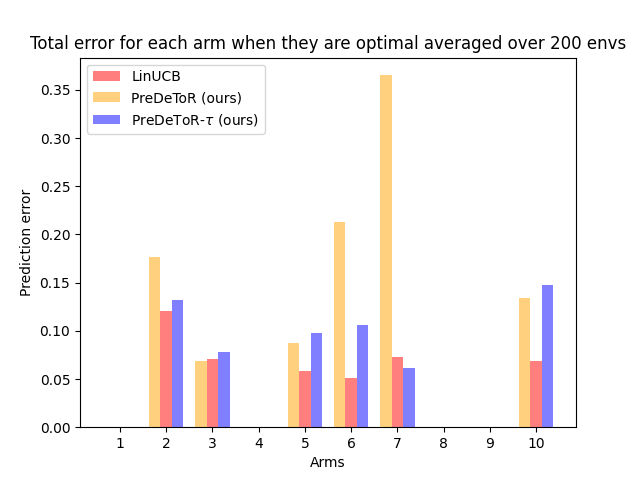}
   \caption{Prediction Error for $10^5$ tasks}
   \label{fig:prediction-error-small}
\end{subfigure}%
\begin{subfigure}[b]{0.32\textwidth}
   \includegraphics[scale=0.27]{img/Neurips/linear/analysis_total.png}
   \caption{Prediction Error for $2\times 10^5$ tasks}
   \label{fig:prediction-error-large}
\end{subfigure}%
\begin{subfigure}[b]{0.32\textwidth}
   \includegraphics[scale=0.13]{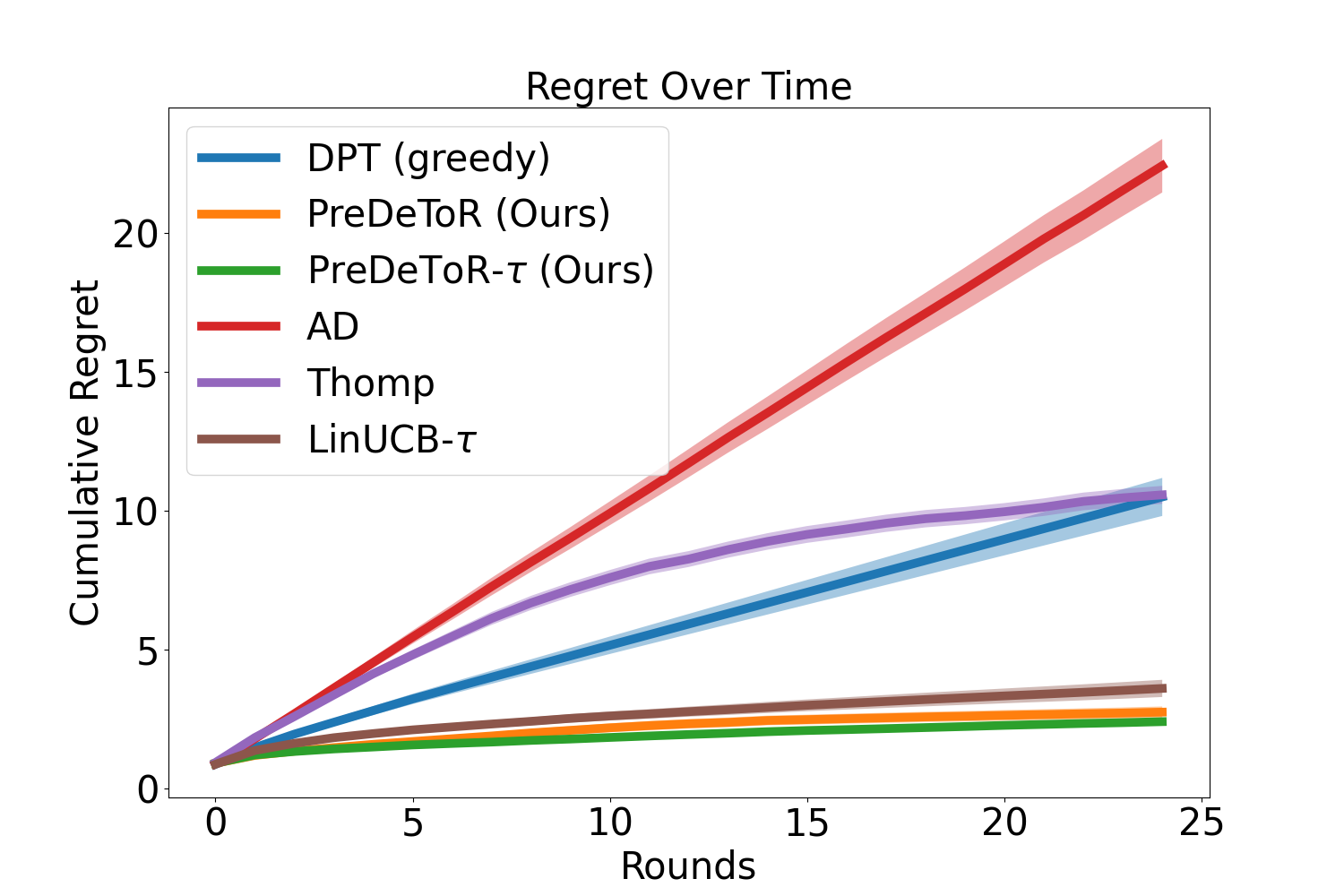}
   \caption{Cumulative Regret of \pred\ (\gt) compared against \linucbt}
   \label{fig:comparison-with-linucb}
\end{subfigure}
\caption{Empirical validation of theoretical analysis}
\label{fig:empirical-validation}
\end{figure}

\textbf{Experimental Result:} These findings are reported in \Cref{fig:empirical-validation}.
In \Cref{fig:prediction-error-small} we show the prediction error of \pred\ (\gt) for each task $m\in[\Mts]$. The prediction error is calculated as $(\wmu_{m,n, *}(a) - \mu_{m, *}(a))^2$ where $\wmu_{m,n, *}(a) = \max_a\wtheta_{m,n}^\top\bx_m(a)$ is the empirical mean at the end of round $n$, and $\mu_{*,m}(a)=\max_a\btheta_{m,*}^\top\bx_m(a)$ is the true mean of the optimal action in task $m$. Then we average the prediction error for the action $a\in [A]$ by the number of times the action $a$ is the optimal action in some task $m$. We see that when the source tasks are $100000$ the reward prediction falls short of \linucb\ prediction for all actions except action $2$.

In \Cref{fig:prediction-error-large} we again show the prediction error of \pred\ (\gt) for each task $m\in[\Mts]$ when the source tasks are $200000$. Note that in both these settings, we kept the horizon $n=25$, and the same set of actions. We now observe that the reward prediction almost matches \linucb\ prediction in almost all the optimal actions.

In \Cref{fig:comparison-with-linucb} we compare \pred\ (\gt) against \linucbt\ and show that they almost match in the linear bandit setting discussed in \Cref{sec:short-horizon} when the source tasks are $100000$. 

\subsection{Empirical Study: Offline Performance}
\label{sec:offline}
In this section, we discuss the offline performance of \pred\ when the number of tasks $\Mpr \gg A \geq n$.

We first discuss how \pred\ (\gt) is modified for the offline setting. 
In the offline setting, the \pred\ first samples a task $m\sim\cTs$, then the test dataset $\H_m\sim \Dts(\cdot|m)$. Then \pred\ and \predt\ act similarly to the online setting, but based on the entire offline dataset $\H_m$. 
%
The full pseudocode of \pred\ is in \Cref{alg:pred-off}.

\begin{algorithm}[!tbh]
\caption{\textbf{Pre}-trained \textbf{De}cision \textbf{T}ransf\textbf{o}rmer with \textbf{R}eward Estimation (\pred)}
\label{alg:pred-off}
    \begin{algorithmic}[1]
    \STATE \textbf{Collecting Pretraining Dataset} 
    \STATE Initialize empty pretraining dataset $\Htr$
    \FOR{$i$ in $[\Mpr]$ }
    \STATE Sample task $m \sim \cTp$, in-context dataset $\H_m \sim \Dpr(\cdot | m)$ and add this to $\Htr$.
    \ENDFOR
    \STATE \textbf{Pretraining model on dataset}
    \STATE Initialize model $\T_{\bTheta}$ with parameters $\bTheta$
    \WHILE{\text{not converged}}
    \STATE Sample $\H_m$ from $\Htr$ and predict $\wr_{m,t}$ for action $(I_{m,t})$ for all $t \in[n]$
    \STATE Compute loss in \eqref{eq:loss-transformer} with respect to $r_{m,t}$ and backpropagate to update model parameter $\bTheta$.
    \ENDWHILE
    \STATE \textbf{Offline test-time deployment}
    \STATE Sample unknown task $m \sim \cTs$, sample dataset $\H_m \sim \Dts(\cdot | m)$
    \STATE Use $\T_{\bTheta}$ on $m$ at round $t$ to choose 
    \begin{align*}
        I_t \begin{cases}
            = \argmax_{a\in\A} \T_{\bTheta}\left( \wr_{m,t}(a) \mid \H_m\right), & \textbf{\pred} \\
            \sim \textrm{softmax}^\tau_a \T_{\bTheta}\left( \wr_{m,t}(a) \mid \H_m\right), & \textbf{\predt}
        \end{cases}
    \end{align*}
    \end{algorithmic}
\end{algorithm}

Recall that $\Dts$ denote a distribution over all possible interactions that can be generated by $\pi^w$ during test time. 
For offline testing, first, a test task $m\sim\cTs$, and then an in-context test dataset $\H_m$ is collected such that $\H_m\sim\Dts(\cdot|m)$. 
%
%
Observe from \Cref{alg:pred-off} that in the offline setting, \pred\ first samples a task $m\sim\cTs$, and then a test dataset $\H_m\sim \Dts(\cdot|m)$ and acts greedily. 
Crucially in the offline setting the 
\pred\ does not add the observed reward $r_t$ at round $t$ to the dataset.
Through this experiment, we want to evaluate the performance of \pred\ to learn the underlying latent structure and reward correlation when the horizon is small.
Finally, recall that when the horizon is small the weak demonstrator $\pi^w$ does not have sufficient samples for each action. This leads to a poor approximation of the greedy action.

%

\textbf{Baselines:} We again implement the same baselines discussed in \Cref{sec:short-horizon}. The baselines are \pred, \predt, \dptg, \ad, \ts, and \linucb. During test time evaluation for offline setting the \dpt\ selects $I_t =  \widehat{a}_{m,t,*}$ where $\widehat{a}_{m,t,*} = \argmax_a \T_{\bTheta}(a|\H^t_m)$ is the predicted optimal action.

\textbf{Outcomes:} We first discuss the main outcomes of our experimental results for increasing the horizon:

\begin{tcolorbox}
\customfinding  \pred\ (\gt) performs comparably to \dptg\ and \ad\ in the offline setting. 
\end{tcolorbox}

\begin{figure}[!hbt]
\centering
\vspace*{-1em}
\begin{subfigure}[b]{0.48\textwidth}
   \includegraphics[scale=0.15]{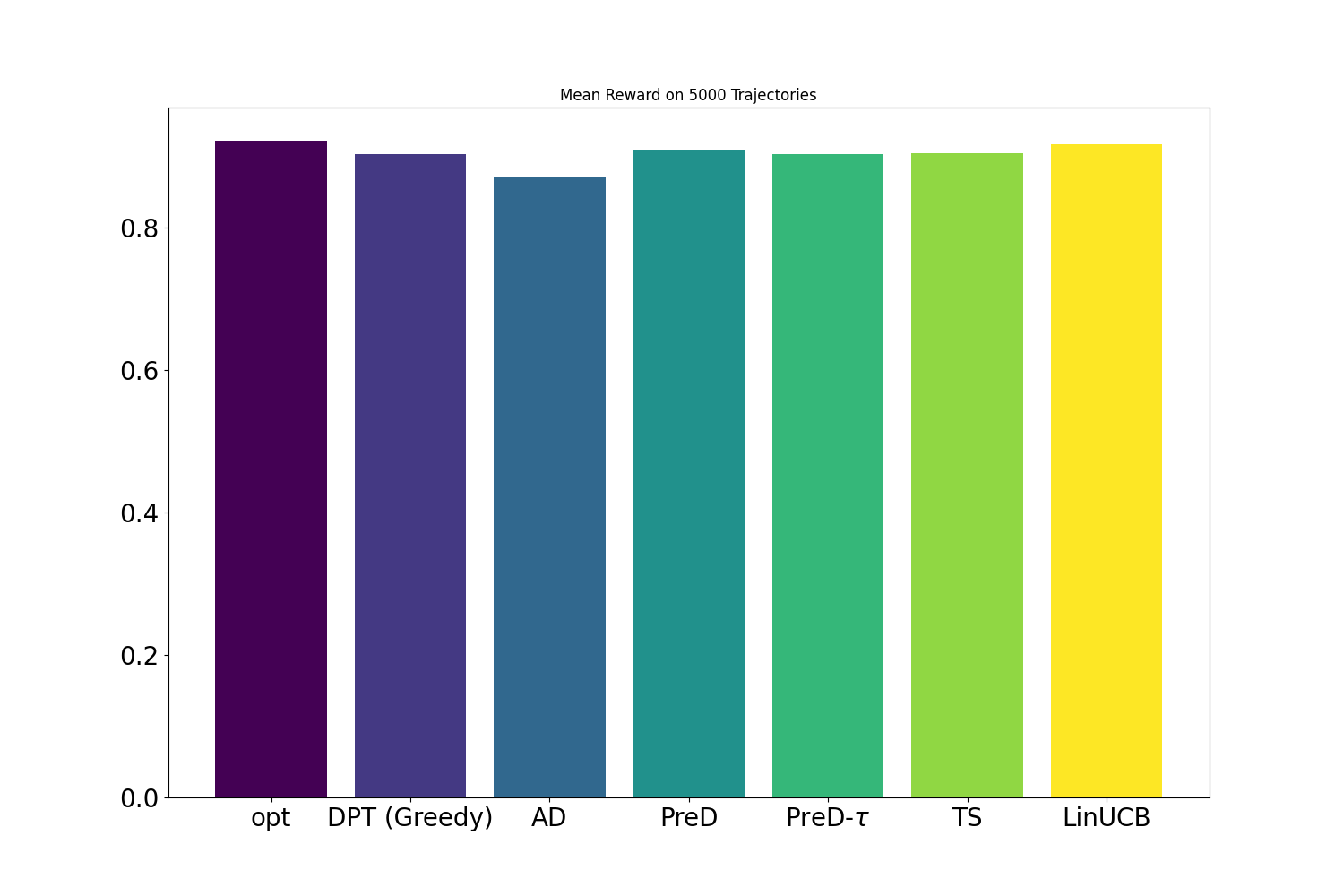}
   \caption{Offline for Linear setting}
   \label{fig:off-lin-small}
\end{subfigure}%
\begin{subfigure}[b]{0.48\textwidth}
   \includegraphics[scale=0.15]{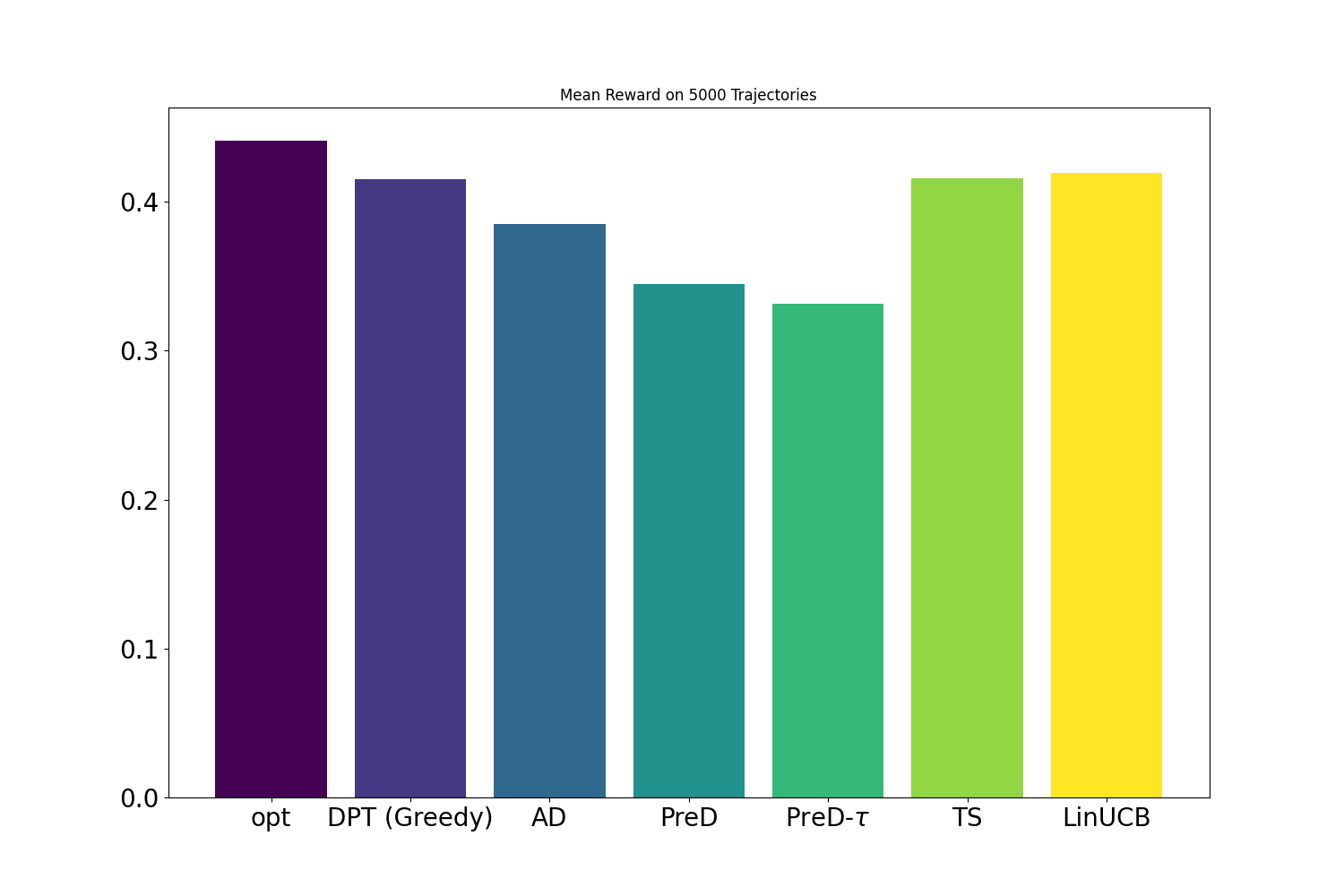}
   \caption{Offline for Non-linear setting}
   \label{fig:off-lin-large}
\end{subfigure}
\vspace*{-1em}
\caption{Offline experiment. The y-axis shows the cumulative reward.}
\label{fig:expt-offline}
\vspace{-0.7em}
\end{figure}

\textbf{Experimental Result:} We observe these outcomes in \Cref{fig:expt-offline}. In \Cref{fig:expt-offline} we show the linear bandit setting for horizon $n=20$, $\Mpr =  200000$, $\Mts = 5000$, $A=20$, and $d=5$ for the low data regime.
%
%
Again, the demonstrator $\pi^w$ is the \ts\ algorithm. We observe that \pred\ (\gt) has comparable cumulative regret to \dptg. Note that for any task $m$ for the horizon $n=20$ the \ts\ will be able to sample all the actions at most once. 
%
%
%
In the non-linear setting of \Cref{fig:off-lin-large} the $n=40$, $\Mpr =  100000$, $A=6$, $d=2$.
Observe that in all of these results, the performance of \pred\ (\gt) is comparable with respect to cumulative regret against \dptg.

\section{Theoretical Analysis}
\label{sec:theory-app}

\subsection{Proof of \Cref{lemma:bayes-reg-1}}
\label{app:proof-lemma-2}

\begin{proof}
The learner collects $n$ rounds of data following $\pi^w$. The weak demonstrator $\pi^w$ only observes the $\{I_t, r_t\}_{t=1}^n$. Recall that $N_n(a)$ denotes the total number of times the action $a$ is sampled for $n$ rounds. 
Define the matrix $\bH_n \in \R^{n \times A}$ where the $t$-th row represents the action sampled at round $t\in [n]$. 
The $t$-th row is a one-hot vector with $1$ as the $a$-th component in the vector for $a\in [A]$.
Then define the reward vector $\bY_n \in \R^n$ as the reward vector where the $t$-th component is the observed reward for the action $I_t$ for $t\in [n]$.
%
Finally define the diagonal matrix $\bD_A\in\R^{A\times A}$ as in \eqref{eq:D-matrix} and the estimated reward covariance matrix as $\bS_A\in\R^{A\times A}$ such that $\bS_A(a,a') = \wmu_n(a)\wmu_n(a')$. This matrix captures the reward correlation between the pairs of actions $a,a'\in [A]$.

Assume $\mu\sim\N(0, \bS_*)$ where $\bS_*\in\R^{A\times A}$. Then the observed mean vector $\bY_n$ is 
\begin{align*}
    \bY_n = \bH_n \mu + \bH_n \bD_A^{1/2} \eta_n
\end{align*}
where, $\eta_n$ is the noise vector over the $[n]$ training data. Then the posterior mean of $\wmu$ by Gauss Markov Theorem \citep{johnson2002applied} is given by
\begin{align}
    \wmu = \bS_*\bH_n^\top\left(\bH_n(\bS_* + \bD_A)\bH_n^\top\right)^{-1}\bY_n. \label{eq:posterior}
\end{align}
However, the learner does not know the true reward co-variance matrix. Hence it needs to estimate the $\bS_*$ from the observed data. Let the estimate be denoted by $\bS_A$. 

\begin{assumption}
\label{assm:exploration}
    We assume that $\pi^w$ is sufficiently exploratory so that each action is sampled at least once. 
\end{assumption}
The \Cref{assm:exploration} ensures that the matrix $\left(\sigma^2_{\btheta}\bH_n(\bS_A + \bD_A)\bH_n^\top\right)^{-1}$ is invertible.
Under \Cref{assm:exploration}, plugging the estimate $\bS_A$ back in \eqref{eq:posterior} shows that the average posterior mean over all the tasks is
\begin{align}
    \wmu = \bS_A\bH_n^\top\left(\bH_n(\bS_A + \bD_A)\bH_n^\top\right)^{-1}\bY_n.
\end{align}
The claim of the lemma follows.
\end{proof}

\section{Generalization and Transfer Learning Proof for \pred}
\label{sec:generalization-app}
\subsection{Generalization Proof}
\label{app:generalization}

$\alg$ is the space of algorithms induced by the transformer $\T_{\bTheta}$. 


\begin{theorem}
\label{thm:multi-task-risk-app}\textbf{(\pred\ risk)}
    Suppose error stability \Cref{assm:stability-assumption} holds and assume loss function $\ell(\cdot,\cdot)$ is $C$-Lipschitz for all $r_t \in [0,B]$ and horizon $n\geq 1$. Let $\widehat{\T}$ be the empirical solution of (ERM) and $\mathcal{N}(\mathcal{A}, \rho, \epsilon)$ be the covering number of the algorithm space $\alg$ following Definition \ref{def:covering-number} and \ref{def:alg-dist}. Then with probability at least $1-2 \delta$, the excess Multi-task learning (MTL) risk of \predt\ is bounded by 
    \begin{align*}
        \cR^{}_{\mathrm{MTL}}(\widehat{\T}_{}) \leq 4 \tfrac{C}{\sqrt{nM}} +2(B+K \log n) \sqrt{\tfrac{\log (\N(\alg, \rho, \varepsilon) / \delta)}{c n M}}
    \end{align*}
    where, $\N(\alg, \rho, \varepsilon)$ is the covering number of transformer $\widehat{\T}_{}$. 
\end{theorem}

\begin{proof}
We consider a meta-learning setting. Let $M$ source tasks are i.i.d. sampled from a task distribution $\cT$, and let $\widehat{\T}$ be the empirical Multitask (MTL) solution. 
Define $\H_{\text {all }}=\bigcup_{m=1}^M \H_m$. We drop the $\bTheta, \mathbf{r}$ from transformer notation $\rT_{\bTheta}$ as we keep the architecture fixed as in \citet{lin2023transformers}. 
Note that this transformer predicts a reward vector over the actions. To be more precise we denote the reward predicted by the transformer at round $t$ after observing history $\H_m^{t-1}$ and then sampling the action $a_{mt}$ as $\T\left(\wr_{m t}(a_{mt})|\H_m^{t-1}, a_{m t}\right)$.
Define the  training risk $$\widehat{\L}_{\H_{\mathrm{all}}}(\T)=\frac{1}{n M} \sum_{m=1}^M \sum_{t=1}^n \ell\left(r_{m t}(a_{mt}), \T\left(\wr_{m t}(a_{mt})|\H_m^{t-1}, a_{m t}\right)\right)$$ and the test risk $$\L_{\mathrm{MTL}}(\T)=\mathbb{E}\left[\widehat{\L}_{\H_{\text {all }}}(\T)\right].$$ 
Define empirical risk minima $\widehat{\T}=\arg \min _{\T \in \alg} \widehat{\L}_{\H_{\text {all }}}(\T)$ and population minima 
\begin{align*}
    \T^{*}=\arg \min _{\T \in \alg} \L_{\mathrm{MTL}}(\T)
\end{align*}
In the following discussion, we drop the subscripts MTL and $\H_{\text {all. }}$ The excess MTL risk is decomposed as follows:
\begin{align*}
\cR_{\mathrm{MTL}}(\widehat{\T}) & =\L(\widehat{\T})-\L\left(\T^*\right) \\
& =\underbrace{\L(\widehat{\T})-\widehat{\L}(\widehat{\T})}_a+\underbrace{\widehat{\L}(\widehat{\T})-\widehat{\L}\left(\T^*\right)}_b+\underbrace{\widehat{\L}\left(\T^*\right)-\L(\T^*}_c) .
\end{align*}
Since $\widehat{\T}$ is the minimizer of empirical risk, we have $b \leq 0$. 

\textbf{Step 1: (Concentration bound $|\L(\T)-\widehat{\L}(\T)|$ for a fixed $\T \in \alg$)} Define the test/train risks of each task as follows:
\begin{align*}
& \widehat{\L}_m(\T):=\frac{1}{n} \sum_{t=1}^n \ell\left(r_{mt}(a_{mt}), \T\left(\wr_{mt}(a_{mt})|\H_m^{t-1}, a_{m t}\right)\right), \quad \text { and } \\
& \L_m(\T):=\mathbb{E}_{\H_m}\left[\widehat{\L}_m(\T)\right]=\mathbb{E}_{\H_m}\left[\frac{1}{n} \sum_{t=1}^n \ell\left(r_{mt}(a_{mt}), \T\left(\wr_{mt}(a_{mt})|\H_m^{t-1}, a_{m t}\right)\right)\right], \quad \forall m \in[M] .
\end{align*}
Define the random variables $X_{m, t}=\mathbb{E}\left[\widehat{\L}_t(\T) \mid \H_m^t\right]$ for $t \in[n]$ and $m \in[M]$, that is, $X_{m, t}$ is the expectation over $\widehat{\L}_t(\T)$ given training sequence $\H_m^t=\left\{\left(a_{m t'}, r_{m t'}\right)\right\}_{t'=1}^t$ (which are the filtrations). With this, we have that $X_{m, n}=\mathbb{E}\left[\widehat{\L}_m(\T) \mid \H_m^n\right]=\widehat{\L}_m(\T)$ and $X_{m, 0}=\mathbb{E}\left[\widehat{\L}_m(\T)\right]=\L_m(\T)$. 
More generally, $\left(X_{m, 0}, X_{m, 1}, \ldots, X_{m, n}\right)$ is a martingale sequence since, for every $m \in [M]$, we have that $\mathbb{E}\left[X_{m, i} \mid \H_m^{t-1}\right]=X_{m, t-1}$.
For notational simplicity, in the following discussion, we omit the subscript $m$ from $a, r$ and $\H$ as they will be clear from the left-hand-side variable $X_{m, t}$. We have that
\begin{align*}
X_{m, t} & =\mathbb{E}\left[\left.\frac{1}{n} \sum_{t=1}^n \ell\left(r_{t'}, \operatorname{TF}\left(\wr_{t'}|\H^{t'-1}, a_{t'}\right)\right) \right\rvert\, \H^t\right] \\
& =\frac{1}{n} \sum_{t'=1}^t \ell\left(r_{t'}, \operatorname{TF}\left(\wr_{t'}|\H^{t'-1}, a_{t'}\right)\right)+\frac{1}{n} \sum_{t'=t+1}^n \mathbb{E}\left[\ell\left(r_{t'}, \operatorname{TF}\left(\wr_{t'}|\H^{t'-1}, a_{t'}\right)\right) \mid \H^t\right]
\end{align*}
Using the similar steps as in \citet{li2023transformers} we can show that
\begin{align*}
    \left|X_{m, t}-X_{m, t-1}\right| \overset{(a)}{\leq} \frac{B}{n}+\sum_{t'=t+1}^n \frac{K}{t' n} \leq \frac{B+K \log n}{n} .
\end{align*}
where, $(a)$ follows by using the fact that loss function $\ell(\cdot, \cdot)$ is bounded by $B$, and error stability assumption.

Recall that $\left|\L_m(\T)-\widehat{\L}_m(\T)\right|=\left|X_{m, 0}-X_{m, n}\right|$ and for every $m \in[M]$, we have $\sum_{t=1}^n\left|X_{m, t}-X_{m, t-1}\right|^2 \leq \frac{(B+K \log n)^2}{n}$. As a result, applying Azuma-Hoeffding's inequality, we obtain
\begin{align}
    \Pb\left(\left|\L_m(\T)-\widehat{\L}_m(\T)\right| \geq \tau\right) \leq 2 e^{-\frac{n \tau^2}{2(B+K \log n)^2}}, \quad \forall m \in[M]  \label{eq:azuma}
\end{align}
Let us consider $Y_m:=\L_m(\T)-\widehat{\L}_m(\T)$ for $m \in[M]$. Then, $\left(Y_m\right)_{m=1}^M$ are i.i.d. zero mean sub-Gaussian random variables. There exists an absolute constant $c_1>0$ such that, the subgaussian norm, denoted by $\|\cdot\|_{\psi_2}$, obeys $\left\|Y_m\right\|_{\psi_2}^2<\frac{c_1(B+K \log n)^2}{n}$ via Proposition 2.5.2 of (Vershynin, 2018). Applying Hoeffding's inequality, we derive
\begin{align*}
    \Pb\left(\left|\frac{1}{M} \sum_{m=1}^M Y_t\right| \geq \tau\right) \leq 2 e^{-\frac{c n M \tau^2}{(B+K \log n)^2}} \Longrightarrow \Pb(|\widehat{\L}(\T)-\L(\T)| \geq \tau) \leq 2 e^{-\frac{c n M \tau^2}{(B+K \log n)^2}}
\end{align*}
where $c>0$ is an absolute constant. Therefore, we have that for any $\T \in \alg$, with probability at least $1-2 \delta$,
\begin{align}
    |\widehat{\L}(\T)-\L(\T)| \leq(B+K \log n) \sqrt{\frac{\log (1 / \delta)}{c n M}} \label{eq:conc-bound}
\end{align}

\textbf{Step 2: (Bound $\sup _{\T \in \alg}|\L(\T)-\widehat{\L}(\T)|$ where $\alg$ is assumed to be a continuous search space)}. Let 
$$
h(\T):=\L(\T)-\widehat{\L}(\T)$$ 
and we aim to bound $\sup _{\T \in \alg}|h(\T)|$. Following \Cref{def:alg-dist}, for $\varepsilon>0$, let $\alg_{\varepsilon}$ be a minimal $\varepsilon$-cover of $\alg$ in terms of distance metric $\rho$. Therefore, $\alg_{\varepsilon}$ is a discrete set with cardinality $\left|\alg_{\varepsilon}\right|:=\mathcal{N}(\alg, \rho, \varepsilon)$. Then, we have
\begin{align*}
    \sup _{\T \in \alg}|\L(\T)-\widehat{\L}(\T)| \leq \sup _{\T \in \alg^{\prime}} \min _{\T \in \alg_{\varepsilon}}\left|h(\T)-h\left(\T{ }^{\prime}\right)\right|+\max _{\T \in \alg_{\varepsilon}}|h(\T)| .
\end{align*}
We will first bound the quantity $\sup _{\T \in \alg^{\prime}} \min _{\T \in \alg_{\varepsilon}}\left|h(\T)-h\left(\T{ }^{\prime}\right)\right|$.
We will utilize that loss function $\ell(\cdot, \cdot)$ is $C$-Lipschitz. For any $\T \in \alg$, let $\T \in \alg_{\varepsilon}$ be its neighbor following \Cref{def:alg-dist}. Then 
we can show that
\begin{align*}
&\left|\widehat{\L}(\mathrm{TF})-\widehat{\L}\left(\mathrm{TF}^{\prime}\right)\right| \\
& =\left|\frac{1}{n M} \sum_{m=1}^M \sum_{t=1}^n\left(\ell\left(r_{mt}(a_{mt}), \T\left(\wr_{mt}(a_{mt})|\H_m^{t-1}, a_{m t}\right)\right)-\ell\left(r_{mt}(a_{mt}), \T'\left(\wr_{mt}(a_{mt})|\H_m^{t-1}, a_{m t}\right)\right)\right)\right| \\
& \leq \frac{L}{n M} \sum_{m=1}^M \sum_{t=1}^n\left\|\T\left(\wr_{mt}(a_{mt})|\H_m^{t-1}, a_{m t}\right)-\T'\left(\wr_{mt}(a_{mt})|\H_m^{t-1}, a_{m t}\right)\right\|_{\ell_2} \\
& \leq L \varepsilon .
\end{align*}
Note that the above bound applies to all data-sequences, we also obtain that for any $\T \in \alg$,
$$
\left|\L(\mathrm{TF})-\L\left(\mathrm{TF}^{\prime}\right)\right| \leq L \varepsilon .
$$
Therefore we can show that,
\begin{align}
    \sup _{\T \in \alg} &\min _{\T} \in \alg_{\varepsilon}\left|h(\T)-h\left(\T F^{\prime}\right)\right| \nonumber\\
    &\leq \sup _{\T \in \alg} \min _{\T} \in \alg_{\varepsilon}\left|\widehat{\L}(\T)-\widehat{\L}\left(\T{ }^{\prime}\right)\right|+\left|\L(\T)-\L\left(\T^{\prime}\right)\right| \leq 2 L \varepsilon . \label{eq:perturbation-bound}
\end{align}

Next we bound the second term $\max _{\T \in \alg_{\varepsilon}}|h(\T)|$. Applying union bound directly on $\alg_{\varepsilon}$ and combining it with \eqref{eq:conc-bound}, then we will have that with probability at least $1-2 \delta$,
\begin{align*}
    \max _{\T \in \alg_{\varepsilon}}|h(\T)| \leq(B+K \log n) \sqrt{\frac{\log (\mathcal{N}(\alg, \rho, \varepsilon) / \delta)}{c n M}}
\end{align*}
Combining the upper bound above with the perturbation bound \eqref{eq:perturbation-bound}, we obtain that
\begin{align*}
    \max _{\T \in \alg}|h(\T)| \leq 2 C \varepsilon+(B+K \log n) \sqrt{\frac{\log (\mathcal{N}(\alg, \rho, \varepsilon) / \delta)}{c n M}}.
\end{align*}
It follows then that
\begin{align*}
    \cR_{\mathrm{MTL}}(\widehat{\T}) \leq 2 \sup _{\T \in \alg}|\L(\T)-\widehat{\L}(\T)| \leq 4 C \varepsilon+ 2(B+K \log n) \sqrt{\frac{\log (\mathcal{N}(\alg, \rho, \varepsilon) / \delta)}{c n M}}
\end{align*}
Again by setting $\varepsilon = 1/\sqrt{n M}$
\begin{align*}
    \L(\widehat{\T})-\L\left(\T^*\right)\leq \dfrac{4 C}{\sqrt{n M}}+2(B+K \log n) \sqrt{\frac{\log (\mathcal{N}(\alg, \rho, \varepsilon) / \delta)}{c n M}}
\end{align*}
The claim of the theorem follows.
\end{proof}



\begin{definition}(Covering number)\label{def:covering-number}
Let $Q$ be any hypothesis set and $d\left(q, q^{\prime}\right) \geq 0$ be a distance metric over $q, q^{\prime} \in \mathcal{Q}$. Then, $\bar{Q}=\left\{q_1, \ldots, q_N\right\}$ is an $\varepsilon$-cover of $Q$ with respect to $d(\cdot, \cdot)$ if for any $q \in \mathcal{Q}$, there exists $q_i \in \bar{Q}$ such that $d\left(q, q_i\right) \leq \varepsilon$. The $\varepsilon$-covering number $\mathcal{N}(Q, d, \varepsilon)$ is the cardinality of the minimal $\varepsilon$-cover.
\end{definition}

\begin{definition}(Algorithm distance). 
\label{def:alg-dist}
Let $\alg$ be an algorithm hypothesis set and $\H=\left(a_t, r_t\right)_{t=1}^n$ be a sequence that is admissible for some task $m \in[M]$. For any pair $\T, \T^{\prime} \in \alg$, define the distance metric $\rho\left(\T, \T^{\prime}\right):=$ $\sup_{\H} \frac{1}{n} \sum_{t=1}^n\left\|\T\left(\wr_t|\H^{t-1}, a_t\right)-\T^{\prime}\left(\wr_t|\H^{t-1}, a_t\right)\right\|_{\ell_2}$.
\end{definition}

\begin{remark}\textbf{(Stability Factor)} 
\label{remark:stability}
The work of \citet{li2023transformers} also characterizes the stability factor $K$ in \Cref{assm:stability-assumption} with respect to the transformer architecture. Assuming loss $\ell(\cdot, \cdot)$ is C-Lipschitz, the algorithm induced by $\T(\cdot)$ obeys the stability assumption with $K=2 C\left((1+\Gamma) e^{\Gamma}\right)^L$, where the norm of the transformer weights are upper bounded by $O(\Gamma)$ and there are $L$-layers of the transformer.
\end{remark}

\begin{remark}\textbf{(Covering Number)} 
\label{remark:covering-number}
From Lemma 16 of \citet{lin2023transformers} we have the following upper bound on the covering number of the transformer class $\T_{\bTheta}$ as 
\begin{align*}
    \log(\mathcal{N}(\alg, \rho, \varepsilon))\leq O(L^2D^2J)
\end{align*}
where $L$ is the total number of layers of the transformer and $J$ and, $D$ denote the upper bound to the number of heads and hidden neurons in all the layers respectively. Note that this covering number holds for the specific class of transformer architecture discussed in section $2$ of \citep{lin2023transformers}.
%
%
\end{remark}

\subsection{Generalization Error to New Task}
\label{app:transfer}

\begin{theorem}\textbf{(Transfer Risk)}
Consider the setting of \Cref{thm:multi-task-risk} and assume the source tasks are independently drawn from task distribution $\cT$. Let $\widehat{\text { TF }}$ be the empirical solution of (ERM) and $g \sim \cT$. Then with probability at least $1-2 \delta$, the expected excess transfer learning risk is bounded by
\begin{align*}
\mathbb{E}_{g}\left[\mathcal{R}_{g}(\widehat{\T}_{})\right] \leq 4 \tfrac{C}{\sqrt{M}} +B \sqrt{\tfrac{2 \log (\mathcal{N}(\alg, \rho, \varepsilon) / \delta)}{M}}
\end{align*}
where, $\mathcal{N}(\alg, \rho, \varepsilon)$ is the covering number of transformer $\widehat{\T}_{}$. 
\end{theorem}

\begin{proof}
Let the target task $g$ be sampled from $\cT$, and the test set $\H_{g} = \{a_t, r_t\}_{t=1}^n$. Define empirical and population risks on $g$ as $\widehat{\L}_{g}(\T)=\frac{1}{n} \sum_{t=1}^n \ell\left(r_t(a_{mt}), \T\left(\wr_t(a_{mt})|\H_{g}^{t-1}, a_t\right)\right)$ and $\L_{g}(\T)=\mathbb{E}_{\H_{g}}\left[\widehat{\L}_{g}(\T)\right]$. Again we drop $\bTheta$ from the transformer notation. Then the expected excess transfer risk following (ERM) is defined as
\begin{align}
\mathbb{E}_{g}\left[\mathcal{R}_{g}(\widehat{\T})\right]=\mathbb{E}_{\H_{g}}\left[\L_{g}(\widehat{\T})\right]-\arg \min _{\T \in \alg} \mathbb{E}_{\H_{g}}\left[\L_{g}(\T)\right] . \label{eq:exces-transfer-risk}
\end{align}
where $\A$ is the set of all algorithms. The goal is to show a bound like this
\begin{align*}
\mathbb{E}_{g}\left[\mathcal{R}_{g}(\widehat{\T})\right] \leq \min _{\varepsilon \geq 0}\left\{4 C \varepsilon+B \sqrt{\frac{2 \log (\mathcal{N}(\alg, \rho, \varepsilon) / \delta)}{T}}\right\}
\end{align*}
where $\mathcal{N}(\alg, \rho, \varepsilon)$ is the covering number.

\textbf{Step 1 (\textbf{(Decomposition)}:} Let $\T^*=\arg \min _{\T \in \alg} \mathbb{E}_{g}\left[\L_{g}(\T)\right]$. The expected transfer learning excess test risk of given algorithm $\widehat{\T} \in \alg$ is formulated as
\begin{align*}
& \widehat{\L}_m(\T):=\frac{1}{n} \sum_{t=1}^n \ell\left(r_{m t}(a_{mt}), \T\left(\widehat{r}_{m t}(a_{mt})| \mathcal{D}_m^{t-1}, a_{m t}\right)\right), \quad \text { and } \\
& \L_m(\T):=\mathbb{E}_{\H_m}\left[\widehat{\L}_t(\T)\right]=\mathbb{E}_{\H_m}\left[\frac{1}{n} \sum_{t=1}^n \ell\left(r_{m t}(a_{mt}), \T\left(\widehat{r}_{m t}(a_{mt})| \mathcal{D}_m^{t-1}, a_{m t}\right)\right)\right], \quad \forall m \in[M] .
\end{align*}
Then we can decompose the risk as 
\begin{align*}
\mathbb{E}_{g}\left[\mathcal{R}_{g}(\widehat{\T})\right] &=\mathbb{E}_{g}\left[\L_{g}(\widehat{\T})\right]-\mathbb{E}_{g}\left[\L_{g}\left(\T^*\right)\right]\\
& = 
\underbrace{\mathbb{E}_{g}\left[\L_{g}(\widehat{\T})\right]-\widehat{\L}_{\H_{\mathrm{all}}}(\widehat{\T})}_a + \underbrace{\widehat{\L}_{\H_{\text {all }}}(\widehat{\T})-\widehat{\L}_{\H_{\text {all }}}\left(\T^*\right)}_b +\underbrace{\widehat{\L}_{\H_{\text {all }}}\left(\T^*\right)-\mathbb{E}_{g}\left[\L_{g}\left(\T^*\right)\right]}_c .
\end{align*}
Here since $\widehat{\T}$ is the minimizer of training risk, $b<0$. Then we obtain
\begin{align}
\mathbb{E}_{g}\left[\mathcal{R}_{g}(\widehat{\T})\right] \leq 2 \sup _{\T \in \alg}\left|\mathbb{E}_{g}\left[\L_{g}(\T)\right]-\frac{1}{M} \sum_{m=1}^M \widehat{\L}_m(\T)\right| .\label{eq:bounding-decomp}
\end{align}
\textbf{Step 2 (Bounding \eqref{eq:bounding-decomp})}For any $\T \in \alg$, let $X_t=\widehat{\L}_t(\T)$ and we observe that
\begin{align*}
\mathbb{E}_{m \sim \cT}\left[X_t\right]=\mathbb{E}_{m \sim \cT}\left[\widehat{\L}_m(\T)\right]=\mathbb{E}_{m \sim \cT}\left[\L_m(\T)\right]=\mathbb{E}_{g}\left[\L_{g}(\T)\right]
\end{align*}
Since $X_m, m \in[M]$ are independent, and $0 \leq X_m \leq B$, applying Hoeffding's inequality obeys
\begin{align*}
\Pb\left(\left|\mathbb{E}_{g}\left[\L_{g}(\T)\right]-\frac{1}{M} \sum_{m=1}^M \widehat{\L}_m(\T)\right| \geq \tau\right) \leq 2 e^{-\frac{2 M \tau^2}{B^2}} .
\end{align*}
Then with probability at least $1-2 \delta$, we have that for any $\T \in \alg$,
\begin{align}
\left|\mathbb{E}_{g}\left[\L_{g}(\T)\right]-\frac{1}{M} \sum_{m=1}^M \widehat{\L}_m(\T)\right| \leq B \sqrt{\frac{\log (1 / \delta)}{2 M}} . \label{eq:decomp-1}
\end{align}

Next, let $\alg_{\varepsilon}$ be the minimal $\varepsilon$-cover of $\alg$ following \Cref{def:covering-number}, which implies that for any task $g \sim \cT$, and any $\T \in \alg$, there exists $\T^{\prime} \in \alg_{\varepsilon}$
\begin{align}
\left|\L_{g}(\T)-\L_{g}\left(\T^{\prime}\right)\right|,\left|\widehat{\L}_{g}(\T)-\widehat{\L}_{g}\left(\T{ }^{\prime}\right)\right| \leq C \varepsilon . \label{eq:decomp-2}
\end{align}

Since the distance metric following Definition 3.4 is defined by the worst-case datasets, then there exists $\T^{\prime} \in \alg_{\varepsilon}$ such that
\begin{align*}
\left|\mathbb{E}_{g}\left[\L_{g}(\T)\right]-\frac{1}{M} \sum_{m=1}^M \widehat{\L}_m(\T)\right| \leq 2 C \varepsilon .
\end{align*}

Let $\mathcal{N}(\alg, \rho, \varepsilon)=\left|\alg_{\varepsilon}\right|$ be the $\varepsilon$-covering number. Combining the above inequalities (\eqref{eq:bounding-decomp}, \eqref{eq:decomp-1}, and \eqref{eq:decomp-2}), and applying union bound, we have that with probability at least $1-2 \delta$,
\begin{align*}
\mathbb{E}_{g}\left[\mathcal{R}_{g}(\widehat{\T})\right] \leq \min _{\varepsilon \geq 0}\left\{4 C \varepsilon+B \sqrt{\frac{2 \log (\mathcal{N}(\alg, \rho, \varepsilon) / \delta)}{M}}\right\}
\end{align*}
Again by setting $\varepsilon = 1/\sqrt{M}$
\begin{align*}
    \L(\widehat{\T})-\L\left(\T^*\right)\leq \dfrac{4 C}{\sqrt{M}}+ 2B \sqrt{\frac{\log (\mathcal{N}(\alg, \rho, \varepsilon) / \delta)}{c M}}
\end{align*}
The claim of the theorem follows.
\end{proof}

\begin{remark} (Dependence on $n$)
\label{remark:target-risk}
In this remark, we briefly discuss why the expected excess risk for target task $\cT$ does not depend on samples $n$. 
The work of \citet{li2023transformers} pointed out that the MTL pretraining process identifies a favorable algorithm that lies in the span of the $M$ source tasks.  
%
This is termed as inductive bias (see section 4 of \citet{li2023transformers}) \citep{soudry2018implicit, neyshabur2017geometry}. 
%
Such bias would explain the lack of dependence of the expected excess transfer risk on $n$ during transfer learning. 
%
%
This is because given a target task $g\sim\cT$, the $\T$ can use the learnt favorable algorithm to conduct a discrete search over span of the $M$ source tasks and return the source task that best fits the new target task. Due to the discrete search space over the span of $M$ source tasks, it is not hard to see that, we need $n \propto \log (M)$ samples (which is guaranteed by the $M$ source tasks) rather than $n \propto d$ (for the linear setting).
%
\end{remark}

\subsection{Table of Notations}
\label{table-notations}

\begin{table}[!tbh]
    \centering
    \begin{tabular}{|p{10em}|p{28em}|}
        \hline\textbf{Notations} & \textbf{Definition} \\\hline
        $M$ & Total number of tasks \\\hline
        $d$ & Dimension of the feature \\\hline
        $\A_m$  & Action set of the $m$-th task \\\hline
        $\X_m$  &Feature space of $m$-th task \\\hline
        $\Mts$  & Tasks for testing\\\hline
        $\Mpr$  & Total Tasks for pretraining\\\hline
        $\bx(m, a)$  & Feature of action $a$ in task $m$\\\hline
        $\btheta_{m,*}$  & Hidden parameter for the task $m$\\\hline
        $\cTp$  & Pretraning distribution on tasks\\\hline
        $\cTs$  & Testing distribution on tasks\\\hline
        $n$  & Total horizon for each task $m$\\\hline
        $\H_m  = \{I_t, r_t\}_{t=1}^n$  & Dataset sampled for the $m$-th task containing $n$ samples\\\hline
        $\H^t_m = \{I_s, r_s\}_{s=1}^t$  & Dataset sampled for the $m$-th task containing samples from round $s=1$ to $t$\\\hline
        $\bw$ & Transformer model parameter\\\hline
        $\T_{\bw}$ & Transformer with model parameter $\bw$\\\hline
        $\Dpr$  & Pretraining in-context distribution\\\hline
        $\Htr$  & Training in-context dataset\\\hline
        $\Dts$  & Testing in-context distribution\\\hline
    \end{tabular}
    \vspace{1em}
    \caption{Table of Notations}
    \label{tab:my_label}
\end{table}






\end{document}